\documentclass[lettersize,journal]{IEEEtran}
\usepackage{amsmath,amsfonts}
\usepackage{algorithmic}
\usepackage{algorithm}
\usepackage{array}
\usepackage{textcomp}
\usepackage{stfloats}
\usepackage{url}
\usepackage{verbatim}
\usepackage{graphicx}
\usepackage{cite}
\usepackage{subfigure}
\usepackage{dsfont}
\usepackage{wrapfig}
 \usepackage{booktabs}
 \usepackage{amssymb}
\newtheorem{lemma}{Lemma}
\newtheorem{theorem}{Theorem}
\newtheorem{proof}{Proof}

\begin{document}

\title{Advantage-based Temporal Attack in Reinforcement Learning}

\author{Shenghong He,~\IEEEmembership{Student Member, IEEE}, Chao Yu,~\IEEEmembership{Member, IEEE,} Danying Mo, Yucong Zhang, Yinqi Wei 
\thanks{ Shenghong He, Chao Yu, Danying Mo and Yucong Zhang are with the School of Computer
 Science and Engineering, Sun Yat-sen University, Guangzhou 510006, China
 (e-mail: yuchao3@mail.sysu.edu.cn)}
\thanks{ Yinqi Wei is with the Nanyang Technological University, Singapore.}}

\markboth{Journal of \LaTeX\ Class Files,~Vol.~14, No.~8, August~2021}%
{Shell \MakeLowercase{\textit{et al.}}: A Sample Article Using IEEEtran.cls for IEEE Journals}


\maketitle

\begin{abstract}
Extensive research demonstrates that Deep Reinforcement Learning (DRL) models are susceptible to adversarially constructed inputs (i.e., adversarial examples), which can mislead the agent to take suboptimal or unsafe actions.
Recent methods improve attack effectiveness by leveraging future rewards to guide adversarial perturbation generation over sequential time steps (i.e., reward-based attacks).
However, these methods are unable to capture dependencies between different time steps in the perturbation generation process, resulting in a weak temporal correlation between the current perturbation and previous perturbations.
In this paper, we propose a novel method called Advantage-based Adversarial Transformer (AAT), which can generate adversarial examples with stronger temporal correlations (i.e., time-correlated adversarial examples) to improve the attack performance.
AAT employs a multi-scale causal self-attention (MSCSA) mechanism to dynamically capture dependencies between historical information from different time periods and the current state, thus enhancing the correlation between the current perturbation and the previous perturbation.
Moreover, AAT introduces a weighted advantage mechanism, which quantifies the effectiveness of a perturbation in a given state and guides the generation process toward high-performance adversarial examples by sampling high-advantage regions.
Extensive experiments demonstrate that the performance of AAT matches or surpasses mainstream adversarial attack baselines on Atari, DeepMind Control Suite and Google football tasks.
\end{abstract}


\section{Introduction}

As Deep Reinforcement Learning (DRL) technology advances in various domains, including computer games~\cite{DBLP:HoHLLWW23, DBLP:LinLSYF022}, large-scale interactive models (e.g., ChatGPT)~\cite{DBLP:Ouyang0JAWMZASR22, DBLP:LiuYFJHN23}, and security analysis~\cite{DBLP:ChuG24, DBLP:FuLWYFWDWX25}, the safety of DRL models has become a critical concern~\cite{wali2025explainable,DBLP:HuangCF23}.
A prominent threat in this concern is adversarial attacks~\cite{lv2024safe,DBLP:QiaobenYZSZZ24}, which manipulate the actions of the agent by injecting imperceptible perturbations into input data, thus leading to instability of the DRL model.

Currently, adversarial attack methods in DRL can be categorized into gradient-based and reward-based attack methods according to the generation technique of adversarial examples.
Gradient-based attack methods~\cite{DBLP:HuangPGDA17, DBLP:LinHLS0S17} employ gradient optimization techniques (e.g., FGSM~\cite{DBLP:GoodfellowSS14, DBLP:SzegedyZSBEGF13}) to generate perturbations at each time step, and thus mislead the decision-making of the agent.
However, since these methods focus solely on the current time step, they fail to account for the long-term objective, limiting their ability to undermine the objective of the agent.

To address this issue, recent works~\cite{DBLP:OikarinenZMDW21,DBLP:abs-2305-17342,schott2024robust, DBLP:SunZX0ZCL20, DBLP:YuS22} use reinforcement learning methods to generate adversarial examples by leveraging current state and future rewards (i.e., reward-based attack methods).
However, existing reward-based methods cannot capture the dependence between the current state and past information, which weakens the correlation between the current perturbation and the previous perturbation, and ultimately fails to effectively reduce the cumulative reward of the agent.

Sequence modeling~\cite{DBLP:DT,DBLP:ZhengZG22} can be employed to capture temporal dependencies and leverage future reward sums to regulate adversarial examples.
However, this approach exhibits two major limitations in adversarial attack scenarios.
The first limitation is that while sequence modeling captures global temporal dependencies, it fails to distinguish the varying dependency structures that exist across different time scales (i.e., short-term and long-term dependencies), which are both crucial for generating adversarial perturbations in sequential decision-making. 
Specifically, short-term dependencies enable the attack strategy to generate perturbations based on local state changes to effectively disrupt the immediate decision-making of the agent.
Conversely, long-term dependencies ensure that these perturbations coherently serve the attack's overarching objective: minimizing the cumulative reward of the agent.

The second limitation is that the sequence modeling approach relies on the cumulative reward of the expert trajectories (i.e., examples with high attack performance) to iteratively refine perturbations across time steps.
In scenarios where expert trajectories are scarce, this approach cannot learn the optimal perturbations from different suboptimal trajectories (i.e., non-expert trajectories) and utilize these perturbations to generate adversarial examples with high attack performance.
The lack of high-performance adversarial examples leads to an inability to effectively reveal the vulnerability of decision boundaries that cannot be detected under conventional tests, which may cause irreparable damage to real-world applications (e.g., causing traffic accidents).


In this paper, we propose a novel method called Advantage-based adversarial transformer (AAT) to generate time-correlated adversarial examples with high attack performance.
AAT comprises two key mechanisms: a multi-scale causal self-attention (MSCSA) mechanism and a weighted advantage mechanism.
The MSCSA mechanism models historical fragments of different lengths (i.e., short-term and long-term perturbation trends) separately, and conducts historical information aggregations between each scale to capture multi-grained timing characteristics, which are then used to generate adversarial examples.
By doing so, MSCSA enables the generated perturbations to precisely disrupt the single-step decision of the agent and effectively reduce the cumulative rewards of the agent, thereby enhancing the overall stability and effectiveness of the attack.
To reduce reliance on expert data, the weighted advantage mechanism quantifies the effectiveness of a perturbation within a specific state (i.e., perturbation advantage), which then serves as a guiding signal to prioritize perturbations with higher advantage values while suppressing less effective ones.
This process is analogous to the greedy selection of optimal actions in dynamic programming~\cite{DBLP:KostrikovNL22,DBLP:GaoWCKZ024}, ensuring that AAT generates the optimal adversarial perturbation in each state based on the maximum feasible advantage.
Crucially, this advantage can be reused in similar states across different trajectories, which enables AAT to learn optimal perturbations even in suboptimal data.
Additionally, the weighted advantage mechanism can constrain the estimated advantage of unseen states within a bounded range, thus mitigating the problem of attack performance degradation caused by the overestimation of advantage values~\cite{DBLP:KostrikovNL22}.
The contributions of this work are as follows:

\begin{itemize}
    \item We explore adversarial example generation from a sequence modeling perspective and propose the MSCSA mechanism to capture long-term dependencies while focusing on short-term dynamic characteristics in the generation process.
    
    \item We propose a weighted advantage mechanism to explore adversarial perturbations in high-advantage regions during adversarial example generation, and theoretically demonstrate the effectiveness of the weighted advantage in improving the performance of AAT attacks.
    
    \item Experimental results confirm that AAT significantly undermines the cumulative rewards of target policies in both white-box and black-box scenarios, outperforming the existing reward-based attack baseline by approximately 3$\%$. Moreover, AAT significantly improves the efficiency of adversarial example generation by producing effective perturbations via a single forward pass.

\end{itemize}


\begin{table}[]
\centering
\caption{Symbol explanation.}
\begin{tabular}{ll}
\hline
notation                                      & Explanation                                                                                                                   \\ \hline
$\mathbf{M}$                                             & Markov Decision Process                                                                                                       \\
$S,\mathbb{A},\mathcal{R},\mathcal{P},\gamma$ & \begin{tabular}[c]{@{}l@{}}State space, action space, reward function, \\ transition probaility, discount factor\end{tabular} \\
$\pi$, $\tau$, $P_\pi(\tau)$                  & \begin{tabular}[c]{@{}l@{}}Policy, trajectory, trajectory distribution under\\ policy $\pi$\end{tabular}                      \\
$J(\pi)$                                      & Expected cumulative reward                                                                                                    \\
$\delta_t$,                                   & Adversarial perturbation at time $t$                                                                                            \\
$s'_t $                                       & Adversarial state, $s'_t = s_t + \delta_t$                                                                                    \\
$\varepsilon$                                    & Perturbation threshold                                                                                                        \\
$\pi_v$, $R_t$                                & Target policy, Sum of future returns                                                                                          \\
$Q_\theta$, $V_\psi$                          & State-action value function, state value function                                                                             \\
$A_t$, $\tilde{A}_t$                          & Advantage, weighted advantage                                                                                                 \\
$\lambda$                                     & Hyperparameter for weighted advantage correction                                                                              \\
MSCSA                                         & Multi-scale Causal Self-attention                                                                                             \\
EMB                                           & Embedding network                                                                                                             \\
$\hat{s}$, $\mathcal{F}_{\theta}$                        & Encoded state vector, Advantage network                                                                                       \\
$o^m_t$                                       & m-th patch of state $s_t$                                                                                                     \\
$L_a$, $L_{norm}$, $L$                         & \begin{tabular}[c]{@{}l@{}}Action loss, Norm loss for perturbation magnitude,\\ total training loss\end{tabular}              \\
$L_\theta$, $L_\psi$                          & Loss for $Q_\theta$ and $V_\psi$                                                                                              \\
$\sigma$                                      & Hyperparameter in expected regression                                                                                         \\
$\Lambda$                                     & Weighted advantage function                                                                                                   \\
$\mu(\delta|s)$                               & \begin{tabular}[c]{@{}l@{}}Advantage distribution of perturbation $\delta$ \\ given state s\end{tabular}                      \\
$\mathcal{H}$, $\mathcal{L}_{reg}$, $\mathcal{L}_\phi$                       & \begin{tabular}[c]{@{}l@{}}Entropy, advantage regression loss, objective \\ for advantage network\end{tabular}                \\
$k$                                           & Hyperparameter balancing entropy and regression                                                                               \\
$L_k$                                         & Length of causal window at scale $k$                                                                                            \\
$Q^{(k)}, K^{(k)}, V^{(k)}$                   & Queries, keys, values at scale $k$                                                                                              \\
$M^{(k)}$                                     & Causal mask at scale $k$                                                                                                        \\
$O^{(k)}$                                     & Attention output at scale $k$                                                                                                  \\
$z_t$                                         & Fused multi-scale representation                                                                                              \\
$H^{(0)}_{t,k}$               & Initial sub-sequence at scale $k$                                                                                               \\
$W_g$                                          & Gated fusion parameter                                                                                                        \\
$W^Q, W^K, W^V$                               & Projection matrices for attention                                                                                             \\
$\epsilon$                                    & Mean error bound for advantage estimate                                                                                       \\ 
$\mathbb{X}$                                            & Number of patches per state \\
\hline
\end{tabular}
\label{tab:notion}
\end{table}

\section{Preliminaries}
\subsection{Adversarial attack}
Adversarial attack~\cite{DBLP:HeWLYJLZ23,DBLP:TianSWXZLL24} is an attack method against machine learning models, particularly prevalent in the field of deep learning. 
Attackers introduce meticulously crafted subtle perturbations into the input data, which are often undetectable to the human eye, leading the model to generate incorrect outputs.
Typically, input examples that can mislead machine learning models into making incorrect decisions are called adversarial examples, which can be expressed as $x'=x+\delta$, where $x$ is a clean example (i.e., original input), and $\delta$ denotes adversarial perturbations added to the clean example.
For a model $F:X \rightarrow Y$ mapping inputs $x \in X$ to labels $y \in Y$, the attacker aims to find an adversarial perturbation $\delta$ within a given budget $\epsilon$ that maximizes the loss $\mathbb{J}$ of model $F$, thereby causing a misclassification.
To ensure stealthiness, $\delta$ is constrained under an $L_p$-norm, where $p \in (1,2,\cdots,\infty)$.
This optimization problem under the constraint can be expressed as follows: 
\begin{equation}
    \delta^* = \arg\max_{\delta: \|\delta\|_p \leq \varepsilon} \mathbb{J}(F(x+\delta), y),
\end{equation}
where $ \mathbb{J} $ is the cross-entropy loss and $\varepsilon$ is the perturbation threshold.

\subsection{DRL adversarial attack}
Deep reinforcement learning (DRL) is a decision-making process that can be modeled by a Markov Decision Process (MDP) $\mathbf{M}=(S,\mathbb{A},\mathcal{R},\mathcal{P},\gamma)$, which includes state space $S$, action space $\mathbb{A}$, reward function $\mathcal{R}$, transition probability $\mathcal{P}$ and discount factor $\gamma$.
In DRL, the agent learns a policy that defines an action distribution conditioned on the state $\pi(a_t|s_t)$ at time step $t$, where $a_t \in \mathbb{A}$, $s_t \in S$.
Building on this definition, the trajectory $\tau=(s_0,a_0,\cdots,s_T,a_T)$ is generated by the interaction of the policy and the environment, which can be described as a distribution $P_\pi(\tau)= \prod^T_{t=0} \pi(a_t|s_t)\mathcal{P}(s_{t+1}|s_t,a_t)$, where $T$ is the length of the trajectory. 
Thus, the ultimate goal of DRL is to maximize the expected cumulative reward over trajectories, formalized as $J(\pi)=\mathbb{E}_{\tau \sim P_\pi(\tau)}[\sum^T_{t=0}\gamma^t \mathcal{R}(s_t,a_t)]$.

Similar to adversarial attack methods in image classification~\cite{DBLP:MaXJZS23,DBLP:YangLLZFZWLS24,DBLP:FangS24}, adversarial attack methods in DRL strategically perturb the decision-making process to minimize its cumulative rewards.
At time step $t$ of the interaction between the agent and the environment, the attacker introduces perturbations into the current state to form adversarial examples. 
Then, the adversarial examples are input into the agent, misleading it to take actions with the lowest rewards.
The goal of the attacker is to minimize the expected cumulative reward of the target policy in the adversarial states $\mathbb{E}_{a_t \sim \pi_V(s'_t)}[\sum^T_{t=0}\gamma^t \mathcal{R}(s'_t,a_t)]$, where $s'_t$ is adversarial example (i.e., $s'_t=s_t+\delta_t)$), and $\pi_V$ is the target policy. In addition, to improve the readability of the paper, Table~\ref{tab:notion} provides a centralized definition and explanation of the main symbols used in this work.


\subsection{Threat Model}
Our work focuses on untargeted temporal state perturbation attacks, which aim to minimize the cumulative reward of the agent by adding small perturbations to the state.
To formalize this approach, we define our threat model based on the attacker's objective, knowledge, and capabilities.

$\quad$
\\
\textbf{Attacker objective}.
The attacker's objective is to minimize the expected cumulative reward of the target policy $\pi_V$ by crafting adversarial perturbations $\delta_t$ that are added to the clean states $s_t$ to form adversarial states $s'_t = s_t + \delta_t$. 
The perturbation $\delta_t$ is constrained by an $L_2$-norm bound $\varepsilon$ to ensure imperceptibility, i.e., $|\delta_t|_2 \leq \varepsilon$. 
The specific values of $\varepsilon$ used in our experiments are provided in Appendix~\ref{sup:struct}.

$\quad$
\\
\noindent\textbf{Attacker knowledge}.
Our method considers white-box and black-box scenarios.
In white-box scenarios, the attacker has complete knowledge of the target policy, including its model structure, parameters, training algorithms, and environment information.
In contrast, black-box scenarios cannot access the information of the target policy, but can still interact with the unattacked environment (i.e., clean environment), which has been widely used in previous work~\cite{DBLP:LinHLS0S17,DBLP:YuS22}.

$\quad$
\\
\textbf{Attacker capabilities}.
The attacker can perturb the state observations $s_t$ that are input to the target policy. 
In white-box scenarios, the attacker has full access to the target policy and can thus collect data generated during the attack.
For example, the attacker can apply random or inverse gradient perturbations (e.g., FGSM) to modify states of the target policy,  and collect the resulting perturbed states, actions and rewards to form a dataset $D$, which can be used for tuning attack strategies.
This setup is common and easily satisfied in the real world, because an attacker can typically access state information and observe the actions of the target policy and the corresponding environment rewards when launching an attack~\cite{DBLP:SunZLH22,DBLP:YuS22}.
In black-box scenarios, the attacker cannot access the internal information of the target policy and thus constructs a substitution policy to approximate the behavior of the target policy.
For instance, in a CartPole environment where the target is an unknown PPO model, the attacker might train a DQN agent as the substitute policy. 
By interacting with the substitution policy, the attacker can collect interaction data to form Data $D$.

\begin{figure*}[t]
    \centering
    \includegraphics[width=0.98\linewidth]{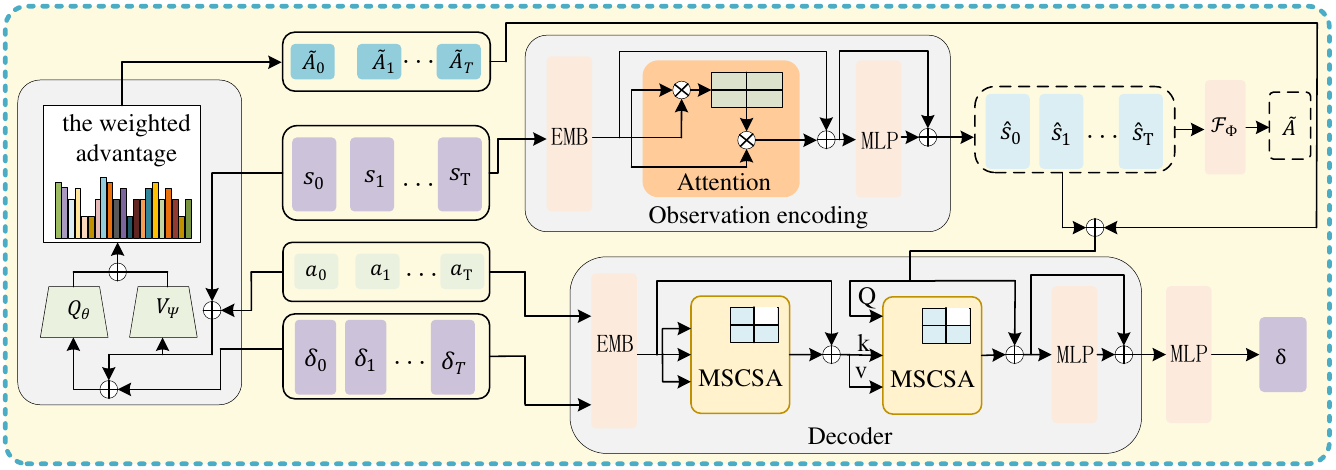}
    \caption{The AAT training structure. $\oplus$ denotes the concatenation of vectors, $\otimes$ denotes the element-wise multiplication of vectors and MLP is a multilayer perception machine. During the training phase, AAT generates adversarial perturbations using the weighted advantage calculated from the $Q_\theta$ and $V_\psi$ value functions.}
    \label{fig:AAT_strcut}
\end{figure*}

\section{AAT method}


\subsection{Design Rationale}
The core design concept of AAT is to generate adversarial examples with temporal correlation and high attack performance. 
Specifically, each perturbation should be designed with respect to preceding perturbation sequences so as to form a coherent strategy aimed at minimizing the cumulative reward of the agent.
Achieving this objective requires addressing two complementary challenges.
First, adversarial perturbations must reflect both short-term fluctuations that alter immediate action choices and long-term trends that steer the agent toward persistently suboptimal trajectories. 
Second, the generation of adversarial examples should prioritize perturbations that have a higher downstream impact on rewards rather than simply learning a mapping from states to adversarial perturbations.

Our design addresses these challenges through two complementary mechanisms: MSCSA and weighted advantage guidance.
MSCSA enhances temporal correlation by providing multi-granular context, where separate causal attention windows capture short-, medium-, and long-term temporal dependencies in parallel, while a gated fusion integrates these multi-scale signals into a unified, temporally-aware latent representation.
Specifically, the short-scale attention enables the model to generate perturbations that immediately flip critical action decisions. 
In contrast, the long-scale attention preserves coherence across extended sequences, thereby allowing perturbations to accumulate into a trajectory-level effect.

The weighted advantage mechanism complements MSCSA by serving as an auxiliary optimization signal that directs the generative process toward high-impact perturbations. 
Specifically, by quantifying the expected marginal effect of each perturbation on the future return, the weighted advantage functions both as a conditioning signal for the decoder and as a sampling factor during evaluation. 
This design yields two practical benefits: it tends to reproduce historical perturbation patterns that have been empirically effective (i.e., those with high advantage), and it dynamically guides the generation process by leveraging the maximum advantage in the current state, thereby promoting local optimality at each step.

MSCSA and the weighted advantage mechanism are distinct but complementary components that operate in coordination.
As can be seen in Fig.~\ref{fig:AAT_strcut}, AAT first computes the weighted advantage ($A$) from training data by training the Q-network ($Q_\theta$) and V-network ($V_\psi$). 
Subsequently, in order to predict advantage based only on the state information during the attack phase, the weighted advantage and embedding vectors $\hat{s}$ (encoded from the state $s$ via an embedding network (EMB) and an attention mechanism) are used to learn an advantage network ($\mathcal{F}_\phi$).
Additionally, this weighted advantage is combined with the state encoding vector as a generative conditional input to the decoder.
Finally, the decoder takes the state embedding vector, weighting advantage, action embedding vector and perturbation embedding vector as inputs to generate adversarial perturbations using MSCSA so as to capture dependencies in different time periods.
The existence of weighted advantages means that the generation process is no longer simply mimicking historical data, but is instead directly optimized towards high-performance attacks.

\subsection{Technical Description}

$\quad$
\\
\textbf{The MSCSA mechanism}.
MSCSA extends standard causal self‑attention by (i) operating concurrently at multiple temporal resolutions and (ii) enforcing strict causality via masking.
More precisely, given an input trajectory of states $\mathcal{S}_t = \bigl\{s_{t-L_K+1},\,s_{t-L_K+2},\,\dots,\,s_t\bigr\}$, each state is first projected into a $d$‑dimensional feature space and a learnable positional encoding is added: $h_i^{(0)} \;=\; \mathrm{LayerNorm}\bigl(W_s\,s_i + p_i\bigr), \quad i = t-L_K+1,\dots,t$, where $W_s \in \mathbb{R}^{d\times\dim(s)}$ and $p_i\in\mathbb{R}^d$ is the positional embedding for time step $i$.
To capture dependencies at different temporal extents, we define $K$ causal windows of increasing lengths $L_1 < L_2 < \dots < L_K$, and extract for each scale $k$ the sub‑sequence $H_{t,k}^{(0)} = \bigl[h_{t-L_k+1}^{(0)},\,h_{t-L_k+2}^{(0)},\,\dots,\,h_t^{(0)}\bigr]\in\mathbb{R}^{L_k\times d}$.
Each $H_{t,k}^{(0)}$ serves as input to an independent causal self‑attention block.
For scale $k$, queries $Q^{(k)}$, keys $K^{(k)}$, and values $V^{(k)}$ are computed via linear projections of $H_{t,k}^{(0)}$: 
\begin{equation}
        Q^{(k)} = H_{t,k}^{(0)}W^Q,\quad 
        K^{(k)} = H_{t,k}^{(0)}W^K,\quad 
        V^{(k)} = H_{t,k}^{(0)}W^V,
\end{equation}
where $W^Q,W^K, \text{and}, W^V\in\mathbb{R}^{d\times d}$.
We then apply a scaled dot‑product attention with a causal mask $M^{(k)}\in{0,-\infty}^{L_k\times L_k}$ to enforce that position $i$ attends only to ${1,\dots,i}$, where: 
\begin{equation}
    M^{(k)}_{i,j} = 
\begin{cases}
0, & j \le i,\\
-\infty, & j > i.
\end{cases}
\end{equation}
The attention output at each position is:
\begin{equation}
    O^{(k)} = \mathrm{Softmax}\!\Bigl(\tfrac{Q^{(k)}(K^{(k)})^\top}{\sqrt{d}} + M^{(k)}\Bigr)\,V^{(k)}.
\end{equation}
Having obtained per‑scale representations $O^{(1)},\dots,O^{(K)}$, they are fused into a single vector $z_t\in\mathbb{R}^d$.  
Let $o^{(k)}=O^{(k)}_{L_k}$ denote the final token of scale $k$. 
In particular, we use a gated fusion operator:
\begin{equation}
    z_t = \sum_{k=1}^K  \varrho \bigl(W_g [\,o^{(k)};h_t^{(0)}]\bigr) \odot o^{(k)},
\end{equation}
where $W_g\in\mathbb{R}^{d\times 2d}$ is the control parameter and $\varrho$ is the sigmoid activation. 

$\quad$
\\
\textbf{Trajectory representation.}
Inspired by the DT method~\cite{DBLP:DT}, which uses returns-to-go to learn the optimal action of the policy, we use the sum of future returns as the reward model for adversarial perturbations: $\hat{R}_t=\sum^T_{t=0} log_{\frac{1}{2}}r_t$, where $r_t$ is the target policy reward and $r \in (0,1]$.
This logarithmic transformation quantifies the inverse relationship between the cumulative rewards of the attack policy and the target policy. We then obtain a trajectory representation $\tau=(\hat{R}_0,s_0,\delta_0,a_0,\cdots,\hat{R}_T,s_T,\delta_T,a_T)$ for autoregressive training and generation.
However, high-dimensional inputs (i.e., images) often contain elements that are irrelevant to the task (e.g., background, luminance), which reduces the generalization ability of the model.
To address this issue, we adopt the Vision Transformer (ViT) architecture~\cite{DBLP:DosovitskiyB0WZ21} and divide each observation image $s_t$ into $\mathbb{X}$ non-overlapping patches, where each patch corresponds to a 14×14-pixel region.
By the self-attention mechanism, this patch-based representation dynamically weights the importance of each patch, filtering out irrelevant features and emphasizing task-critical patterns.
Thus, $\tau$ can be rewritten as $ \tau=(\hat{R}_0,o^1_0,\cdots,o^{\mathbb{X}}_0,\delta_0,a_0,\cdots,\hat{R}_T,o^1_T,\cdots,o^{\mathbb{X}}_T,\delta_T,a_T)$.

$\quad$
\\
\textbf{The learning phase}.
The prediction head associated with the input samples is trained to predict adversarial perturbations based on the mean squared error of actions, which can be expressed as:
\begin{equation}
    L_{a}=\mathbb{E}_{\tau \sim D}[\sum_{t=0}^T (\pi_V(s_t+AAT(\tau_{<t},\hat{R}_t,s_t))-a_t)^2],
\end{equation}
where $\tau_{<t}$ represents the historical trajectory before time step $t$ (i.e., $\tau_{<t}=(\hat{R}_1,s_1,\delta_1,a_1,\cdots,\hat{R}_{t-1},s_{t-1},\delta_{t-1},a_{t-1})$), and $a_t$ represents the action that the attacker desires the target policy to take at time step $t$.
To ensure the perturbations remain imperceptible, we simultaneously minimize the $L_2$ norm of the perturbation magnitude:
\begin{equation}
\label{eq:l2}
    L_{norm}=\mathbb{E}_{s_t}[||s'_t-s_t||_2],
\end{equation}
where $s'_t=s_t+AAT(\tau_{<t},s_t)$ is the adversarial example, and $s_t$ is the original state.
The overall training loss is expressed as:
\begin{equation}
\label{eq:total_loss}
    L=L_a+L_{norm}.
\end{equation}
This autoregressive learning method using a supervised learning paradigm can thus quickly learn adversarial perturbations from training data.

$\quad$
\\
\textbf{Policy improvement}.
AAT can generate time-correlated adversarial perturbations based on historical information to minimize the cumulative reward of the target policy.
However, the effectiveness of autoregressive learning depends on the expert dataset.
In the case where the dataset includes a large number of suboptimal trajectories, naive AAT fails to learn effective adversarial perturbations that minimize the cumulative reward of the target policy.
To address this limitation, an advantage function is then used to  boost the performance of adversarial perturbations by extracting effective perturbation patterns from high-quality trajectory segments across different trajectories.
Specifically, AAT first uses expectation regression~\cite{DBLP:KostrikovNL22,DBLP:GaoWCKZ024} to improve the estimation of the state-action value function $Q$ and the state value function $V$, expressed as: 
\begin{equation}
\label{eq:qv}
    \begin{split}
        L_\theta=\mathbb{E}_{(s_t,a_t,r_t,\delta_t,s_{t+1}, a_{t+1})}[(r_t&+\gamma V_\psi(s_{t+1},a_{t+1})\\ &-Q_\theta(s_t,a_t,\delta_t))^2],
    \end{split}
\end{equation}
\begin{equation}
    L_\psi=\mathbb{E}_{(s_t,a_t,\delta_t)}[L_{\sigma}(Q_{\hat{\theta}}(s_t,a_t,\delta_t)-V_\psi(s_t,a_t))],
\end{equation}
where $\theta$ and $\psi$ denote the parameters of the state-action value function and the state value function, respectively, $\sigma$ is a hyperparameter and $L_{\sigma}(\nu)=|\sigma - \mathds{1}(\nu<0)| \nu^2$ is the expected regression loss.
In the case of infinite examples and $\sigma \rightarrow 1$, $L_{\sigma}$ produces an approximation of the optimal value function within the sample.
By Leveraging $Q$ and $V$ functions, the adversarial advantage $A$ can be calculated to quantify the impact of perturbations on policy degradation as follows:
\begin{equation}
\label{eq:advantage}
    A_t=Q_\theta(s_t,a_t,\delta_t)-V_\psi(s_t,a_t).
\end{equation}

However, directly applying the advantage estimation of Eq.~(\ref{eq:advantage}) may suffer from overestimation errors when evaluating unseen states and actions, which can cause AAT to sample from a suboptimal advantage distribution.
To address this problem, we propose a weighted advantage mechanism to adjust the weights of the estimated advantage $A$ as follows:
\begin{equation}
    \tilde{A_t} = \Lambda(s_t,a_t,\delta_t)= \frac{Q_\theta(s_t,a_t,\delta_t)-V_\psi(s_t,a_t)}{1 + \lambda |Q_\theta(s_t,a_t,\delta_t)-V_\psi(s_t,a_t)|},
\end{equation}
where $\lambda$ serves as a hyperparameter that modulates the correction intensity for overestimated values, and $|\cdot|$ denotes the absolute value of advantage.
The weighted advantage function can constrain the evaluated advantages within a continuous and smooth interval to avoid the occurrence of extreme values.
In particular, when the absolute value of $A$ is large (i.e., $|A| \to \infty$), the weighted advantage $\tilde{A}$ is approximately: $\tilde{A} \approx \frac{\text{sign}(A)}{\lambda}$, which implies that extreme estimates are substantially bounded within a constrained interval, mitigating the adverse effects on performance estimation.
When the absolute value $|A|$ is small (i.e., $|A| \ll 1/\lambda$), we have $\tilde{A} \approx A$, which retains sensitivity to small advantage values.
Replacing $\tilde{A}$ with $\hat{R}$, we obtain an advantage-augmented dataset $D=\{\tau_i\}_{i=0}^N$, where $\tau=(s_0,a_0,\delta_0,\tilde{A}_0,\cdots,s_T,a_T,\delta_T,\tilde{A}_T)$.

AAT then accepts sequences of states, actions, perturbations, and advantages as inputs to generate adversarial perturbations.
The perturbation generation process can be decomposed into $P(\delta_t | \tau_{<t}, s_t, \tilde{A}) \propto P(\delta_t | \tau_{<t},s_t) e^{\omega \tilde{A}}$, where $P(\delta_t | \tau_{<t},s_t)$ is a basic generative model that learns the distribution of perturbation from training data to ensure the rationality of generated perturbation, $e^{\omega \tilde{A}}$ adjusts weights based on advantage values to prioritize perturbations aligned with target advantages, and $\propto$ denotes the normalization that turns a weighted distribution into a legitimate probability distribution.
By integrating the target advantage value into the generation process, AAT can dynamically optimize perturbations to generate high-performance adversarial examples without dependence on complete expert examples.

During the attack phase, adversarial examples are generated by conditioning on the attack advantage to ensure that perturbations are optimal at each time step.
To this end, we construct an advantage prediction function based on maximum entropy regularization, which learns a more dispersed policy distribution by maximizing the entropy of the advantage distribution, and thus enables better coverage of high-advantage regions in the sample space.
Assume that the advantage distribution of perturbation $\delta$ at a given state $s$ is defined as $\mu(\delta | s) = \mathbb{E}_{a \sim \pi(\cdot|s)} [\frac{\exp(\Lambda(s,a, \delta))}{\int_{\delta' \in \Delta} \exp(\Lambda(s, a, \delta')) d\delta'} ]$.
Based on this advantage distribution, we can obtain the entropy term and the advantage regression term as follows:
\begin{equation}
    \mathcal{H}(\mu(\delta \mid s))
= -\int_{\Delta} \mu(\delta \mid s)\,\log\bigl(\mu(\delta \mid s)\bigr)\,d\delta,
\end{equation}
\begin{equation}
    \mathcal{L}_{\text{reg}} = \mathbb{E}_{(s, \delta) \sim D} \left[ \left( \mathcal{F}_\phi(s) - \max_\delta \mu(\delta | s) \cdot \Lambda(s,a, \delta) \right)^2 \right],
\end{equation}
where $\mathcal{F}_\phi(s)$ is the in-sample advantage predicted by the model.
The final objective function is:
\begin{equation}
\label{eq:adv_f}
    \mathcal{L}_\phi = \mathbb{E}_{(s, \delta) \sim D} \left[ \mathcal{L}_{\text{reg}} - \kappa \mathcal{H}(\mu(\delta | s)) \right],
\end{equation}
where $\kappa$ is a hyperparameter that balances the entropy regularization term and the advantage regression term.
By weighting the high-value parts of $\mu(a|s)$, $\mathcal{F}_\phi$ can effectively approximate the in-sample maximum advantage, which enables AAT to generate superior adversarial perturbations.

$\quad$
\\
\textbf{Performance analysis}. We now give a theoretical analysis on the performance of AAT.


\begin{lemma}
\label{lem:ad}(\textbf{The Advantage Performance Difference Lemma}~\cite{NIPS2016_cc7e2b87,kakade2002approximately})
     For any attack policies $\pi$ and $\pi'$, we have $V^\pi(s)-V^{\pi'}(s)=\sum_{t=0}^\infty \gamma^t \mathbb{E}_{\tau \sim \pi}\hat{A}^{\pi'}(s,\delta)$,  where $\hat{A}^{\pi'}$ is the advantage function under policy $\pi'$, and $\tau$ is the trajectory of the policy.
\end{lemma}

Assuming the policy $\pi$ and the behavior policy $ \beta$ satisfy the above lemma, the original advantage estimate $ \hat{A}^\beta(s,a)=Q_\theta(s,a)-V_\psi(s)$ satisfies a uniform bound on the mean error $\epsilon$, where $\max_{s,a}|\hat{A}^\beta(s,a)-A^\beta(s,a)|\le\epsilon$.
Similarly, the error of the weighted advantage satisfies $\max_{s,a}|\tilde{A}^\beta(s,a)-A^\beta(s,a)|\le\epsilon_{\text{new}}$.

\begin{theorem}
    (\textbf{The Weighted Advantage Improvement Theorem}) When the following conditions are satisfied: 1) $\epsilon_{\text{new}}\le \epsilon$; and 2) there exists a constant $\delta\ge0 $ such that \\ $\mathbb{E}_{\tau\sim\pi}\left[\sum_{t=0}^\infty\gamma^t\left(\hat{A}^\beta(s_t,a_t)-\tilde{A}^\beta(s_t,a_t)\right)\right]\le \delta$ and $\delta\le \frac{\epsilon-\epsilon_{\text{new}}}{1-\gamma}$, the new performance lower bound satisfies $\Delta_C^{\text{new}} \ge \Delta_C$.
\end{theorem}

\begin{proof}
   See Appendix~\ref{sup:a}.
\end{proof}

Theorem 1 shows that the performance lower bound $\Delta_C^{\text{new}}$ obtained using the weighted advantage function $ \tilde{A}^\beta(s,a)=\frac{Q_\theta(s,a)-V_\psi(s)}{1+\lambda|Q_\theta(s,a)-V_\psi(s)|}$ not only reduces the estimation error ($\epsilon_{\text{new}}\le\epsilon$) but also ensures a higher performance bound $\Delta_C^{\text{new}} \ge \Delta_c$.
Collectively, these results demonstrate that the weighted advantage function provides a theoretically stronger lower bound for performance improvement compared to standard advantage formulations.



$\quad$
\\
\textbf{Summary}
Algorithm~\ref{alg:aat} outlines the overall workflow of AAT, which consists of a training phase and an attack phase.
During the training phase, AAT first obtains an estimate of the weighted advantage from the offline dataset (Lines 4$\sim$9).
Subsequently, AAT uses the observation encoder to transform each state into an encoded vector $\hat{s}$ (Line 14).
Using the encoded vectors and the estimated weighted advantages, the advantage network is trained to predict the advantage of the adversarial perturbation (Line 15).
Moreover, the decoder is employed to learn perturbation generation by leveraging four elements: the encoded states, actions, perturbation sequences, and estimated weighted advantages (Lines 16 $\sim 17$).
In the attack phase, AAT first uses $\mathcal{F}_\phi$ to predict the maximum advantage based on the encoding state (Lines 30$\sim$31).
Then, AAT utilizes the maximum advantage, state, action, and perturbation sequence to predict the perturbation at time $t$ (Line 32).
Finally, the perturbation ($\delta_{t}$) is added to the original state ($s_{t}$) to form an adversarial example ($s'_{t}$) that is fed into the target policy (Line 33).
Based on $s'_t$, the target policy selects an action, interacts with the environment, and stores the resulting transition (state, action, reward) in the historical sequence buffer (Line 22 and 33).
This attack process is repeated until the episode terminates.

\begin{algorithm}[t]
\caption{AAT}
\label{alg:aat}
\begin{algorithmic}[1]
\STATE \textbf{Input:} Dataset $D$, target policy $\pi$, max steps $T$, state buffer $D_s$, action buffer $D_a$, perturbation buffer $D_\delta$
\STATE \textbf{Initialize:} Observation encoder, Decoder, Q-network $Q_\theta$, V-network $V_\psi$, Advantage network $\mathcal{F}_\phi$

\STATE \textbf{// Stage 1: Weighted Advantage Estimation}
\FOR{each batch from $D$}
    \STATE $y = r_t + \gamma V_\psi(s_{t+1}, a_{t+1})$
    \STATE Update $Q_\theta$ by minimizing $\left(Q_\theta(s_t,a_t,\delta_t)-y\right)^2$
    \STATE Update $V_\psi$ via expectile regression
    \STATE $A_t = Q_\theta(s_t,a_t,\delta_t)-V_\psi(s_t,a_t)$ \textit{// weighted advantage }
    \STATE $\tilde{A}_t = \frac{A_t}{1 + \lambda |A_t|}$ 
\ENDFOR

\STATE \textbf{// Stage 2: Autoregressive Perturbation and Weighted Advantage Learning}
\STATE Sample a trajectory $\tau=(s_0,a_0,\delta_0,\tilde{A}_0,\cdots, \tilde{A}_T)$ from $D$ \textit{ // $\tau$ can be divided into state sequence $\tau^s$, action sequence $\tau^a$ and perturbation sequence $\tau^\delta$}
\FOR{t = 0 to T}
    \STATE Encode state: $\tau^{\hat{s}}_{\leq t} = \text{Observation encoder} (\tau^s_{\leq t})$ \textit{// $\tau^{\hat{s}}_{\leq t}$ includes the state sequence from $\hat{s}_0$ to $\hat{s}_t$}
    \STATE Using $\tilde{A}$ and $\hat{s}$ train advantage network $\mathcal{F}_\phi$ 
    \STATE Predict perturbation: $\delta_t = \text{Decoder}(\tau^{\hat{s}}_{\leq t}, \tau^a_{<t},\tau^\delta_{<t}, \tilde{A}_t)$
    \STATE Train generator: $\mathcal{L} = (\pi(s_t + \delta_t)-a_t)^2 + \|\delta_t\|_2$
\ENDFOR

\STATE \textbf{// Stage 3: Attack}
\FOR{t = 0 to T}
    \STATE Observe $s_t$ 
    \STATE $D_s \leftarrow s_t$ \textit{// $s_t$ is stored in $D_s$}
    \STATE $\tau^s_{\leq t}\leftarrow D_s$ \textit{// Get sequences ($s_{_0}$, $\cdots$, $s_{t}$)}
    \IF{t = 0}
        \STATE $\tau^a_{< t}\leftarrow 0$, $\tau^\delta_{< t}\leftarrow 0 $
    \ELSE
       \STATE $\tau^a_{< t}\leftarrow D_a$ \textit{// Get sequences ($a_{_0}$, $\cdots$, $a_{t-1}$)}
       \STATE $\delta_{< t}\leftarrow D_\delta$ \textit{// Get sequences ($\delta_{_0}$, $\cdots$, $\delta_{t-1}$)}
    \ENDIF
    \STATE $\tau^{\hat{s}}_{\leq t} = \text{Observation encoder} (\tau^s_{\leq t})$  \textit{//Get the state encoding vector through the encoder}
    \STATE $\tilde{A}_t = \mathcal{F}_\phi(\hat{s}_t)$ \textit{// predicted max advantageous }
    \STATE $\delta_t = \text{Decoder}(\tau^{\hat{s}}_{\leq t}, \tau^a_{<t},\tau^\delta_{<t}, \tilde{A}_t)$
    \STATE $s'_t = s_t + \delta_t$, \quad $a_t = \pi(s'_t)$
    \STATE $D_a \leftarrow a_t $, $D_\delta \leftarrow \delta_t$ \textit{// add action and perturbation}
    \STATE Step environment
\ENDFOR

\end{algorithmic}
\end{algorithm}

\section{Experiments}
To evaluate the effectiveness of the proposed AAT attack on DRL agents, we conduct experiments on six Atari 2600 games: Breakout, Pong, Chopper Command, Sequest, Qbert and Space Invaders, utilizing the OpenAI Gym environment\footnote{https://github.com/openai/gym}.
In addition, we further validate AAT in environments with continuous action spaces (DeepMind Control Suite) and high-dimensional state spaces (Google Football). Refer to Appendix~\ref{sup:env} for more detailed environment specifications.

\subsection{Experimental setup}
We use DQN~\cite{DBLP:MnihKSGAWR13}, A3C~\cite{DBLP:SharifBBR16}, TRPO~\cite{DBLP:SchulmanLAJM15}, DDPG~\cite{DBLP:abs-1801-00690}, and PPO~\cite{DBLP:SchulmanWDRK17} as the target policies.
Since the goal of AAT is to generate time-correlated adversarial perturbations, we compare AAT with the reward-based attack baselines \textbf{PA-AD}~\cite{DBLP:SunZLH22}, \textbf{AdvRL-GAN}~\cite{DBLP:YuS22}, \textbf{TSGE}~\cite{DBLP:QiaobenYZSZZ24} and \textbf{PIA}~\cite{barto1988neuronlike}, which construct attack policies via RL to generate optimal adversarial perturbations.
Note that some recent trajectory-level attack methods, such as MBRL-Attack~\cite{chen2023dynamics} and CAAD~\cite{yamabe2024robust}, are not included in our baseline comparisons due to the unavailability of their open-source implementations.
We also compare ATT with mainstream gradient-based attack methods, including FGSM~\cite{DBLP:HuangPGDA17}, SKip~\cite{DBLP:KosS17}, S-T~\cite{DBLP:LinHLS0S17} and EDGE~\cite{DBLP:GuoWKX21}.
Refer to Appendix~\ref{sec:train_data} and~\ref{sup:struct} for hyperparameter settings and network structure for AAT.
Moreover, the structure and input settings of the target network are shown in Appendix~\ref{sup:target_policy}.

\begin{table*}[t]
\centering
\caption{Comparisons between the ATT attack and other methods in the Breakout game. $*$ denotes the cumulative reward of the target policy under no attack. The bold indicates the best attack performance. Values represent the average cumulative reward (in game points) of the target policy under attack, calculated over 10 independent evaluation runs. Lower values indicate better attack performance. }
\label{tab:breakout_w}
\begin{tabular}{lcccccccc}
\hline
\multicolumn{9}{c}{Breakout}                                                                                                                                                                                                                                                                                                                                                                                                                                                                                                                                                                                                                                                                                                                                                                                                                                                                                                                                                                                                                                                                                                                                                                                                                                                                                                                                                                                                                                                                                                                                                     \\ \hline
\multicolumn{1}{l|}{}                                                                                                            & \multicolumn{4}{c|}{White-box}                                                                                                                                                                                                                                                                                                                                                                                                                                                                                                                                                                                                                                                                                           & \multicolumn{4}{c}{Black-box}                                                                                                                                                                                                                                                                                                                                                                                                                                                                                                                                                                                                                                                                                      \\
\multicolumn{1}{l|}{Method}                                                                                                      & DQN                                                                                                                                                                    & A3C                                                                                                                                                                    & TRPO                                                                                                                                                                    & \multicolumn{1}{c|}{PPO}                                                                                                                                                                     & DQN                                                                                                                                                                       & A3C                                                                                                                                                                        & TRPO                                                                                                                                                                         & PPO                                                                                                                                                                        \\
\multicolumn{1}{l|}{\begin{tabular}[c]{@{}l@{}}*\\ FGSM\\ Skip\\ S-T\\ EDGE\\ PA-AD\\ AdvRL-GAN\\ TSGE\\ PIA\\ AAT\end{tabular}} & \begin{tabular}[c]{@{}c@{}}355.83$\pm$0.89\\ 34.30$\pm$3.78\\ 31.45$\pm$4.15\\ 57.68$\pm$5.61\\ 43.26$\pm$3.51\\ 46.35$\pm$4.31\\ 32.44$\pm$5.42\\ \textbf{6.42}$\pm$\textbf{4.21}\\ 35.17$\pm$4.44\\ 6.56$\pm$2.54\end{tabular} & \begin{tabular}[c]{@{}c@{}}384.52$\pm$1.45\\ 35.45$\pm$4.05\\ 32.46$\pm$4.33\\ 55.44$\pm$2.78\\ 46.32$\pm$3.49\\ 53.87$\pm$3.27\\ 9.11$\pm$2.56\\ 15.29$\pm$3.42\\ 27.84$\pm$4.23\\ \textbf{8.21}$\pm$\textbf{3.51}\end{tabular} & \begin{tabular}[c]{@{}c@{}}415.64$\pm$2.73\\ 34.58$\pm$4.54\\ 38.79$\pm$3.98\\ 58.47$\pm$3.12\\ 44.89$\pm$4.27\\ 47.96$\pm$4.32\\ 33.98$\pm$3.75\\ 25.46$\pm$4.21\\ 40.51$\pm$3.41\\ \textbf{5.98}$\pm$\textbf{3.81}\end{tabular} & \multicolumn{1}{c|}{\begin{tabular}[c]{@{}c@{}}426.58$\pm$2.14\\ 40.56$\pm$3.82\\ 38.42$\pm$4.08\\ 44.36$\pm$3.84\\ 50.33$\pm$4.35\\ 65.73$\pm$3.91\\ 38.21$\pm$2.74\\ 34.32$\pm$3.27\\ 53.47$\pm$2.83\\ \textbf{6.88}$\pm$\textbf{2.56}\end{tabular}} & \begin{tabular}[c]{@{}c@{}}355.83$\pm$0.89\\ 118.32$\pm$8.32\\ 88.46$\pm$7.74\\ 68.32$\pm$5.24\\ 56.88$\pm$6.13\\ 98.48$\pm$7.88\\ 40.60$\pm$9.56\\ 43.73$\pm$6.02\\ 53.45$\pm$7.23\\ \textbf{28.32}$\pm$\textbf{6.21}\end{tabular} & \begin{tabular}[c]{@{}c@{}}384.52$\pm$1.45\\ 120.42$\pm$9.84\\ 90.34$\pm$9.47\\ 68.21$\pm$8.54\\ 60.36$\pm$7.43\\ 123.63$\pm$6.83\\ 45.28$\pm$7.47\\ 53.46$\pm$8.23\\ 48.72$\pm$9.01\\ \textbf{20.36}$\pm$\textbf{7.14}\end{tabular} & \begin{tabular}[c]{@{}c@{}}412.64$\pm$2.73\\ 128.24$\pm$4.32\\ 92.34$\pm$10.56\\ 70.24$\pm$11.05\\ 59.86$\pm$9.35\\ 113.74$\pm$8.46\\ 46.68$\pm$9.37\\ 52.61$\pm$7.83\\ 63.47$\pm$8.42\\ \textbf{22.22}$\pm$\textbf{8.73}\end{tabular} & \begin{tabular}[c]{@{}c@{}}426.58$\pm$2.14\\ 130.26$\pm$7.68\\ 95.42$\pm$8.75\\ 69.21$\pm$9.42\\ 60.84$\pm$8.35\\ 99.87$\pm$9.84\\ 48.23$\pm$7.81\\ 89.36$\pm$8.43\\ 100.47$\pm$9.41\\ \textbf{20.21}$\pm$\textbf{7.21}\end{tabular} \\ \hline
\end{tabular}
\end{table*}

\subsection{Main results}
\noindent \textbf{In discrete action environments}. 
Table~\ref{tab:breakout_w} compares the attack performance of AAT with other approaches in the Breakout game in terms of cumulative rewards under attack obtained using target policies (i.e., DQN, A3C, TRPO, and PPO).
The results of other games are provided in Appendix~\ref{sup:discrete_reult}.
All experimental results are averaged over 10 independent trials.

In the white-box scenario, all methods effectively diminish the cumulative rewards of the target policy.
For example, FGSM reduces the cumulative rewards of the target policy across different games by exploiting its gradient information. 
Since the localization of critical states is challenging and directly impacts the attack effect, the attack effects of Skip, S-T, and EDGE are not as good as those of FGSM.
Although PA-AD finds the optimal attack direction through policy perturbation, it only focuses on the current state that causes the generated perturbation to be independent of the previous perturbation, which reduces the attack effect.
Unlike previous methods for generating adversarial perturbations, AAT employs sequence modeling techniques to effectively capture the relationship between historical information and existing perturbations to generate time-correlated adversarial examples.
Moreover, AAT can dynamically adjust adversarial perturbations through the weighted advantage mechanism, which further enhances the attack performance.

Compared to white-box experiments, black-box experiments are more challenging because the attacker cannot access any information about the target policy.
In order to ensure the fairness of the experiment, all attack methods (e.g, FGSM, Skip and AdvRL-GAN) use a substitution policy to generate examples, which are then used to attack the target policy. 
As can be seen from Table~\ref{tab:breakout_w}, all methods can reduce the attack effectiveness compared to the white-box scenario.
Specifically, the FGSM, Skip, S-T, and EDGE methods exhibit the most significant decline in attack effectiveness.
Two primary reasons account for this phenomenon: 1) the adversarial perturbations computed via substitution policy gradients fail to align with decision-making patterns of the black-box policy, and 2) adversarial examples tend to overfit the substitution policy~\cite{DBLP:RiceWK20}, which results in poor transferability to the black-box policy.
AdvRL-GAN is effective in black-box scenarios by training a generative adversarial network using TRPO to generate adversarial examples.
Unlike the aforementioned methods, by utilizing the MSCSA mechanism and the weighted advantage mechanism, AAT can generate effective adversarial perturbations for sequential decision-making in black-box scenarios.

Moreover, the experimental results show that the standard deviations of various attack methods are larger in the black-box scenario, indicating more pronounced performance instability. This phenomenon may be attributed to the discrepancies between the substitution policy and the target policy. Despite this underlying cause of instability, AAT overall remains more stable than the benchmark methods in black-box scenarios.

$\quad$
\\
\noindent \textbf{In continuous action environments}.
To assess the effectiveness of attacks in continuous action environments, we conduct comparative experiments in the DeepMind Control Suite.
Fig.~\ref{fig:app_cons_a3c} displays the results of various attack methods on the target policy D4PG in both white-box and black-box scenarios (see Appendix~\ref{sup:discrete_reult} for additional results).
All results are the average of 10 experiments.

Experimental results show that AAT exhibits superior attack performance in both white-box and black-box scenarios, while the attack performance of other methods has large fluctuations in different environments or attack scenarios.
For example, FGSM generates adversarial perturbations based on gradient information, which results in relatively good attack performance in white-box scenarios across all environments.
However, FGSM has significant performance degradation due to attack overfitting in black-box scenarios, especially in the Humanoid, Pendulum, and Walker environments.
Skip and S-T perform worse than other methods, primarily because they struggle to identify critical frames and establish suitable attack thresholds.
EDGE uses the C$\&$W attack method, which exhibits high stability in black-box scenarios.
By generating hidden space perturbations, AdvRL-GAN can achieve better attack results in most environments, but still show instability in black-box Humanoid environment.
Although TSGE generates perturbations using a deception policy, these perturbations fail to effectively induce the target (victim) policy to take suboptimal actions because of the structural discrepancy between the two policies. 
PIA introduces adversarial strategies through a zero-sum game to disrupt the decision-making of the policy, but the constraint of minimizing the distribution distance between attack observations and unattacked transitions strictly limits the actions of the adversarial strategy, leading to suboptimal attack performance. 
Compared to the above methods, AAT outperforms all methods, consistently demonstrating superior attack performance and unmatched stability in the white-box scenario.

In the black-box scenario, the estimation of the substitution model often contains bias. 
Consequently, attack methods that rely on such models to craft adversarial examples suffer from performance degradation.
However, the weighted advantage mechanism constrains the advantage estimation range for unobserved states, thereby mitigating overestimation and preventing AAT from being misled by erroneous peaks in the surrogate model. 
Beyond this corrective function, the weighting mechanism proactively identifies effective perturbation segments across trajectories and generalizes these perturbations to similar states. 
Consequently, even with a biased substitution model, AAT can achieve stronger attack performance in the black-box scenario by leveraging this weighted sampling to capture locally valid attack signals.

\begin{figure}[]
  \centering
  \subfigure{\includegraphics[width=0.44\textwidth]{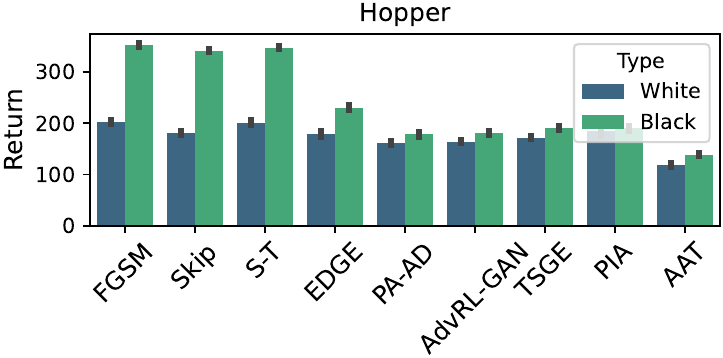}}
  \subfigure{\includegraphics[width=0.44\textwidth]{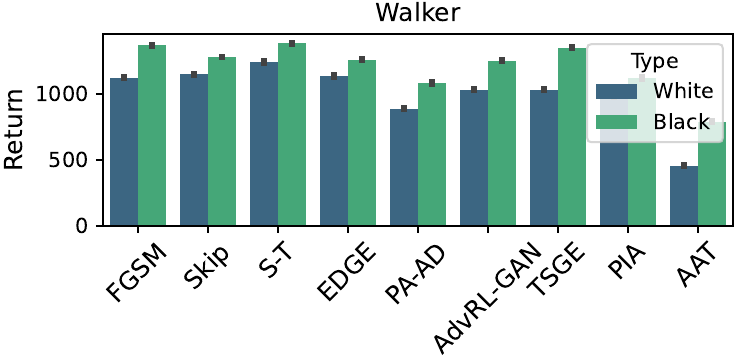}}
  \caption{All experimental results are the average of 10 experiments. In the black-box experimental results, we generate adversarial examples using A3C and D4PG as substitution strategies, respectively.}
  \label{fig:app_cons_a3c}
\end{figure}

\begin{table*}[t]
\centering
\caption{Comparisons between the ATT attack and other methods in the Pong game. $*$ indicates the cumulative reward when the target strategy adopts a defensive strategy but is not attacked.}
\resizebox{0.98\linewidth}{!}{
\setlength{\tabcolsep}{0.5mm} 
\begin{tabular}{lcccccccccccc}
\hline
\multicolumn{13}{c}{Pong}                                                                                                                                                                                                                                                                                                                                                                                                                                                                                                                                                                                                                                                                                                                                                                                                                                                                                                                                                                                                                                                                                                                                                                                                                                                                                                                                                                                                                                                                                                                                                                                                                                                                                                                                                                                                                                                                                                                                                                                                                                                                                                                                                                                                                                                        \\ \hline
                                                                                                                    & \multicolumn{4}{c|}{Adversarial training}                                                                                                                                                                                                                                                                                                                                                                                                                                                                                                                                                                                                                                                                                  & \multicolumn{4}{c|}{SA-MDP}                                                                                                                                                                                                                                                                                                                                                                                                                                                                                                                                                                                                                                                                                             & \multicolumn{4}{c}{POTBSG}                                                                                                                                                                                                                                                                                                                                                                                                                                                                                                                                                                                                                                                                          \\ \hline
\multicolumn{1}{l|}{Method}                                                                                         & DQN                                                                                                                                                                      & A3C                                                                                                                                                                    & TRPO                                                                                                                                                                    & \multicolumn{1}{c|}{PPO}                                                                                                                                                                     & DQN                                                                                                                                                                     & A3C                                                                                                                                                                    & TRPO                                                                                                                                                                   & \multicolumn{1}{c|}{PPO}                                                                                                                                                                    & DQN                                                                                                                                                                     & A3C                                                                                                                                                                    & TRPO                                                                                                                                                                    & PPO                                                                                                                                                                    \\
\multicolumn{1}{l|}{\begin{tabular}[c]{@{}l@{}}*\\ FGSM\\ Skip\\ S-T\\ EDGE\\ PA-AD\\ AdvRL-GAN\\ AAT\end{tabular}} & \begin{tabular}[c]{@{}c@{}}17.88$\pm$0.33\\ 16.08$\pm$0.32\\ 16.11$\pm$1.45\\ 15.12$\pm$0.86\\ 10.12$\pm$1.33\\ -4.67$\pm$1.37\\ -3.56$\pm$1.54\\ -\textbf{10.21}$\pm$\textbf{0.67}\end{tabular} & \begin{tabular}[c]{@{}c@{}}20.10$\pm$0.40\\ 19.19$\pm$0.58\\ 18.11$\pm$0.64\\ 19.98$\pm$1.47\\ 5.34$\pm$1.57\\ -6.78$\pm$1.37\\ -2.41$\pm$1.43\\ -\textbf{8.15}$\pm$\textbf{1.52}\end{tabular} & \begin{tabular}[c]{@{}c@{}}19.92$\pm$0.35\\ 18.78$\pm$0.64\\ 18.13$\pm$0.37\\ 18.68$\pm$0.62\\ 3.73$\pm$1.35\\ -7.36$\pm$1.05\\ -5.76$\pm$0.94\\ -\textbf{12.06}$\pm$\textbf{1.34}\end{tabular} & \multicolumn{1}{c|}{\begin{tabular}[c]{@{}c@{}}20.96$\pm$0.05\\ 18.86$\pm$0.74\\ 19.38$\pm$1.01\\ 20.11$\pm$0.32\\ -2.98$\pm$1.84\\ -4.62$\pm$2.01\\ -5.46$\pm$1.51\\ -\textbf{7.98}$\pm$\textbf{1.18}\end{tabular}} & \begin{tabular}[c]{@{}c@{}}18.88$\pm$1.03\\ 17.38$\pm$0.24\\ 16.11$\pm$1.36\\ 17.32$\pm$0.51\\ 13.42$\pm$1.41\\ -1.41$\pm$1.74\\ -2.27$\pm$1.35\\ -\textbf{8.21}$\pm$\textbf{1.34}\end{tabular} & \begin{tabular}[c]{@{}c@{}}19.12$\pm$0.35\\ 19.07$\pm$0.91\\ 18.74$\pm$1.07\\ 17.18$\pm$0.67\\ 9.74$\pm$1.35\\ -\textbf{3.41}$\pm$\textbf{2.17}\\ -1.21$\pm$2.36\\ -1.48$\pm$2.41\end{tabular} & \begin{tabular}[c]{@{}c@{}}20.02$\pm$0.41\\ 18.43$\pm$0.5\\ 19.13$\pm$1.32\\ 17.76$\pm$1.54\\ 10.54$\pm$2.07\\ -8.13$\pm$1.87\\ -6.26$\pm$1.97\\ -\textbf{9.47}$\pm$\textbf{1.84}\end{tabular} & \multicolumn{1}{c|}{\begin{tabular}[c]{@{}c@{}}20.16$\pm$0.51\\ 19.99$\pm$0.84\\ 18.31$\pm$1.54\\ 18.21$\pm$2.07\\ 13.78$\pm$1.74\\ 1.32$\pm$1.98\\ -\textbf{3.46}$\pm$\textbf{2.13}\\ -3.24$\pm$2.07\end{tabular}} & \begin{tabular}[c]{@{}c@{}}17.96$\pm$1.52\\ 16.63$\pm$1.69\\ 16.41$\pm$2.04\\ 14.12$\pm$1.87\\ 12.42$\pm$1.64\\ -6.21$\pm$2.07\\ -9.56$\pm$1.73\\ -\textbf{9.74}$\pm$\textbf{2.13}\end{tabular} & \begin{tabular}[c]{@{}c@{}}18.74$\pm$0.25\\ 16.04$\pm$0.74\\ 17.41$\pm$1.86\\ 16.13$\pm$1.62\\ 8.44$\pm$1.73\\ -\textbf{7.37}$\pm$\textbf{1.61}\\ -6.19$\pm$1.87\\ -4.15$\pm$1.92\end{tabular} & \begin{tabular}[c]{@{}c@{}}19.12$\pm$0.53\\ 17.01$\pm$1.02\\ 17.19$\pm$0.98\\ 18.18$\pm$1.41\\ 18.41$\pm$1.23\\ -5.46$\pm$1.17\\ -7.21$\pm$1.32\\ -\textbf{9.67}$\pm$\textbf{1.12}\end{tabular} & \begin{tabular}[c]{@{}c@{}}19.46$\pm$0.34\\ 17.43$\pm$2.4\\ 18.74$\pm$1.52\\ 18.22$\pm$2.14\\ -4.33$\pm$1.52\\ -6.74$\pm$1.74\\ -\textbf{7.88}$\pm$\textbf{1.32}\\ -7.32$\pm$2.07\end{tabular} \\ \hline
\end{tabular}
}
\label{tab:defen}
\end{table*}

$\quad$
\\
\textbf{Defense evaluation}.
We conduct comparative experiments to investigate the attack performance of AAT against target policies with defensive capabilities.
Specifically, we implement three strategies to improve policy robustness: (1) all target policies adopt adversarial training~\cite{DBLP:Wu0WX21} to enhance target policy resistance to adversarial examples; (2) All target policies implement policy regularization via the SA-MDP~\cite{DBLP:0001CX0LBH20} theoretical framework, thereby enhancing policy robustness; (3) All target policies achieve robustness defense through a random Stackelberg game and by solving the equilibrium of a partially observable turn-based stochastic game (i.e., POTBSG~\cite{DBLP:McMahanW0X24}).
Table~\ref{tab:defen} presents the average results of the AAT attack conducted 10 times under 10 different random seeds.
The results show that AAT can effectively reduce the performance of the target policies in the setting of defensive adversarial training.
FGSM, Skip, and S-T fail to reduce the performance significantly, exhibiting minimal impact on adversarially trained target policies. 
Although EDGE and AdvRL-GAN reduce the performance of target policies, they remain suboptimal compared to AAT.
For the SA-MDP and POTBSG defenses, AAT can still effectively reduce the cumulative reward of the target policy overall.
Overall, while existing defense methods can resist FGSM, Skip, S-T and EDGE attacks, and partially mitigate PA-AD and AdvRL-GAN attacks, they fail to effectively resist AAT attacks.

To further verify the defensive capability of existing methods against AAT, we adopt two classical defense strategies: (1) applying spatial smoothing (i.e., median filtering, MF) to blur input states and suppress image perturbations, and (2) reconstructing input states with MimicDiffusion~\cite{DBLP:SongLP024} to remove adversarial noise.
The experimental results (see Appendix~\ref{sup:defense_eval}) indicate that although MF and MimicDiffusion provide better defense than adversarial training, AAT can still effectively reduce the performance of the target policy.

\begin{figure*}[t]
    \centering
    \subfigure[\tiny White-box (DQN)]{\includegraphics[width=0.24\textwidth]{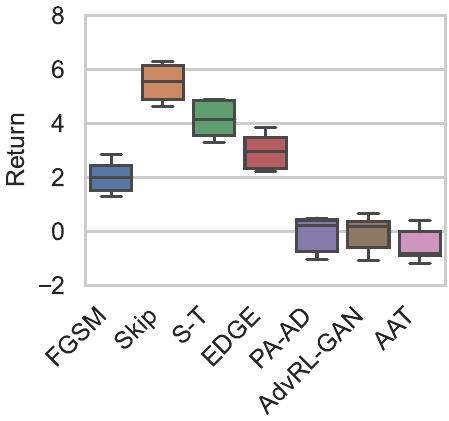}}
    \subfigure[\tiny Black-box (DQN)]{\includegraphics[width=0.24\textwidth]{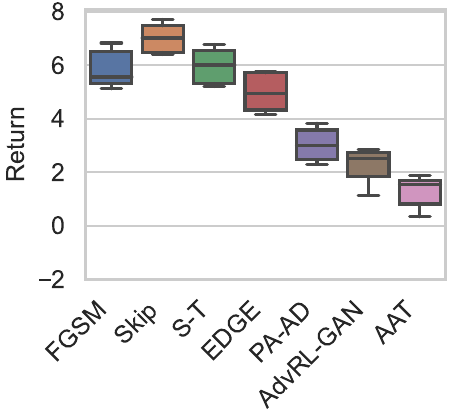}}
    \subfigure[\tiny White-box (PPO)]{\includegraphics[width=0.24\textwidth]{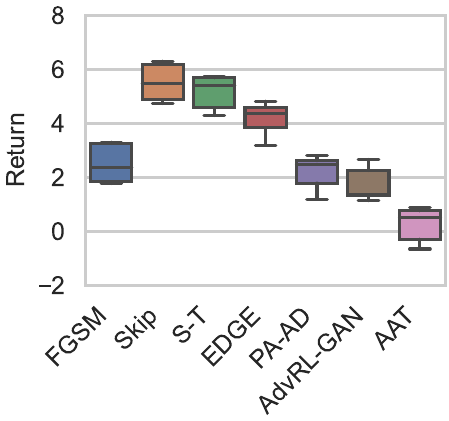}}
    \subfigure[\tiny Black-box (PPO)]{\includegraphics[width=0.24\textwidth]{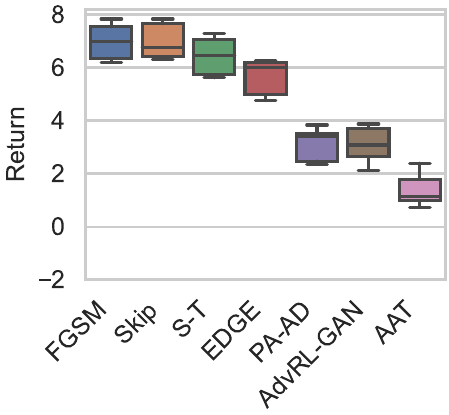}}
    \caption{The performance of different attack methods in Gfootball. The expected cumulative rewards achieved by DQN and PPO target policies are 7.12 and 7.88 respectively. }
    \label{fig:gfootball_box}
\end{figure*}

\begin{figure}[t]
    \centering
    \subfigure{\includegraphics[width=0.48\textwidth]{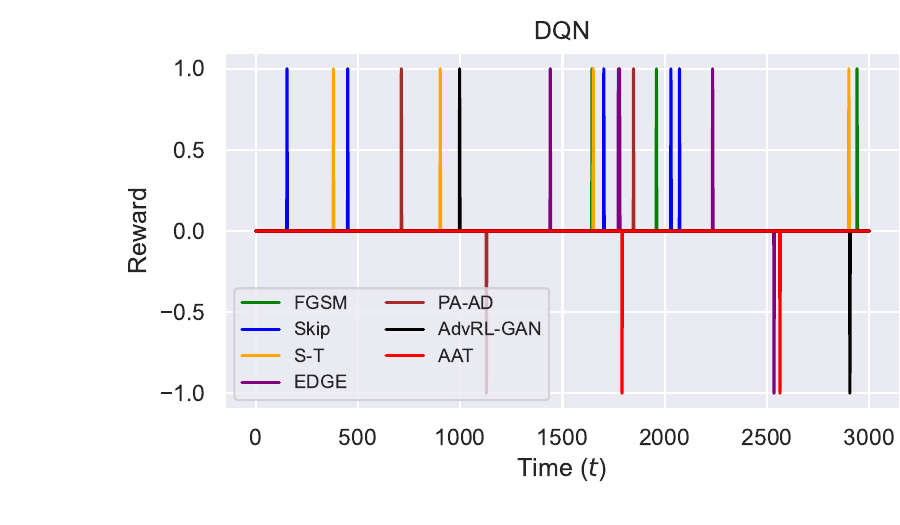}}
    \subfigure{\includegraphics[width=0.48\textwidth]{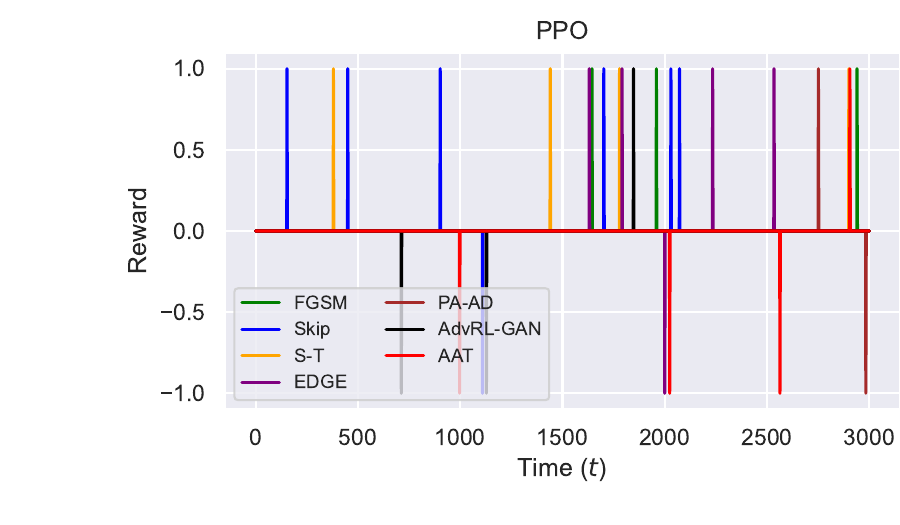}}
  \caption{Visualization of immediate  rewards for target policy. The x-axis represents the moments from the start to the end of the target policy's execution in the Gfootball environment, and the y-axis is the reward value obtained at each time step.}
  \label{fig:reward_image}
\end{figure}

$\quad$
\\
\textbf{Time-correlated evaluation}. 
To explore the benefits of the time-correlated adversarial perturbations of AAT, we conduct comparative experiments in the Gfootball environment.
As shown in Fig.~\ref{fig:gfootball_box},
Skip and S-T methods exhibit lower attack performance because their overemphasis on attack frequency leads to a failure in effectively reducing the cumulative reward of the target policy.
FGSM exhibits superior attack efficacy in white-box scenarios, but not in black-box scenarios where information about the target policy is unavailable.
Although both PA-AD and ADV employ RL to learn an attack policy for identifying optimal perturbations, they fail to achieve optimal attack performance due to overlooking the influence of historical information.
Compared with these methods, AAT generates adversarial perturbations based on both the current state and historical data to mislead the target policy while adaptively fine-tuning perturbations based on the advantage of current inputs, thereby achieving superior attack performance in both white-box and black-box scenarios.

Furthermore, during the process of attacking the target policy, we observe that gradient-based attack methods (e.g., FGSM, Skip, and S-T) often cause the ball handler to lose control of the ball.
In contrast, the AAT method tends to force ball handlers into passing to opponents, accelerating scoring opportunities for the opposing team.
Fig.~\ref{fig:reward_image} illustrates the immediate rewards received by the target policy at each time step under different attack methods.
The adversarial perturbations generated by FGSM, Skip, S-T, and EDGE methods reduce the immediate rewards of the target policy to zero.
However, the target policy achieves positive rewards at most time steps due to its robustness.
While AdvRL-GAN can mislead the target into receiving negative rewards, it fails to achieve optimal attacks due to neglecting the impact of historical information on current perturbations.
In contrast, adversarial perturbations generated by AAT result in more negative rewards for the target policy, and consequently reducing the cumulative returns of the target policy.

$\quad$
\\
\noindent \textbf{Generalizability evaluation}.
To verify the generalization capability of AAT, we conduct experiments in cooperative multi-agent environments based on the experimental setup described in FGSMLW~\cite{DBLP:LinDZLP20}.       
Our evaluation employs QMIX as the target policy and uses the 2s3z map in the StarCraft Multi-Agent Challenge (SMAC).
In this scenario, each cooperative multi-agent team comprises five units: two Stalkers and three Zealots.
Within this framework, the attacker targets a single Stalker as the victim agent, perturbing its observations with the objective of degrading the team's overall cooperative performance.
The experiment results (see Appendix~\ref{sup:genera_multi}) show that AAT significantly reduces both the team reward and win rate more effectively than a range of baseline methods, including EDGE, PA-AD, AdvRL-GAN, TSGE, and PIA.
Furthermore, in terms of reward reduction, AAT even surpasses the FGSMLW method, which specializes in attacking c-MARL algorithms.
More multi-agent settings (e.g., competitive environments) will be the focus of our subsequent research.

\subsection{Adversarial perturbations concealment}
\noindent \textbf{Adversarial example visualization}.
Adversarial examples aim to mislead the policy into making incorrect decisions by introducing imperceptible perturbations that minimally alter the original input to evade detection.
To analyze these perturbations, we visually contrast adversarial examples with their original examples and quantify the perturbation intensity using mean squared error (MSE) and structural similarity index metric (SSIM)~\cite{DBLP:WangBSS04}. 
The results visualized in Appendix~\ref{sup:adver_noise_conceal} show that AAT generates adversarial examples with minimal perturbations, while other attack methods generate adversarial examples that exhibit significant differences from the original states.

$\quad$
\\
\noindent \textbf{Adversarial example detection}.
To further validate the stealthiness of the adversarial examples generated by AAT, we employ two detection methods: the Expected Perturbation Score-based adversarial detection (EPS-AD)~\cite{DBLP:ZhangLYYL0T23} and Layer Regression (LR)~\cite{DBLP:MumcuY25}. 
EPS-AD quantifies the difference between test samples and natural samples using the maximum mean discrepancy of the expected perturbation score.
LR identifies adversarial examples by analyzing how attacks affect different DNN layers to varying degrees.
In our experiments, EPS-AD and LR detect 1,000 adversarial states generated by each attack method. 
We use the detection rate, defined as $r = \frac{n}{m}$, where $n$ is the number of successfully detected adversarial examples and $m$ is the total number tested.
A low detection rate shows that the attack has strong evasion capability, indicating that the generated adversarial examples have better stealthiness.

Table~\ref{tab:sup_detection_results} in Appendix~\ref{sup:adver_noise_conceal} presents the detection performance of two detection methods (EPS-AD and LR) on states generated by various adversarial attack models across three Atari environments (Pong, Seaquest, and Qbert).
This result indicates that AAT can effectively evade adversarial detection mechanisms.
More detailed analysis of experimental results is provided in Appendix~\ref{sup:adver_noise_conceal}.

\subsection{Efficiency exploration}
\label{sup:effic_explor}
We explore the time efficiency of generating adversarial examples for FGSM, Skip, S-T, EDGE, AdvRL-GAN, and AAT.
Table~\ref{tab:time_sample} provides the time required for these methods to generate an adversarial example in the Breakout, Pong, Seaquest, Humanoid, Hopper, and Walker environments. 
This time is calculated as the total generation time divided by the number of adversarial examples.
We can see that FGSM calculates adversarial perturbations based on gradients multiple steps, hence it is comparatively time-consuming.
Skip and S-T methods are slower because of critical frame calculations and threshold comparisons. 
EDGE efficiently provides key trajectories through linear Gaussian methods, but the C$\&$W method is time-consuming due to multiple optimization processes. 
AdvRL-GAN generates examples via a RL method but takes a long time due to the need to rebuild examples using a GAN.
Compared to the above methods, AAT directly outputs adversarial perturbations based on current input information without a complex computational process, thus it can quickly generate adversarial examples.

\begin{table}[]
\caption{The average time to generate an adversarial examples. Note that the time unit is milliseconds.}
\label{tab:time_sample}
\centering
\setlength{\tabcolsep}{0.8mm} 
\begin{tabular}{lccccll}
\hline
Method & Pong & Breakout & Seaquest & Humanoid & Hopper & Walker \\ \hline
\begin{tabular}[c]{@{}l@{}}FGSM\\ Skip\\ S-T\\ EDGE\\ AdvRL-GAN\\ AAT\end{tabular} & 
\begin{tabular}[c]{@{}c@{}}5.3\\ 8.5\\ 13.4\\ 12.8\\ 14.2\\ \textbf{1.56}\end{tabular} & 
\begin{tabular}[c]{@{}c@{}}8.4\\ 13.7\\ 9.8\\ 16.6\\ 18.8\\ \textbf{2.2}\end{tabular} & 
\begin{tabular}[c]{@{}c@{}}4.3\\ 7.2\\ 5.5\\ 9.4\\ 9.83\\ \textbf{2.4}\end{tabular} & 
\begin{tabular}[c]{@{}c@{}}6.3\\ 8.4\\ 8.3\\ 4.1\\ 10.58\\ \textbf{1.5}\end{tabular} & 
\begin{tabular}[c]{@{}l@{}}9.6\\ 13.2\\ 14.1\\ 13.5\\ 15.4\\ \textbf{2.1}\end{tabular} & 
\begin{tabular}[c]{@{}l@{}}7.8\\ 12.9\\ 15.6\\ 16.4\\ 16.8\\ \textbf{2.6}\end{tabular} \\ \hline
\end{tabular}
\end{table}

\subsection{Training data analysis}
\noindent\textbf{Training data quality}.
In this section, we investigate the impact of training data quality on AAT performance by analyzing the cumulative rewards of the target policy in Seaquest and Qbert.
The training data is divided into expert data, medium data, and random data, with each category containing 20,000 trajectories.
Specifically, expert data comprises trajectories where the cumulative reward is reduced by more than 80$\%$ with the attack of FGSM, medium data includes trajectories where the cumulative reward reduction falls between 50$\%$ and 80$\%$, and random data consists of trajectories obtained by randomly attacking the target strategy state, with reward reductions ranging from 0$\%$ to 50$\%$.

\begin{table}[t]
\centering
\caption{The impact of different data quality on AAT performance. The table displays the mean values obtained from 10 attack experiments.}
\setlength{\tabcolsep}{1.5mm}
\begin{tabular}{l|l|cccc}
\hline
Environments & Data type                                                                       & DQN                                                               & A3C                                                               & TRPO                                                              & PPO                                                               \\ \hline
Seaquest     & \begin{tabular}[c]{@{}l@{}}Expert data\\ Medium data\\ Random data\end{tabular} & \begin{tabular}[c]{@{}c@{}}\textbf{359.82}\\ 367.41\\ 2325.65\end{tabular} & \begin{tabular}[c]{@{}c@{}}\textbf{353.64}\\ 373.56\\ 2354.29\end{tabular} & \begin{tabular}[c]{@{}c@{}}360.23\\ \textbf{359.89}\\ 2332.79\end{tabular} & \begin{tabular}[c]{@{}c@{}}359.67\\ 393.45\\ 2467.83\end{tabular} \\ \hline
Qbert        & \begin{tabular}[c]{@{}l@{}}Expert data\\ Medium data\\ Random data\end{tabular} & \begin{tabular}[c]{@{}c@{}}\textbf{236.89}\\ 379.46\\ 5020.57\end{tabular} & \begin{tabular}[c]{@{}c@{}}\textbf{183.92}\\ 245.67\\ 5596.35\end{tabular} & \begin{tabular}[c]{@{}c@{}}\textbf{203.89}\\ 299.67\\ 4478.36\end{tabular} & \begin{tabular}[c]{@{}c@{}}\textbf{211.67}\\ 374.56\\ 6340.27\end{tabular} \\ \hline
\end{tabular}
\label{tab:data_qulity}
\end{table}

Table~\ref{tab:data_qulity} presents the adversarial attack results of AAT on the target policies. 
In the clean state of Seaquest, the cumulative rewards achieved by DQN, A3C, TRPO, and PPO are 2385.32, 2856.89, 1856.89, and 2108.38, respectively. 
Similarly, the cumulative rewards achieved by DQN, A3C, TRPO, and PPO in clean state Qbert are 17890.93, 21840.35, 22840.56, and 24583.68, respectively.
As can be seen, in the Seaquest game with medium data, AAT can still achieve efficient attacks under the guidance of the advantage function, even surpassing the performance of expert data when attacking TRPO.

\begin{figure}
    \centering
    \includegraphics[width=0.95\linewidth]{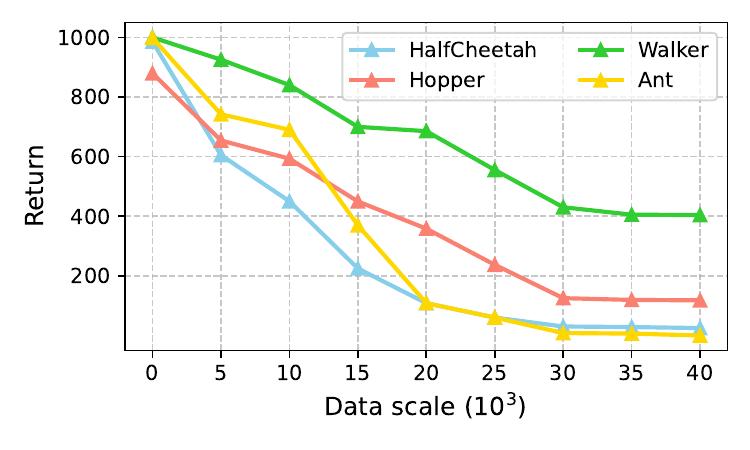}
    \caption{The impact of data scale on AAT. The x-axis represents the scale of the data. For example, 5 represents 5,000 historical trajectories.}
    \label{fig:data_size_show}
\end{figure}

$\quad$
\\
\noindent\textbf{Training data size}.
The size of the dataset is also an important factor that affects the performance of AAT.
To this end, we test the sensitivity of AAT to dataset size across the HalfCheetah, Hopper, Humanoid, Pendulum, and Walker environments.
AAT is trained on datasets of varying sizes until losses converge, after which it undergoes 10 evaluations in the respective environment.
Fig.~\ref{fig:data_size_show} shows the cumulative reward of the target policy D4PG when attacked by AAT trained on datasets of different sizes.
Experimental results indicate that the attack performance of AAT gradually improves with the expansion of dataset size, particularly before 30,000 samples (i.e., trajectory).
However, when the dataset size exceeds 35,000 samples, the improvement in AAT's attack performance gradually becomes slower and even fluctuates.
Furthermore, Appendix~\ref{sup:training_data} provides experiments demonstrating the impact of \textit{training data diversity} and different \textit{data collection methods} on the performance of AAT, as well as \textit{the advantages of AAT}.
We also conduct experiments on the impact of different substitution policies on AAT performance in a black-box scenario.
Experimental results in Appendix~\ref{sup:substitu_policy} indicate that the performance of AAT in black-box scenarios is largely unaffected by the specific choice of substitution policies.

\begin{figure*}[t]
  \centering
  \subfigure{\includegraphics[width=0.24\textwidth]{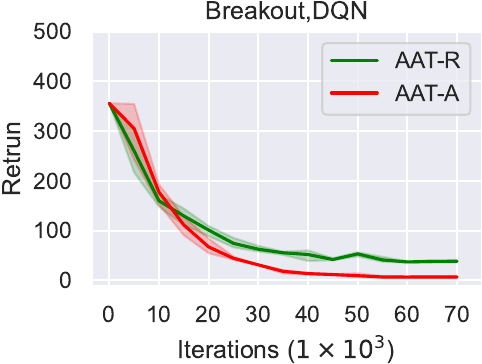}}
  \subfigure{\includegraphics[width=0.24\textwidth]{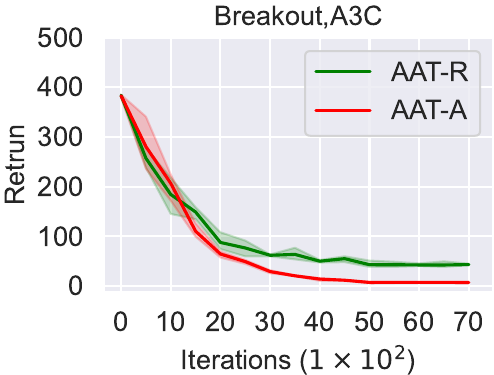}}
  \subfigure{\includegraphics[width=0.24\textwidth]{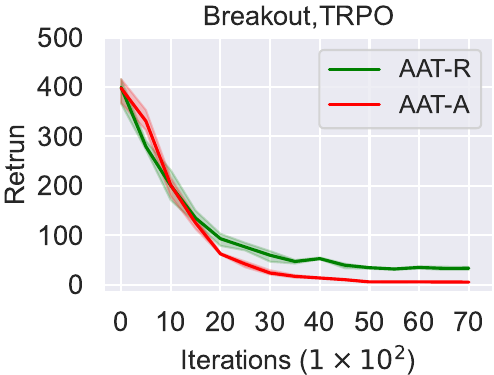}}
  \subfigure{\includegraphics[width=0.24\textwidth]{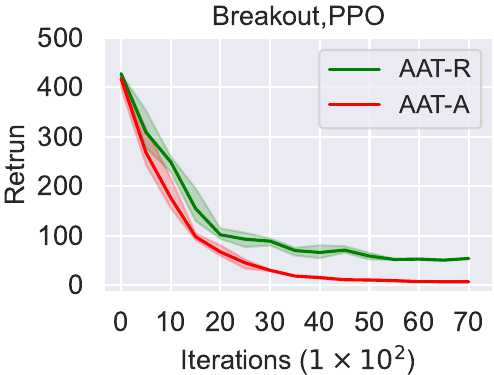}}
  \caption{The impact of weighted advantages on AAT performance.}
\label{fig:advantage_image}
\end{figure*}

\subsection{Ablation study}
\noindent\textbf{Ablation on the weighted advantage.}
To evaluate the impact of the weighted advantage, we perform ablation studies across six different games.
During AAT training, attack performance metrics (i.e., cumulative rewards for target policies) are assessed every 1,000 iterations to monitor optimization progress.
Fig.~\ref{fig:advantage_image} illustrates the attack performance in the Breakout environments (see Appendix~\ref{sup:ab_result} for more experimental results), where AAT-R denotes the method incorporating future reward sum learning, while AAT-A indicates the approach integrating the weighted advantage learning.
Overall, AAT-R and AAT-A demonstrate differences in attack effectiveness and convergence speed across various target policies within identical environments.
Before $10 \times 10^2$ iterations, the effects of AAT-R and AAT-A attacks closely resemble each other in the Breakout environment. 
As the number of iterations increases, the attack effectiveness of AAT-A begins to surpass that of AAT-R, eventually converging around $40 \times 10^2$.
In the Pong and Chopper Command environments, AAT-A consistently outperforms AAT-R, which highlights the effectiveness of the weighted advantage in enhancing attack performance. 

In addition, we verify the impact of the weighted advantage and the advantage calculated by the Q and V functions (i.e., ordinary advantage) on the performance of AAT, and test the attack performance of AAT with different advantage learning methods in Atari.
The experimental results (refer to Appendix~\ref{sup:ab_result}) show that the weighted advantage is generally superior to the ordinary advantage, especially achieving a comprehensive surpass in black-box scenarios.

\begin{figure}
    \centering
    \includegraphics[width=0.98\linewidth]{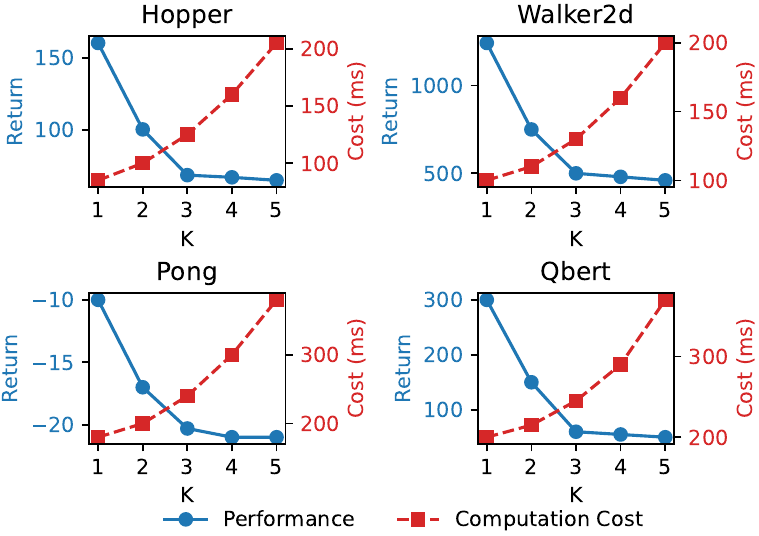}
    \caption{The trade-off relationship between model performance and computational overhead. Cost means that each output token of the A100 GPU costs approximately.}
    \label{fig:K_abotion}
\end{figure}

$\quad$
\\
\noindent \textbf{MSCSA performance}.
We compare the performance of Adversarial Attack Training (AAT) using MSCSA and traditional self-attention (SA) in the DeepMind Control Suite, aiming to analyze their effectiveness in generating adversarial perturbations. 
To evaluate their adversarial perturbation generation capabilities, we conduct experiments on D4PG and A3C agents across 10 random seeds, with detailed results summarized in Appendix~\ref{sup:ab_result}.
The experimental results demonstrate that MSCSA outperforms SA in both white-box and black-box scenarios.

$\quad$
\\
\textbf{Constraint investigation.}
We study the impact of employing different constraint normal in Eq.4 on the outcomes of the attacks.
Specifically, $L_1$, $L_2$, and $L_\infty$ are utilized to compute the difference between adversarial examples and original states, and they serve as the loss function for the AAT method.
Fig.\ref{fig:normal_image} of Appendix~\ref{sup:ab_result} depicts the attack outcomes under various L-norms.
In the current context, $L_2$ is preferable over $L_1$ and $L_\infty$ as the constraint norm for AAT.

$\quad$
\\
\noindent \textbf{Ablation on the number of scales $K$}.
To evaluate the impact of the number of temporal scales used in MSCSA, we conduct ablation experiments by varying $K \in \{1, 2, 3, 4, 5\}$, where the window length of each scale is a multiple of the previous one, and length of $L_1$ is 5 (i.e., $L_K=h^{K-1}*L_1$, where $h=2$ and $L_1=5$).  
Fig.~\ref{fig:K_abotion} shows that increasing $K$ from 1 to 3 consistently improves model performance, confirming that multi-scale modeling captures both short- and long-term dependencies.
However, further increasing K beyond 3 yields diminishing returns and introduces noticeable computational overhead due to the quadratic cost $\mathcal{O}(L_k^2 d)$ per scale, where $d$ is the dimension of the representation for each token. In particular, using $K = 4 $ or $5$ provides only marginal gains while significantly increasing latency, which may be undesirable in real-time or resource-constrained scenarios. We therefore set $K = 3$ as the default configuration, balancing efficiency and accuracy.

$\quad$
\\
\textbf{Ablation on window initialization and growth strategy.}
To assess the impact of temporal window design in MSCSA, we conduct ablations on both the growth strategy and the initial window size ($Len$).
The experimental results in Appendix~\ref{sup:ab_result} show that exponential growth with $r=2$ and $Len=5$ achieves the best average reward while maintaining moderate computational cost. 
To further examine the sensitivity of MSCSA to the initial window size, we systematically vary $Len$ from 1 to 10 with a step size of 1.
Keeping the growth policy and the scale factor K constant, we perform repeated runs with 10 different random seeds for each $Len$ and record the resulting attack performance of AAT.
As shown in Fig.~\ref{fig:len_aboty} (see Appendix~\ref{sup:ab_result}), in most environments, setting the window size to 5 achieves better attack performance in both Atari and MuJoCo environments.

\section{Related work}
\noindent \textbf{Adversarial attacks}.
In Deep Reinforcement Learning (DRL), value functions and target policies are often approximated using Deep Neural Networks (DNNs), inheriting the sensitivity of DNNs to adversarial attacks.
Attackers can exploit this weakness by injecting subtle perturbations into states of the agent, thus causing the agent to fail its intended tasks.
Gradient-based attack methods~\cite{DBLP:HuangPGDA17} first attempt to attack DRL policy by applying FGSM~\cite{DBLP:GoodfellowSS14} to the state at each time step.
On this basis, Behzadan et al.~\cite{DBLP:BehzadanM17} study the transferability of adversarial examples in different DRL models and evaluate the robustness of DRL adversarial training-time and test-time attacks.
In order to improve the transferability of adversarial examples, Pattanaik et al.~\cite{DBLP:PattanaikTLBC18} develop a method for creating adversarial examples by increasing the likelihood of choosing the worst action.
In addition, Yang et al.~\cite{DBLP:abs-1905-12282} propose universal adversarial examples, aiming to generate a universal perturbation that is added to all subsequent states to alter the agent goal.

To reduce the cost of attacks,
Jernej et al.~\cite{DBLP:KosS17} propose a critical frame attack method, where perturbations are injected into $N$ frames when their reward value exceeds a predefined threshold.
Similarly,  Chen et al.~\cite{DBLP:LinHLS0S17} use a value function to measure the variance between an agent’s best and worst actions over multiple steps, leveraging this variance to identify critical states.
Then, the attacker launches attacks on these critical states to effectively reduce the attack frequency while maintaining superior attack performance.
Jianwen et al.~\cite{DBLP:SunZX0ZCL20} introduce two techniques: (1) a critical point attack, which predicts future states and evaluates attack strategies using a damage awareness metric, and (2) an antagonist attack, where an adversarial agent optimizes its strategy by exploiting the reward function of the victim agent.
However, gradient-based attack methods only focus on using traditional techniques to attack specific states, while ignoring the ultimate goal of the entire DRL task.

To solve this issue, reward-based attack methods~\cite{gronauer2022multi} leverage RL techniques to generate adversarial perturbations aligned with sequence decision-making processes. Specifically, these methods focus on attack strategies in game environments that holistically account for both the current state and future rewards of the target policy, rather than optimizing perturbations in isolation. 
For example, Sun et al.~\cite{DBLP:SunZLH22} propose the policy adversarial actor-director (PA-AD) method, where a director and an actor collaborate to identify the optimal state perturbation.
Similarly, Zhang et al.~\cite{DBLP:ZhangCBH21} propose a framework for alternating training with learning opponents (ATLA), which follows the idea of adversarial games and uses policy gradients to train the opponent online with the agent.
Additionally, Chen et al.~\cite{DBLP:YuS22} introduce a method that utilizes GANs to generate adversarial examples, and employs the REINFORCE~\cite{DBLP:journals/ml/Williams92} algorithm to train an adversarial strategy, which can generate natural adversarial examples with high attack performance.

However, both gradient-based and reward-based methods fail to effectively capture the dependency between the current state and historical information when generating adversarial examples. As a consequence, the resulting adversarial examples lack consistent influence across decision steps, making these examples ineffective at disrupting the agent’s cumulative return in sequential decision-making. Since this deficiency stems from the adversarial examples generation mechanism itself rather than the attacker’s access to the model, the dependency issue persists in both white-box and black-box scenarios.
Unlike prior methods, AAT can effectively enhance attack performance by capturing historical information to generate time-correlated adversarial examples and leveraging weighted advantages to explore adversarial perturbations with high rewards during adversarial example generation.

$\quad$
\\
\textbf{Transformers for sequential decision-making}.
The success of Transformers in natural language processing~\cite{DBLP:VaswaniSPUJGKP17} motivates their application in other fields, including object detection~\cite{DBLP:LiDDX24,DBLP:MoonJJO24}, speech recognition~\cite{DBLP:Deschamps-Berger22, DBLP:GaoZX023}, and image classification~\cite{zhao2023hyperspectral, DBLP:MaLFW24}.
Recently, Transformers are widely applied to sequence decision-making tasks.
For example, Decision Transformer (DT)~\cite{DBLP:DT} is one of the first studies to treat the policy optimization problem as a sequence modeling task and trains the Transformer policy based on historical data and expected future returns.
However, DT performance is hardly satisfactory when there is a lack of optimal offline data.
Inspired by this limitation, Q-learning DT~\cite{DBLP:YamagataKS23} uses Q-values to refine the returns in the training data, thus giving DT the ability of Q-learning to handle sub-optimal data.
Other variants like online DT (ODT)~\cite{DBLP:ZhengZG22} unifies offline pre-training and online fine-tuning, and Multi-Game DT (MGDT)~\cite{DBLP:LeeNYLFGFXJMM22} enables cross-task generalization for Atari games without task-specific fine-tuning.

Unlike previous Transformer models, we propose a multi-scale causal self-attention mechanism that models historical segments of varying lengths and facilitates information interaction across scales to capture multi-granularity temporal features.
Moreover, we introduce a weighted advantage mechanism that quantifies the effectiveness of adversarial perturbations in the current state. 
By maximizing the entropy of this advantage distribution, AAT encourages exploration of diverse perturbation strategies, which enables the sequence model to identify optimal perturbations tailored to different trajectories during sequence generation, consequently generating adversarial examples with higher performance.

\section{Conclusion}
In this paper, we investigate the impact of historical dependencies on the performance of adversarial examples and propose a novel method AAT that is able to generate time-correlated adversarial examples with high attack performance.
Experimental results demonstrate the effectiveness of AAT in both discrete and continuous action environments.
However, AAT currently still focuses on single-agent attacks.
In the future, we plan to extend the framework structure of AAT to increase the attack effectiveness in multi-agent scenarios.

\section*{Acknowledgments}
We gratefully acknowledge support from the National Natural Science Foundation of China (No. 62076259), the Fundamental and Applicational Research Funds of Guangdong province (No. 2023A1515012946), and the Fundamental Research Funds for the Central Universities-Sun Yat-sen University. This research was supported by Meituan.




\bibliographystyle{IEEEtran}
\bibliography{cas-refs}

@inproceedings{DBLP:HuangPGDA17,
  author       = {Sandy H. Huang and
                  Nicolas Papernot and
                  Ian J. Goodfellow and
                  Yan Duan and
                  Pieter Abbeel},
  title        = {Adversarial Attacks on Neural Network Policies},
  booktitle    = {{ICLR}},
  publisher    = {OpenReview.net},
  year         = {2017}
}

@inproceedings{DBLP:LinHLS0S17,
  author       = {Yen{-}Chen Lin and
                  Zhang{-}Wei Hong and
                  Yuan{-}Hong Liao and
                  Meng{-}Li Shih and
                  Ming{-}Yu Liu and
                  Min Sun},
  title        = {Tactics of Adversarial Attack on Deep Reinforcement Learning Agents},
  booktitle    = {{ICLR}},
  publisher    = {OpenReview.net},
  year         = {2017}
}

@inproceedings{DBLP:GoodfellowSS14,
  author       = {Ian J. Goodfellow and
                  Jonathon Shlens and
                  Christian Szegedy},
  title        = {Explaining and Harnessing Adversarial Examples},
  booktitle    = {{ICLR}},
  year         = {2015}
}

@inproceedings{DBLP:SzegedyZSBEGF13,
  author       = {Christian Szegedy and
                  Wojciech Zaremba and
                  Ilya Sutskever and
                  Joan Bruna and
                  Dumitru Erhan and
                  Ian J. Goodfellow and
                  Rob Fergus},
  title        = {Intriguing properties of neural networks},
  booktitle    = {{ICLR}},
  year         = {2014}
}

@inproceedings{DBLP:SunZX0ZCL20,
  author       = {Jianwen Sun and
                  Tianwei Zhang and
                  Xiaofei Xie and
                  Lei Ma and
                  Yan Zheng and
                  Kangjie Chen and
                  Yang Liu},
  title        = {Stealthy and Efficient Adversarial Attacks against Deep Reinforcement
                  Learning},
  booktitle    = {{AAAI}},
  pages        = {5883--5891},
  publisher    = {{AAAI} Press},
  year         = {2020}
}

@inproceedings{DBLP:DT,
  author       = {Lili Chen and
                  Kevin Lu and
                  Aravind Rajeswaran and
                  Kimin Lee and
                  Aditya Grover and
                  Michael Laskin and
                  Pieter Abbeel and
                  Aravind Srinivas and
                  Igor Mordatch},
  title        = {Decision Transformer: Reinforcement Learning via Sequence Modeling},
  booktitle    = {NeurIPS},
  pages        = {15084--15097},
  year         = {2021}
}

@inproceedings{DBLP:KostrikovNL22,
  author       = {Ilya Kostrikov and
                  Ashvin Nair and
                  Sergey Levine},
  title        = {Offline Reinforcement Learning with Implicit Q-Learning},
  booktitle    = {{ICLR}
                  },
  publisher    = {OpenReview.net},
  year         = {2022}
}

@inproceedings{DBLP:VaswaniSPUJGKP17,
  author       = {Ashish Vaswani and
                  Noam Shazeer and
                  Niki Parmar and
                  Jakob Uszkoreit and
                  Llion Jones and
                  Aidan N. Gomez and
                  Lukasz Kaiser and
                  Illia Polosukhin},
  title        = {Attention is All you Need},
  booktitle    = {NeurIPS},
  pages        = {5998--6008},
  year         = {2017}
}

@inproceedings{DBLP:BehzadanM17,
  author       = {Vahid Behzadan and
                  Arslan Munir},
  title        = {Vulnerability of Deep Reinforcement Learning to Policy Induction Attacks},
  booktitle    = {{MLDM}},
  series       = {Lecture Notes in Computer Science},
  volume       = {10358},
  pages        = {262--275},
  publisher    = {Springer},
  year         = {2017}
}

@inproceedings{DBLP:PattanaikTLBC18,
  author       = {Anay Pattanaik and
                  Zhenyi Tang and
                  Shuijing Liu and
                  Gautham Bommannan and
                  Girish Chowdhary},
  title        = {Robust Deep Reinforcement Learning with Adversarial Attacks},
  booktitle    = {{AAMAS}},
  pages        = {2040--2042},
  publisher    = {International Foundation for Autonomous Agents and Multiagent Systems
                  Richland},
  year         = {2018}
}

@article{DBLP:abs-1905-12282,
  author       = {L{\'{e}}onard Hussenot and
                  Matthieu Geist and
                  Olivier Pietquin},
  title        = {Targeted Attacks on Deep Reinforcement Learning Agents through Adversarial
                  Observations},
  journal      = {CoRR},
  volume       = {abs/1905.12282},
  year         = {2019}
}

@inproceedings{DBLP:KosS17,
  author       = {Jernej Kos and
                  Dawn Song},
  title        = {Delving into adversarial attacks on deep policies},
  booktitle    = {{ICLR}},
  publisher    = {OpenReview.net},
  year         = {2017}
}

@inproceedings{DBLP:YuS22,
  author       = {Mengran Yu and
                  Shiliang Sun},
  title        = {Natural Black-Box Adversarial Examples against Deep Reinforcement
                  Learning},
  booktitle    = {{AAAI}},
  pages        = {8936--8944},
  publisher    = {{AAAI} Press},
  year         = {2022}
}

@inproceedings{DBLP:SharifBBR16,
  author       = {Mahmood Sharif and
                  Sruti Bhagavatula and
                  Lujo Bauer and
                  Michael K. Reiter},
  title        = {Accessorize to a Crime: Real and Stealthy Attacks on State-of-the-Art
                  Face Recognition},
  booktitle    = {{CCS}},
  pages        = {1528--1540},
  publisher    = {{ACM}},
  year         = {2016}
}

@article{DBLP:MnihKSGAWR13,
  author       = {Volodymyr Mnih and
                  Koray Kavukcuoglu and
                  David Silver and
                  Alex Graves and
                  Ioannis Antonoglou and
                  Daan Wierstra and
                  Martin A. Riedmiller},
  title        = {Playing Atari with Deep Reinforcement Learning},
  journal      = {CoRR},
  volume       = {abs/1312.5602},
  year         = {2013}
}

@inproceedings{DBLP:SchulmanLAJM15,
  author       = {John Schulman and
                  Sergey Levine and
                  Pieter Abbeel and
                  Michael I. Jordan and
                  Philipp Moritz},
  title        = {Trust Region Policy Optimization},
  booktitle    = {{ICML}},
  volume       = {37},
  pages        = {1889--1897},
  publisher    = {JMLR.org},
  year         = {2015}
}

@article{DBLP:SchulmanWDRK17,
  author       = {John Schulman and
                  Filip Wolski and
                  Prafulla Dhariwal and
                  Alec Radford and
                  Oleg Klimov},
  title        = {Proximal Policy Optimization Algorithms},
  journal      = {CoRR},
  volume       = {abs/1707.06347},
  year         = {2017}
}

@inproceedings{DBLP:OikarinenZMDW21,
  author       = {Tuomas P. Oikarinen and
                  Wang Zhang and
                  Alexandre Megretski and
                  Luca Daniel and
                  Tsui{-}Wei Weng},
  title        = {Robust Deep Reinforcement Learning through Adversarial Loss},
  booktitle    = {NeurIPS},
  pages        = {26156--26167},
  year         = {2021}
}

@article{DBLP:abs-2305-17342,
  author       = {Xiangyu Liu and
                  Souradip Chakraborty and
                  Yanchao Sun and
                  Furong Huang},
  title        = {Rethinking Adversarial Policies: {A} Generalized Attack Formulation
                  and Provable Defense in Multi-Agent {RL}},
  journal      = {CoRR},
  volume       = {abs/2305.17342},
  year         = {2023}
}

@inproceedings{DBLP:SunZLH22,
  author       = {Yanchao Sun and
                  Ruijie Zheng and
                  Yongyuan Liang and
                  Furong Huang},
  title        = {Who Is the Strongest Enemy? Towards Optimal and Efficient Evasion
                  Attacks in Deep {RL}},
  booktitle    = {{ICLR}},
  publisher    = {OpenReview.net},
  year         = {2022}
}

@inproceedings{DBLP:ZhangCBH21,
  author       = {Huan Zhang and
                  Hongge Chen and
                  Duane S. Boning and
                  Cho{-}Jui Hsieh},
  title        = {Robust Reinforcement Learning on State Observations with Learned Optimal
                  Adversary},
  booktitle    = {ICLR},
  publisher    = {OpenReview.net},
  year         = {2021}
}

@article{DBLP:HuangCF23,
  author       = {Jeffrey Huang and
                  Ho Jin Choi and
                  Nadia Figueroa},
  title        = {Trade-Off Between Robustness and Rewards Adversarial Training for
                  Deep Reinforcement Learning Under Large Perturbations},
  journal      = {{IEEE} Robotics Autom. Lett.},
  volume       = {8},
  number       = {12},
  pages        = {8018--8025},
  year         = {2023}
}

@article{schott2024robust,
  author={Schott, Lucas and Delas, Josephine and Hajri, Hatem and Gherbi, Elies and Yaich, Reda and Boulahia-Cuppens, Nora and Cuppens, Frederic and Lamprier, Sylvain},
  title={Robust Deep Reinforcement Learning Through Adversarial Attacks and Training: A Survey},
  journal={arXiv preprint arXiv:2403.00420},
  year={2024}
}

@inproceedings{DBLP:DosovitskiyB0WZ21,
  author       = {Alexey Dosovitskiy and
                  Lucas Beyer and
                  Alexander Kolesnikov and
                  Dirk Weissenborn and
                  Xiaohua Zhai and
                  Thomas Unterthiner and
                  Mostafa Dehghani and
                  Matthias Minderer and
                  Georg Heigold and
                  Sylvain Gelly and
                  Jakob Uszkoreit and
                  Neil Houlsby},
  title        = {An Image is Worth 16x16 Words: Transformers for Image Recognition
                  at Scale},
  booktitle    = {{ICLR}},
  publisher    = {OpenReview.net},
  year         = {2021}
}

@inproceedings{DBLP:YamagataKS23,
  author       = {Taku Yamagata and
                  Ahmed Khalil and
                  Ra{\'{u}}l Santos{-}Rodr{\'{\i}}guez},
  title        = {Q-learning Decision Transformer: Leveraging Dynamic Programming for
                  Conditional Sequence Modelling in Offline {RL}},
  booktitle    = {{ICML}},
  volume       = {202},
  pages        = {38989--39007},
  publisher    = {{PMLR}},
  year         = {2023}
}

@inproceedings{kakade2002approximately,
  author={Kakade, Sham and Langford, John},
  title={Approximately optimal approximate reinforcement learning},
  booktitle={Proceedings of the Nineteenth International Conference on Machine Learning},
  pages={267--274},
  year={2002}
}

@article{gronauer2022multi,
  author={Gronauer, Sven and Diepold, Klaus},
  title={Multi-agent deep reinforcement learning: a survey},
  journal={Artificial Intelligence Review},
  volume={55},
  number={2},
  pages={895--943},
  year={2022}
}

@inproceedings{NIPS2016_cc7e2b87,
 author = {Ho, Jonathan and Ermon, Stefano},
 title = {Generative Adversarial Imitation Learning},
 booktitle = {NeurIPS},
 publisher = {Curran Associates, Inc.},
 pages = {4565–4573},
 year = {2016}
}

@article{DBLP:journals/ml/Williams92,
  author       = {Ronald J. Williams},
  title        = {Simple Statistical Gradient-Following Algorithms for Connectionist
                  Reinforcement Learning},
  journal      = {Mach. Learn.},
  volume       = {8},
  pages        = {229--256},
  year         = {1992}
}

@article{zhao2023hyperspectral,
  author={Zhao, Chunhui and Qin, Boao and Feng, Shou and Zhu, Wenxiang and Sun, Weiwei and Li, Wei and Jia, Xiuping},
  title={Hyperspectral image classification with multi-attention transformer and adaptive superpixel segmentation-based active learning},
  journal={IEEE Transactions on Image Processing},
  publisher={IEEE},
  year={2023}
}

@inproceedings{DBLP:MaLFW24,
  author       = {Yingfan Ma and
                  Xiaoyuan Luo and
                  Kexue Fu and
                  Manning Wang},
  title        = {Transformer-Based Video-Structure Multi-Instance Learning for Whole
                  Slide Image Classification},
  booktitle    = {{AAAI}},
  pages        = {14263--14271},
  publisher    = {{AAAI} Press},
  year         = {2024}
}

@inproceedings{DBLP:GaoZX023,
  author       = {Yingxue Gao and
                  Huan Zhao and
                  Yufeng Xiao and
                  Zixing Zhang},
  title        = {GCFormer: {A} Graph Convolutional Transformer for Speech Emotion Recognition},
  booktitle    = {{ICMI}},
  pages        = {307--313},
  publisher    = {{ACM}},
  year         = {2023}
}

@inproceedings{DBLP:Deschamps-Berger22,
  author       = {Th{\'{e}}o Deschamps{-}Berger and
                  Lori Lamel and
                  Laurence Devillers},
  title        = {Investigating Transformer Encoders and Fusion Strategies for Speech
                  Emotion Recognition in Emergency Call Center Conversations},
  booktitle    = {{ICMI}},
  pages        = {144--153},
  publisher    = {{ACM}},
  year         = {2022}
}

@article{DBLP:LiDDX24,
  author       = {Jianan Li and
                  Shaocong Dong and
                  Lihe Ding and
                  Tingfa Xu},
  title        = {MsSVT++: Mixed-Scale Sparse Voxel Transformer With Center Voting for
                  3D Object Detection},
  journal      = {{IEEE} Trans. Pattern Anal. Mach. Intell.},
  volume       = {46},
  number       = {5},
  pages        = {3736--3752},
  year         = {2024}
}

@article{DBLP:MoonJJO24,
  author       = {Joonhyeok Moon and
                  Munsu Jeon and
                  Siheon Jeong and
                  Ki{-}Yong Oh},
  title        = {RoMP-transformer: Rotational bounding box with multi-level feature
                  pyramid transformer for object detection},
  journal      = {Pattern Recognit.},
  volume       = {147},
  pages        = {110067},
  year         = {2024}
}

@inproceedings{DBLP:ZhengZG22,
  author       = {Qinqing Zheng and
                  Amy Zhang and
                  Aditya Grover},
  title        = {Online Decision Transformer},
  booktitle    = {{ICML}},
  series       = {Proceedings of Machine Learning Research},
  volume       = {162},
  pages        = {27042--27059},
  publisher    = {{PMLR}},
  year         = {2022}
}

@inproceedings{DBLP:LeeNYLFGFXJMM22,
  author       = {Kuang{-}Huei Lee and
                  Ofir Nachum and
                  Mengjiao Yang and
                  Lisa Lee and
                  Daniel Freeman and
                  Sergio Guadarrama and
                  Ian Fischer and
                  Winnie Xu and
                  Eric Jang and
                  Henryk Michalewski and
                  Igor Mordatch},
  title        = {Multi-Game Decision Transformers},
  booktitle    = {NeurIPS},
  year         = {2022}
}

@article{DBLP:abs-1801-00690,
  author       = {Yuval Tassa and
                  Yotam Doron and
                  Alistair Muldal and
                  Tom Erez and
                  Yazhe Li and
                  Diego de Las Casas and
                  David Budden and
                  Abbas Abdolmaleki and
                  Josh Merel and
                  Andrew Lefrancq and
                  Timothy P. Lillicrap and
                  Martin A. Riedmiller},
  title        = {DeepMind Control Suite},
  journal      = {CoRR},
  volume       = {abs/1801.00690},
  year         = {2018}
}

@inproceedings{DBLP:Ouyang0JAWMZASR22,
  author       = {Long Ouyang and
                  Jeffrey Wu and
                  Xu Jiang and
                  Diogo Almeida and
                  Carroll L. Wainwright and
                  Pamela Mishkin and
                  Chong Zhang and
                  Sandhini Agarwal and
                  Katarina Slama and
                  Alex Ray and
                  John Schulman and
                  Jacob Hilton and
                  Fraser Kelton and
                  Luke Miller and
                  Maddie Simens and
                  Amanda Askell and
                  Peter Welinder and
                  Paul F. Christiano and
                  Jan Leike and
                  Ryan Lowe},
  title        = {Training language models to follow instructions with human feedback},
  booktitle    = {NeurIPS},
  year         = {2022}
}

@article{DBLP:LiuYFJHN23,
  author       = {Pengfei Liu and
                  Weizhe Yuan and
                  Jinlan Fu and
                  Zhengbao Jiang and
                  Hiroaki Hayashi and
                  Graham Neubig},
  title        = {Pre-train, Prompt, and Predict: {A} Systematic Survey of Prompting
                  Methods in Natural Language Processing},
  journal      = {{ACM} Comput. Surv.},
  volume       = {55},
  number       = {9},
  pages        = {195:1--195:35},
  year         = {2023}}

@inproceedings{DBLP:HoHLLWW23,
  author       = {Kuo{-}Hao Ho and
                  Ping{-}Chun Hsieh and
                  Chiu{-}Chou Lin and
                  You{-}Ren Luo and
                  Feng{-}Jian Wang and
                  I{-}Chen Wu},
  title        = {Towards Human-Like {RL:} Taming Non-Naturalistic Behavior in Deep
                  {RL} via Adaptive Behavioral Costs in 3D Games},
  booktitle    = { {ACML}},
  series       = {Proceedings of Machine Learning Research},
  volume       = {222},
  pages        = {438--453},
  publisher    = {{PMLR}},
  year         = {2023}
}

@inproceedings{DBLP:LinLSYF022,
  author       = {Zichuan Lin and
                  Junyou Li and
                  Jianing Shi and
                  Deheng Ye and
                  Qiang Fu and
                  Wei Yang},
  title        = {JueWu-MC: Playing Minecraft with Sample-efficient Hierarchical Reinforcement
                  Learning},
  booktitle    = {{IJCAI}},
  pages        = {3257--3263},
  publisher    = {ijcai.org},
  year         = {2022}
}

@article{DBLP:YangLLZFZWLS24,
  author       = {Yulong Yang and
                  Chenhao Lin and
                  Qian Li and
                  Zhengyu Zhao and
                  Haoran Fan and
                  Dawei Zhou and
                  Nannan Wang and
                  Tongliang Liu and
                  Chao Shen},
  title        = {Quantization Aware Attack: Enhancing Transferable Adversarial Attacks
                  by Model Quantization},
  journal      = {{IEEE} Trans. Inf. Forensics Secur.},
  volume       = {19},
  pages        = {3265--3278},
  year         = {2024}
}

@article{DBLP:HeWLYJLZ23,
  author       = {Shenghong He and
                  Ruxin Wang and
                  Tongliang Liu and
                  Chao Yi and
                  Xin Jin and
                  Renyang Liu and
                  Wei Zhou},
  title        = {Type-I Generative Adversarial Attack},
  journal      = {{IEEE} Trans. Dependable Secur. Comput.},
  volume       = {20},
  number       = {3},
  pages        = {2593--2606},
  year         = {2023}
}

@article{DBLP:FangS24,
  author       = {Shengbang Fang and
                  Matthew C. Stamm},
  title        = {Attacking Image Splicing Detection and Localization Algorithms Using
                  Synthetic Traces},
  journal      = {{IEEE} Trans. Inf. Forensics Secur.},
  volume       = {19},
  pages        = {2143--2156},
  year         = {2024}
}

@inproceedings{DBLP:RiceWK20,
  author       = {Leslie Rice and
                  Eric Wong and
                  J. Zico Kolter},
  title        = {Overfitting in adversarially robust deep learning},
  booktitle    = {{ICML}},
  series       = {Proceedings of Machine Learning Research},
  volume       = {119},
  pages        = {8093--8104},
  publisher    = {{PMLR}},
  year         = {2020}
}

@inproceedings{DBLP:GuoWKX21,
  author       = {Wenbo Guo and
                  Xian Wu and
                  Usmann Khan and
                  Xinyu Xing},
  title        = {{EDGE:} Explaining Deep Reinforcement Learning Policies},
  booktitle    = {NeurIPS},
  pages        = {12222--12236},
  year         = {2021}
}

@inproceedings{DBLP:Wu0WX21,
  author       = {Xian Wu and
                  Wenbo Guo and
                  Hua Wei and
                  Xinyu Xing},
  title        = {Adversarial Policy Training against Deep Reinforcement Learning},
  booktitle    = {{USENIX}},
  pages        = {1883--1900},
  publisher    = {{USENIX} Association},
  year         = {2021}
}

@article{lv2024safe,
  author={Lv, Zefang and Xiao, Liang and Chen, Yifan and Chen, Haoyu and Ji, Xiangyang},
  title={Safe multi-agent reinforcement learning for wireless applications against adversarial communications},
  journal={IEEE Transactions on Information Forensics and Security},
  year={2024},
  publisher={IEEE}
}

@article{DBLP:MaXJZS23,
  author       = {Haotian Ma and
                  Ke Xu and
                  Xinghao Jiang and
                  Zeyu Zhao and
                  Tanfeng Sun},
  title        = {Transferable Black-Box Attack Against Face Recognition With Spatial
                  Mutable Adversarial Patch},
  journal      = {{IEEE} Trans. Inf. Forensics Secur.},
  volume       = {18},
  pages        = {5636--5650},
  year         = {2023}
}

@article{DBLP:QiaobenYZSZZ24,
  author       = {You Qiaoben and
                  Chengyang Ying and
                  Xinning Zhou and
                  Hang Su and
                  Jun Zhu and
                  Bo Zhang},
  title        = {Understanding adversarial attacks on observations in deep reinforcement
                  learning},
  journal      = {Sci. China Inf. Sci.},
  volume       = {67},
  number       = {5},
  year         = {2024}
}

@inproceedings{DBLP:GaoWCKZ024,
  author       = {Chenxiao Gao and
                  Chenyang Wu and
                  Mingjun Cao and
                  Rui Kong and
                  Zongzhang Zhang and
                  Yang Yu},
  title        = {{ACT:} Empowering Decision Transformer with Dynamic Programming via
                  Advantage Conditioning},
  booktitle    = {{AAAI}},
  pages        = {12127--12135},
  publisher    = {{AAAI} Press},
  year         = {2024}
}

@incollection{barto1988neuronlike,
  title={Neuronlike adaptive elements that can solve difficult learning control problems},
  author={Barto, Andrew G and Sutton, Richard S and Anderson, Charles W},
  booktitle={Neurocomputing: foundations of research},
  pages={535--549},
  year={1988}
}

@article{wali2025explainable,
  title={Explainable AI and random forest based reliable intrusion detection system},
  author={Wali, Syed and Farrukh, Yasir Ali and Khan, Irfan},
  journal={Computers \& Security},
  pages={104542},
  year={2025},
  publisher={Elsevier}
}

@article{DBLP:TianSWXZLL24,
  author       = {Jiwei Tian and
                  Chao Shen and
                  Buhong Wang and
                  Xiaofang Xia and
                  Meng Zhang and
                  Chenhao Lin and
                  Qian Li},
  title        = {{LESSON:} Multi-Label Adversarial False Data Injection Attack for
                  Deep Learning Locational Detection},
  journal      = {{IEEE} Trans. Dependable Secur. Comput.},
  volume       = {21},
  number       = {5},
  pages        = {4418--4432},
  year         = {2024}
}

@article{DBLP:ChuG24,
  author       = {Kai{-}Fung Chu and
                  Weisi Guo},
  title        = {Multi-Agent Reinforcement Learning-Based Passenger Spoofing Attack
                  on Mobility-as-a-Service},
  journal      = {{IEEE} Trans. Dependable Secur. Comput.},
  volume       = {21},
  number       = {6},
  pages        = {5565--5581},
  year         = {2024}
}

@article{DBLP:FuLWYFWDWX25,
  author       = {Songtao Fu and
                  Qi Li and
                  Xiaoliang Wang and
                  Su Yao and
                  Xuewei Feng and
                  Ziqiang Wang and
                  Xinle Du and
                  Kao Wan and
                  Ke Xu},
  title        = {Secure Fault Localization in Path Aware Networking},
  journal      = {{IEEE} Trans. Dependable Secur. Comput.},
  volume       = {22},
  number       = {1},
  pages        = {205--220},
  year         = {2025}
}

@inproceedings{DBLP:ZhangLYYL0T23,
  author       = {Shuhai Zhang and
                  Feng Liu and
                  Jiahao Yang and
                  Yifan Yang and
                  Changsheng Li and
                  Bo Han and
                  Mingkui Tan},
  title        = {Detecting Adversarial Data by Probing Multiple Perturbations Using
                  Expected Perturbation Score},
  booktitle    = {{ICML}},
  series       = {Proceedings of Machine Learning Research},
  volume       = {202},
  pages        = {41429--41451},
  publisher    = {{PMLR}},
  year         = {2023}
}

@article{DBLP:MumcuY25,
  author       = {Furkan Mumcu and
                  Yasin Yilmaz},
  title        = {Universal and Efficient Detection of Adversarial Data through Nonuniform
                  Impact on Network Layers},
  journal      = {Trans. Mach. Learn. Res.},
  volume       = {2025},
  year         = {2025}
}

@inproceedings{DBLP:LinDZLP20,
  author       = {Jieyu Lin and
                  Kristina Dzeparoska and
                  Sai Qian Zhang and
                  Alberto Leon{-}Garcia and
                  Nicolas Papernot},
  title        = {On the Robustness of Cooperative Multi-Agent Reinforcement Learning},
  booktitle    = {IEEE {SPW}},
  pages        = {62--68},
  publisher    = {{IEEE}},
  year         = {2020}
}

@inproceedings{DBLP:SongLP024,
  author       = {Kaiyu Song and
                  Hanjiang Lai and
                  Yan Pan and
                  Jian Yin},
  title        = {MimicDiffusion: Purifying Adversarial Perturbation via Mimicking Clean
                  Diffusion Model},
  booktitle    = {
                  {CVPR}},
  pages        = {24665--24674},
  publisher    = {{IEEE}},
  year         = {2024}
}

@article{DBLP:WangBSS04,
  author       = {Zhou Wang and
                  Alan C. Bovik and
                  Hamid R. Sheikh and
                  Eero P. Simoncelli},
  title        = {Image quality assessment: from error visibility to structural similarity},
  journal      = {{IEEE} Trans. Image Process.},
  volume       = {13},
  number       = {4},
  pages        = {600--612},
  year         = {2004}
}

@inproceedings{DBLP:0001CX0LBH20,
  author       = {Huan Zhang and
                  Hongge Chen and
                  Chaowei Xiao and
                  Bo Li and
                  Mingyan Liu and
                  Duane S. Boning and
                  Cho{-}Jui Hsieh},
  title        = {Robust Deep Reinforcement Learning against Adversarial Perturbations
                  on State Observations},
  booktitle    = {NeurIPS},
  year         = {2020}
}

@inproceedings{DBLP:McMahanW0X24,
  author       = {Jeremy McMahan and
                  Young Wu and
                  Xiaojin Zhu and
                  Qiaomin Xie},
  title        = {Optimal Attack and Defense for Reinforcement Learning},
  booktitle    = {{AAAI}},
  pages        = {14332--14340},
  publisher    = {{AAAI} Press},
  year         = {2024}
}

@article{chen2023dynamics,
  author       = {Chen, Xuan and Tao, Guanhong and Zhang, Xiangyu},
  title        = {Dynamics Model Based Adversarial Training For Competitive Reinforcement Learning},
  journal      = {CoRR},
  year         = {2023},
}

@article{yamabe2024robust,
  title={Robust Deep Reinforcement Learning against Adversarial Behavior Manipulation},
  author={Yamabe, Shojiro and Fukuchi, Kazuto and Sakuma, Jun},
  journal={arXiv preprint arXiv:2406.03862},
  year={2024}
}


\section{Biography Section}

\vspace{-33pt}

\begin{IEEEbiographynophoto}{Shenghong He}
is currently a PhD student under the supervision of Prof. Chao Yu in the School of Computer Science and Engineering, Sun Yat-sen University, China. He has published a few works at IEEE TDSC and TKDE, and is now working on the problem of offline learning and security with deep reinforcement learning.
\end{IEEEbiographynophoto}
\vspace{-20pt}
\begin{IEEEbiographynophoto}{Chao Yu}
received the Ph.D. degree in computer science from the University of Wollongong, Australia, in 2014. He is currently a professor at the School of Computer Science and Engineering, Sun Yat-sen University, China. He has published more than 100 papers in prestigious conferences and journals, such as ICML, AAAI, NeurIPS, IEEE TCYB, IEEE TVT, and IEEE TKDE. His research interests include multi-agent systems, reinforcement learning, and their wide applications in autonomous driving, smart grid, robotic control, and intelligent healthcare.
\end{IEEEbiographynophoto}
\vspace{-20pt}
\begin{IEEEbiographynophoto}{Danying Mo}
is currently a PhD student under the supervision of Prof. Chao Yu in the School of Computer Science and Engineering, Sun Yat-sen University, China. She is now working on the problem of Agent.
\end{IEEEbiographynophoto}
\vspace{-20pt}
\begin{IEEEbiographynophoto}{Yucong Zhang}
is currently a PhD student under the supervision of Prof. Chao Yu in the School of Computer Science and Engineering, Sun Yat-sen University, China. He has published a few works at AAMAS and is now working on the problem of deep reinforcement learning.
\end{IEEEbiographynophoto}
\vspace{-20pt}
\begin{IEEEbiographynophoto}{Yinqi Wei}
received the B.Sc. degree in engineering from Sun Yat-Sen University, Shenzhen, Guangdong, China, in 2024, and the M.Sc. degree in electronic from Nanyang Technological University, Singapore, in 2025, respectively.
\end{IEEEbiographynophoto}

\vfill

\clearpage
\appendix

\renewcommand{\theequation}{\Alph{section}.\arabic{equation}}  
\setcounter{equation}{0}  
\renewcommand{\thetheorem}{A\arabic{theorem}} 
\setcounter{theorem}{0} 
\setcounter{proof}{0} 
\renewcommand{\thetable}{\Alph{section}.\arabic{table}}
\renewcommand{\thefigure}{\Alph{section}.\arabic{figure}}

\setcounter{table}{0}
\setcounter{figure}{0}

\subsection{Proof of theorem 1}

\label{sup:a}
$\quad$
\\
    According to the Lemma 1, for any advantage estimation method, the lower bound can be expressed as: 
    \begin{equation}
        \label{eq:DE_1}\Delta_C=\mathbb{E}_{\tau\sim\pi}\left[\sum_{t=0}^\infty\gamma^t\hat{A}^\beta(s_t,a_t)\right] - \frac{\epsilon}{1-\gamma},
    \end{equation}
    while the lower bound obtained by the weighted advantage function is: 
    \begin{equation}
        \label{eq:DE_2}
        \Delta_C^{\text{new}}=\mathbb{E}_{\tau\sim\pi}\left[\sum_{t=0}^\infty\gamma^t\tilde{A}^\beta(s_t,a_t)\right] - \frac{\epsilon_{\text{new}}}{1-\gamma}.
    \end{equation}
    Based on Eqs.~(\ref{eq:DE_1}) and~(\ref{eq:DE_2}), we can calculate the lower bound difference:
    \begin{equation}
        \begin{split}
            \Delta_C^{\text{new}}-\Delta_C & = \underbrace{\mathbb{E}_{\tau\sim\pi}\left[\sum_{t=0}^\infty\gamma^t\left(\tilde{A}^\beta(s_t,a_t)-\hat{A}^\beta(s_t,a_t)\right)\right]}_{I}\\& + \frac{\epsilon-\epsilon_{\text{new}}}{1-\gamma}.
        \end{split}
    \end{equation}
    For any $x$, the normalized mapping of weighted advantages satisfies $\Lambda(x)=x-\frac{\lambda|x|}{1+\lambda|x|}$. Therefore, we can obtain $\tilde{A}^\beta(s,a) \le \hat{A}^\beta(s,a)$. When $\hat{A}^\beta(s,a)\ge 0 $ (a symmetric inequality that holds for the negative case), we get $C(s,a)=\hat{A}^\beta(s,a)-\tilde{A}^\beta(s,a)\ge0$. Then, we obtain $I = -\mathbb{E}_{\tau\sim\pi}\left[\sum_{t=0}^\infty\gamma^t C(s_t,a_t)\right]$. According to condition 2), there is:
    \begin{equation}
        \mathbb{E}_{\tau\sim\pi}\left[\sum_{t=0}^\infty\gamma^t C(s_t,a_t)\right]\le \delta.
    \end{equation}
    Hence, $I\ge -\delta$. Substituting this inequality into the lower bounded difference, we derive:
    \begin{equation}
        \Delta_C^{\text{new}}-\Delta_c \ge -\delta + \frac{\epsilon-\epsilon_{\text{new}}}{1-\gamma}.
    \end{equation}
    Based on condition 2) where $\delta\le\frac{\epsilon-\epsilon_{\text{new}}}{1-\gamma}$, we have $-\delta + \frac{\epsilon-\epsilon_{\text{new}}}{1-\gamma} \ge 0$.
    Thus, $\Delta_C^{\text{new}}-\Delta_c \ge 0$, which $\Delta_C^{\text{new}} \ge \Delta_c$.

\section{Additional experimental setup information}

\subsection{Experimental Platform}
All experiments are carried out on a high-performance computing platform with the following core hardware components: CPU is an Intel Xeon E5-2620 v4 (Broadwell architecture), which features an 8-core, 16-thread processing capability, a base frequency of 2.1 GHz, and supports Turbo Boost up to 3.0 GHz.
The memory subsystem is configured with quad-channel DDR4-2400 ECC memory, with a total capacity of 128GB.
For graphics processing, two NVIDIA RTX 3090 GPUs are used, each featuring 24GB of GDDR6X memory and 10,496 CUDA cores.

The software environment is deployed on the Ubuntu 20.04 LTS operating system. The programming environment utilizes the Python 3.8.12 interpreter, integrated with the CUDA 11.8 toolkit and the cuDNN 8.6 acceleration library to build the deep learning framework.
The primary dependencies include PyTorch 2.0.1, TorchVision 0.15.2, and the NVIDIA Apex mixed-precision training toolkit. 

\subsection{Environments}
\label{sup:env}
\textbf{Atari} is a platform that uses Atari 2600 games for testing and training RL algorithms.
\begin{itemize}
    \item Pong simulates a two-player ping-pong game in which each player controls a paddle that moves vertically. The objective is to hit the ball with the paddle to send it past the opponent's paddle. 

    \item Breakout is a classic arcade game where players control a paddle to bounce a ball and destroy bricks, aiming to destroy as many as possible for a higher score. Each action taken by the agent changes the game state, which then results in either positive or negative rewards based on the outcome.

    \item Chopper Command is a classic side-scrolling shooter game in which players control a military helicopter to protect a convoy of trucks on the ground. The game ends when the player loses all their lives or reaches a score of 999,999 points.

    \item Seaquest is a classic side-scrolling shooter game in which players control a submarine to shoot at enemies and rescue divers. For example, enemies include sharks and submarines that fire missiles at the player's submarine. 

    \item Qbert is a classic arcade game where the player's objective is to control the character Qbert as he jumps on cubes to change each cube in the pyramid to the target color, while avoiding obstacles and enemies. 

    \item Space Invaders is an arcade game where the player controls a laser cannon to eliminate descending aliens. In RL, the agent simulates the player's actions to earn points by moving the laser cannon left and right and firing lasers to eliminate the aliens.

\end{itemize}

$\quad$
\\
\noindent \textbf{DeepMind Control Suite} is a collection of RL environments developed by DeepMind, which includes a series of continuous control tasks with standardized structure and interpretable rewards.
We primarily test the following environments:
\begin{itemize}
    \item HalfCheetah is a 2-dimensional robot consisting of 9 links and 8 joints, including two paws.  Its goal is to exert torque on the joints to move forward (to the right) as quickly as possible, receiving positive rewards for forward movement and negative rewards for backward movement. 

    \item Hopper is a two-dimensional one-legged figure that consist of four main body parts - the torso at the top, the thigh in the middle, the leg in the bottom, and a single foot on which the entire body rests. The objective is to execute hops in the forward (right) direction by applying torques to the three hinges that connect the four body parts.

    \item Walker is a two-dimensional two-legged figure consisting of four main body parts: a single torso at the top, two thighs below the torso, two legs below the thighs, and two feet attached to the legs, supporting the entire body. The goal is to coordinate the movement of both sets of feet, legs, and thighs to move forward (to the right) by applying torques to the six hinges that connect the body parts.

    \item Ant is a four-legged creature with a torso and four limbs, each with two joints. The goal of the Ant is typically to move forward as quickly as possible without falling over. Its reward function encourages forward progress while penalizing high control efforts and contact with the ground.

    \item Humanoid environment is designed to simulate human movement. It has a torso (abdomen) with a pair of legs and arms. Each leg consists of two links, and each arm represents the knees and elbows, respectively.

    \item Pendulum involves a cart that can move linearly, with a pole fixed on it at one end and the other end free. The cart can be pushed left or right, and the goal is to balance the pole on the top of the cart by applying forces on the cart.

\end{itemize}

$\quad$
\\
\noindent \textbf{Gfootball} environment is an RL environment based on the open-source game Gameplay Football, which is created by the Google Brain team for research purposes.
The core of the Gfootball environment is an advanced football simulator called Football Engine, which simulates a football match based on the input actions of two opposing teams, including goals, fouls, corner kicks, penalty kicks, and offside situations.
Simply put, Gfootball is a physics-based 3D simulator of a soccer game environment that simulates a complete soccer game according to standard rules, featuring 11 players per team.
The length of the game is measured in terms of the number of frames, and the default duration of a full game is 3000 (10 frames per second for 5 minutes). 

\begin{table*}[]
\centering
\caption{Hyperparameter}
\label{tab:aat_hyperparameter}
\begin{tabular}{llll}
\hline
\multicolumn{4}{c}{Hyperparameter}                                                                                                                \\ \hline
\multicolumn{2}{c|}{GPT}                                                                                     & \multicolumn{2}{c}{Perturbation network} \\ \hline
Name                 & \multicolumn{1}{l|}{value}                                                            & Name                 & value       \\ \hline
Number of layers     & \multicolumn{1}{l|}{6}                                                                & Linear layer         & 128, 3136   \\ \hline
Attention heads      & \multicolumn{1}{l|}{8}                                                                & Aactivite function   & Rlue        \\ \hline
Embedding dimension  & \multicolumn{1}{l|}{128}                                                              & ConvTranspose2d      & 64,32,4     \\ \hline
Batch size           & \multicolumn{1}{l|}{\begin{tabular}[c]{@{}l@{}}Atari:128\\ DeepMind:256\end{tabular}} & kernel\_size         & 3,4,8       \\ \hline
Nonlinearity         & \multicolumn{1}{l|}{LayerNorm}                                                        &                      &             \\ \hline
Encoder channels     & \multicolumn{1}{l|}{32,64,128}                                                        &                      &             \\ \hline
Encoder filter sizes & \multicolumn{1}{l|}{8,4,3,1}                                                          &                      &             \\ \hline
Decoder channels     & \multicolumn{1}{l|}{128,128}                                                          &                      &             \\ \hline
Dncoder filter sizes & \multicolumn{1}{l|}{1}                                                                &                      &             \\ \hline
Optimizer            & \multicolumn{1}{l|}{Adam}                                                             &                      &             \\ \hline
Learning rate        & \multicolumn{1}{l|}{$1*10^{-3}$}                                         &                      &             \\ \hline
Dropout              & \multicolumn{1}{l|}{0.2}                                                              &                      &             \\ \hline
\end{tabular}
\end{table*}

\subsection{Train dataset}
\label{sec:train_data}
In white-box scenarios, where the internal information of the target policy is accessible, we employ diverse perturbation methods (FGSM, EDGE, Skip, S-T, and random noise\footnote{Random perturbations enhance exploration by injecting noise into the state of policy.}) to interact with the policy and generate learning trajectories. The resulting dataset comprises 40,000 trajectories, with each method contributing 20$\%$ of the data.  
In black-box scenarios, since the target policy information is not known, we use substitution policies A2C, ACER and ACKTR as the target policies.
Then, following the same proportion as in white-box scenarios, we still use FGSM, EDGE, Skip, S-T, and random perturbation methods to interact with substitution policies and collect 40,000 historical trajectories.

\subsection{AAT structure and hyperparameters}
\label{sup:struct}
AAT uses the architecture of minGPT (\url{https://github.com/karpathy/minGPT}) to implement the generation of adversarial examples in Atari, DeepMind Control Suite and Gfootball games, and AAT follows most of the default hyperparameters in its character-level GPT example for most of the default hyperparameters.
Furthermore, we employ a processing method similar to DT by reducing the batch size, block size, number of layers and embedding dimensions to speed up the training process.
To generate adversarial perturbations of the same size as the input, we added a perturbation layer to decode the encoded vector, which includes two linear layers and three deconvolutional layers.
For the advantage function, we inject it after the minGPT encoder, where a linear layer network is used to predict the advantage of the current input.
Table~\ref{tab:aat_hyperparameter} provides the hyperparameter settings of the AAT network.
Moreover, the perturbation threshold $\varepsilon$ is set to 1.5 in the Atari and DeepMind Control Suite environments and to 2.0 in Gfootball.
These configurations can ensure that the perturbations generated by AAT under the $L_2$ norm remain imperceptible to the human eye.

\subsection{Target policies}
\label{sup:target_policy}
\noindent \textbf{Atari}. We employ four DRL algorithms as target policies: DQN, A3C, TRPO, and PPO.
The policies employ the network architecture delineated in, which consists of a convolutional layer with 16 filters of size 8×8 and a stride of 4, followed by another convolutional layer with 32 filters of size 4×4 and a stride of 2.
The final layer consists of a fully-connected layer with 256 hidden units, and 
each hidden layer is followed by a ReLU activation function.
The input of the policy is a concatenation of the last 4 images, which are converted from RGB to luminance (Y) and resized to 84 × 84. 
The output of the policy is a distribution over possible actions.
\\
\\
\textbf{DeepMind control suite}. 
Following Yuval et al., we use A3C and D4PG as the target policies, where A3C and D4PG are augmented with a 2-layer convolutional network (ConvNet) for feature extraction. 
Both layers use 3 × 3 kernels with 32 channels and ELU activation; the first layer applies a stride of 2, and the second layer applies a stride size of 1.
The ConvNet output is fed into a fully connected layer with 50 neurons, incorporating layer normalization and tanh activation. 
The target policy processes input as a stack of four consecutive 84 × 84 RGB images.
\\
\\
\textbf{Gfootball}.
We adopt the settings from work by Karol et al. and use DQN and PPO as target policies.
These policies utilize the network architecture described in, which includes a convolutional layer with a 3 × 3 filter and a stride of 1, followed by a pooling layer with a size of 3 × 3.
In addition, two residual blocks are added, with each block comprising two 3 × 3 convolutions and a ReLU activation.
The final layer consists of a fully-connected layer with 256 hidden units. 
The input of the target policy consists of stacked raw pixels of size 72 x 96 x 4.

\section{Additional experimental results} 
\subsection{Additional comparison results}
\label{sup:discrete_reult}
Tables~\ref{tab:pong_w}, ~\ref{tab:chopper_w},~\ref{tab:seaquest},~\ref{tab:qbert} and~\ref{tab:space_invader} provide the attack results of different attack methods on Pong, Chopper Command, Seaquest, Qbert, and Space Invaders. In the black-box scenario, we utilize A2C, ACER, and ACKTR, with network architectures differing from the target policy, to produce adversarial examples. 
In particular, we use A2C to collect historical data for the Pong game, ACER for Chopper Command and Seaquest, and ACKTR for Qbert.
In addition, Fig.~\ref{fig:sup_cons_a3c} shows the attack results of AAT and other attack methods in a continuous environment.

\begin{table*}[]
\centering
\caption{Comparisons between the ATT attack and other methods in the Pong game. $*$ denotes the cumulative reward of the target policy under no attack. The bold indicates the best attack performance.}
\label{tab:pong_w}
\begin{tabular}{lcccccccc}
\hline
\multicolumn{9}{c}{Pong}                                                                                                                                                                                                                                                                                                                                                                                                                                                                                                                                                                                                                                                                                                                                                                                                                                                                                                                                                                                                                                                                                                                                                                                                                                                                                                                                                                                                                                                                                                                                                                                               \\ \hline
\multicolumn{1}{l|}{}                                                                                                            & \multicolumn{4}{c|}{White-box}                                                                                                                                                                                                                                                                                                                                                                                                                                                                                                                                                                                                                                                                                                                              & \multicolumn{4}{c}{Black-box}                                                                                                                                                                                                                                                                                                                                                                                                                                                                                                                                                                                                                                                                                         \\
\multicolumn{1}{l|}{Method}                                                                                                      & DQN                                                                                                                                                                              & A3C                                                                                                                                                                             & TRPO                                                                                                                                                                           & \multicolumn{1}{c|}{PPO}                                                                                                                                                                              & DQN                                                                                                                                                                         & A3C                                                                                                                                                                       & TRPO                                                                                                                                                                         & PPO                                                                                                                                                                          \\
\multicolumn{1}{l|}{\begin{tabular}[c]{@{}l@{}}*\\ FGSM\\ Skip\\ S-T\\ EDGE\\ PA-AD\\ AdvRL-GAN\\ TSGE\\ PIA\\ AAT\end{tabular}} & \begin{tabular}[c]{@{}c@{}}19.83$\pm$1.31\\ -20.31$\pm$0.48\\ -18.35$\pm$0.74\\ -17.86$\pm$1.23\\ -19.36$\pm$0.98\\ -\textbf{21.00}$\pm$\textbf{0.00}\\ -12.44$\pm$0.22\\ -15.34$\pm$2.37\\ -17.23$\pm$1.95\\ -20.96$\pm$0.04\end{tabular} & \begin{tabular}[c]{@{}c@{}}21.00$\pm$0.00\\ -20.35$\pm$0.25\\ -19.36$\pm$1.03\\ -10.84$\pm$3.41\\ -14.33$\pm$2.54\\ -19.8$\pm$0.21\\ -11.18$\pm$1.78\\ -12.47$\pm$2.44\\ -18.34$\pm$0.58\\ -\textbf{20.61}$\pm$\textbf{0.33}\end{tabular} & \begin{tabular}[c]{@{}c@{}}20.64$\pm$0.32\\ -\textbf{20.48}$\pm$\textbf{0.15}\\ -18.79$\pm$0.27\\ -12.77$\pm$2.14\\ -14.89$\pm$1.95\\ -9.98$\pm$2.13\\ -18.4$\pm$1.21\\ -19.76$\pm$0.34\\ -17.41$\pm$1.43\\ -19.97$\pm$1.24\end{tabular} & \multicolumn{1}{c|}{\begin{tabular}[c]{@{}c@{}}21.00$\pm$0.00\\ -19.56$\pm$0.24\\ -15.42$\pm$2.84\\ -14.66$\pm$2.31\\ -15.13$\pm$3.41\\ -10.08$\pm$3.42\\ -\textbf{20.12}$\pm$\textbf{0.24}\\ -16.38$\pm$2.18\\ -17.73$\pm$0.31\\ -19.87$\pm$0.27\end{tabular}} & \begin{tabular}[c]{@{}c@{}}19.83$\pm$1.31\\ 18.38$\pm$0.54\\ 18.10$\pm$0.21\\ 18.12$\pm$0.06\\ -9.22$\pm$3.45\\ -10.31$\pm$2.78\\ -13.56$\pm$1.44\\ -11.53$\pm$0.97\\ -9.87$\pm$0.12\\ -\textbf{18.34}$\pm$\textbf{0.43}\end{tabular} & \begin{tabular}[c]{@{}c@{}}21.00$\pm$0.00\\ 20.25$\pm$0.14\\ 20.21$\pm$0.34\\ 19.65$\pm$0.46\\ -6.54$\pm$2.37\\ -9.31$\pm$3.56\\ -11.36$\pm$2.77\\ -8.24$\pm$3.07\\ -9.21$\pm$2.89\\ -\textbf{18.45}$\pm$\textbf{1.94}\end{tabular} & \begin{tabular}[c]{@{}c@{}}20.64$\pm$0.32\\ 19.88$\pm$0.23\\ 18.02$\pm$0.71\\ 19.88$\pm$1.02\\ -8.98$\pm$3.07\\ -13.86$\pm$2.94\\ -\textbf{16.36}$\pm$\textbf{0.97}\\ -13.36$\pm$1.41\\ -11.49$\pm$2.51\\ -16.06$\pm$1.35\end{tabular} & \begin{tabular}[c]{@{}c@{}}21.00$\pm$0.00\\ 20.26$\pm$0.34\\ 19.18$\pm$0.81\\ 20.01$\pm$0.43\\ -5.98$\pm$2.34\\ -15.67$\pm$1.95\\ -10.46$\pm$2.47\\ -10.73$\pm$2.85\\ -13.02$\pm$3.14\\ -\textbf{18.98}$\pm$\textbf{1.07}\end{tabular} \\ \hline
\end{tabular}
\end{table*}

\begin{table*}[]
\centering
\caption{Comparisons between the ATT attack and other methods in the Chopper Command game. $*$ denotes the cumulative reward of the target policy under no attack. The bold indicates the best attack performance.}
\label{tab:chopper_w}
\setlength{\tabcolsep}{0.8mm} 
\begin{tabular}{lcccccccc}
\hline
\multicolumn{9}{c}{Chopper Command}                                                                                                                                                                                                                                                                                                                                                                                                                                                                                                                                                                                                                                                                                                                                                                                                                                                                                                                                                                                                                                                                                                                                                                                                                                                                                                                                                                                                                                                                                                                                                                                                                                                                                                                                                      \\ \hline
\multicolumn{1}{l|}{}                                                                                                            & \multicolumn{4}{c|}{White-box}                                                                                                                                                                                                                                                                                                                                                                                                                                                                                                                                                                                                                                                                                                                                                                                              & \multicolumn{4}{c}{Black-box}                                                                                                                                                                                                                                                                                                                                                                                                                                                                                                                                                                                                                                                                                                                                                                                           \\
\multicolumn{1}{l|}{Method}                                                                                                      & DQN                                                                                                                                                                                            & A3C                                                                                                                                                                                              & TRPO                                                                                                                                                                                            & \multicolumn{1}{c|}{PPO}                                                                                                                                                                                              & DQN                                                                                                                                                                                                 & A3C                                                                                                                                                                                                  & TRPO                                                                                                                                                                                                 & PPO                                                                                                                                                                                                 \\
\multicolumn{1}{l|}{\begin{tabular}[c]{@{}l@{}}*\\ FGSM\\ Skip\\ S-T\\ EDGE\\ PA-AD\\ AdvRL-GAN\\ TSGE\\ PIA\\ AAT\end{tabular}} & \begin{tabular}[c]{@{}c@{}}3555.83$\pm$23.54\\ \textbf{401.32}$\pm$\textbf{14.53}\\ 770.21$\pm$15.39\\ 1180.65$\pm$17.41\\ 879.36$\pm$21.06\\ 789.35$\pm$24.68\\ 1461.37$\pm$19.46\\ 876.34$\pm$23.44\\ 854.37$\pm$19.81\\ 408.68$\pm$19.74\end{tabular} & \begin{tabular}[c]{@{}c@{}}3864.12$\pm$19.87\\ 468.36$\pm$17.84\\ 778.33$\pm$21.33\\ 1098.46$\pm$17.49\\ 906.24$\pm$20.19\\ 1335.46$\pm$23.54\\ 1389.56$\pm$23.78\\ 1031.27$\pm$21.94\\ 986.54$\pm$20.84\\ \textbf{410.21}$\pm$\textbf{21.07}\end{tabular} & \begin{tabular}[c]{@{}c@{}}3615.44$\pm$26.41\\ 470.32$\pm$19.97\\ 789.34$\pm$21.42\\ 1206.26$\pm$25.46\\ 896.63$\pm$21.53\\ 784.36$\pm$25.19\\ 1280.62$\pm$22.97\\ 987.31$\pm$22.43\\ 1052.48$\pm$21.68\\ \textbf{420.98}$\pm$\textbf{20.97}\end{tabular} & \multicolumn{1}{c|}{\begin{tabular}[c]{@{}c@{}}3842.52$\pm$28.33\\ 469.36$\pm$20.89\\ 780.56$\pm$21.41\\ 1240.36$\pm$23.01\\ 900.44$\pm$20.89\\ 578.69$\pm$24.81\\ 1350.43$\pm$23.07\\ 1003.42$\pm$22.74\\ 1203.47$\pm$20.08\\ \textbf{418.36}$\pm$\textbf{19.35}\end{tabular}} & \begin{tabular}[c]{@{}c@{}}3555.83$\pm$23.54\\ 2015.46$\pm$29.35\\ 2898.36$\pm$31.04\\ 2754.38$\pm$28.74\\ 2012.46$\pm$31.14\\ 893.72$\pm$30.82\\ 1864.32$\pm$31.79\\ 1562.49$\pm$34.03\\ 1671.39$\pm$29.04\\ \textbf{568.56}$\pm$\textbf{27.89}\end{tabular} & \begin{tabular}[c]{@{}c@{}}3864.12$\pm$19.87\\ 2212.42$\pm$31.81\\ 2934.98$\pm$29.87\\ 2856.42$\pm$30.47\\ 2984.56$\pm$33.04\\ 1893.67$\pm$31.78\\ 1906.48$\pm$32.46\\ 1935.47$\pm$32.75\\ 2085.81$\pm$30.84\\ \textbf{628.54}$\pm$\textbf{30.07}\end{tabular} & \begin{tabular}[c]{@{}c@{}}3615.44$\pm$26.41\\ 2189.38$\pm$31.25\\ 2856.94$\pm$32.18\\ 2789.35$\pm$31.45\\ 2365.38$\pm$32.85\\ 1352.71$\pm$30.94\\ 1832.91$\pm$29.67\\ 1489.26$\pm$30.85\\ 1793.91$\pm$32.42\\ \textbf{609.88}$\pm$\textbf{30.33}\end{tabular} & \begin{tabular}[c]{@{}c@{}}3842.52$\pm$28.33\\ 2306.85$\pm$30.84\\ 2890.52$\pm$29.47\\ 2745.35$\pm$33.42\\ 2689.36$\pm$32.81\\ 987.62$\pm$31.74\\ 1896.54$\pm$32.83\\ 1328.46$\pm$29.84\\ 1573.23$\pm$33.45\\ \textbf{635.86}$\pm$\textbf{29.04}\end{tabular} \\ \hline
\end{tabular}
\end{table*}

\begin{table*}[]
\centering
\caption{Comparisons between the ATT attack and other methods in the Seaquest game. The expected cumulative rewards achieved by DQN, A3C, TRPO and PPO target policies are 2385.32, 2856.89, 1856.89 and 2108.38 respectively. The bold indicates the best attack performance.}
\label{tab:seaquest}
\setlength{\tabcolsep}{0.8mm} 
\begin{tabular}{lcccccccc}
\hline
\multicolumn{9}{c}{Seaquest}                                                                                                                                                                                                                                                                                                                                                                                                                                                                                                                                                                                                                                                                                                                                                                                                                                                                                                                                                                                                                                                                                                                                                                                                                                                                                                                                                                                                                                                                                                                                                                                                \\ \hline
\multicolumn{1}{l|}{}                                                                                                            & \multicolumn{4}{c|}{White-box}                                                                                                                                                                                                                                                                                                                                                                                                                                                                                                                                                                                                                                                                                                                     & \multicolumn{4}{c}{Black-box}                                                                                                                                                                                                                                                                                                                                                                                                                                                                                                                                                                                                                                                                                                       \\ \hline
\multicolumn{1}{l|}{Method}                                                                                                      & DQN                                                                                                                                                                          & A3C                                                                                                                                                                           & TRPO                                                                                                                                                                          & \multicolumn{1}{c|}{PPO}                                                                                                                                                                            & DQN                                                                                                                                                                                & A3C                                                                                                                                                                           & TRPO                                                                                                                                                                          & PPO                                                                                                                                                                            \\ \hline
\multicolumn{1}{l|}{\begin{tabular}[c]{@{}l@{}}*\\ FGSM\\ Skip\\ S-T\\ EDGE\\ PA-AD\\ AdvRL-GAN\\ TSGE\\ PIA\\ AAT\end{tabular}} & \begin{tabular}[c]{@{}c@{}}2385.32$\pm$15.43\\ \textbf{17.35}$\pm$\textbf{3.56}\\ 89.42$\pm$3.02\\ 150.35$\pm$2.97\\ 75.46$\pm$4.03\\ 18.97$\pm$0.98\\ 108.23$\pm$4.56\\ 19.32$\pm$2.78\\ 23.71$\pm$3.54\\ 18.34$\pm$1.07\end{tabular} & \begin{tabular}[c]{@{}c@{}}2856.96$\pm$16.84\\ 65.42$\pm$3.07\\ 158.34$\pm$4.13\\ 221.84$\pm$3.27\\ 98.55$\pm$3.08\\ \textbf{17.63}$\pm$\textbf{2.41}\\ 129.33$\pm$2.83\\ 35.43$\pm$1.76\\ 48.29$\pm$2.81\\ 20.22$\pm$1.23\end{tabular} & \begin{tabular}[c]{@{}c@{}}1856.89$\pm$14.87\\ 77.54$\pm$3.21\\ 198.33$\pm$3.44\\ 210.04$\pm$2.97\\ 98.45$\pm$3.12\\ 23.54$\pm$2.87\\ 119.83$\pm$3.04\\ 54.82$\pm$2.14\\ 49.41$\pm$2.52\\ \textbf{19.84}$\pm$\textbf{2.03}\end{tabular} & \multicolumn{1}{c|}{\begin{tabular}[c]{@{}c@{}}2108.38$\pm$16.53\\ 74.32$\pm$2.17\\ 178.26$\pm$3.82\\ 256.83$\pm$3.04\\ 100.82$\pm$3.51\\ 24.81$\pm$3.71\\ 132.42$\pm$2.87\\ 54.39$\pm$3.17\\ 53.28$\pm$2.89\\ \textbf{21.42}$\pm$\textbf{3.07}\end{tabular}} & \begin{tabular}[c]{@{}c@{}}2385.32$\pm$15.43\\ 290.32$\pm$5.74\\ 300.48$\pm$6.52\\ 328.18$\pm$4.31\\ 298.43$\pm$7.98\\ 178.46$\pm$5.04\\ 245.44$\pm$6.15\\ 167.89$\pm$6.71\\ 216.74$\pm$7.16\\ \textbf{80.45}$\pm$\textbf{5.34}\end{tabular} & \begin{tabular}[c]{@{}c@{}}2856.96$\pm$16.84\\ 65.42$\pm$7.32\\ 158.34$\pm$6.89\\ 221.84$\pm$7.18\\ 98.55$\pm$7.24\\ \textbf{17.63}$\pm$\textbf{5.37}\\ 129.33$\pm$7.05\\ 35.43$\pm$6.05\\ 48.29$\pm$6.19\\ 20.22$\pm$5.23\end{tabular} & \begin{tabular}[c]{@{}c@{}}1856.89$\pm$14.87\\ 77.54$\pm$6.83\\ 198.33$\pm$8.47\\ 210.04$\pm$5.24\\ 98.45$\pm$6.53\\ 23.54$\pm$4.32\\ 119.83$\pm$3.56\\ 54.82$\pm$5.41\\ 49.41$\pm$4.32\\ \textbf{19.84}$\pm$\textbf{3.17}\end{tabular} & \begin{tabular}[c]{@{}c@{}}2108.38$\pm$16.53\\ 74.32$\pm$6.44\\ 178.26$\pm$7.59\\ 256.83$\pm$6.37\\ 100.82$\pm$6.04\\ 24.81$\pm$3.21\\ 132.42$\pm$4.68\\ 54.39$\pm$5.24\\ 53.28$\pm$4.52\\ \textbf{21.42}$\pm$\textbf{4.37}\end{tabular} \\ \hline
\end{tabular}
\end{table*}

\begin{table*}[]
\centering
\caption{Comparisons between the ATT attack and other methods in the Qbert game. $*$ denotes the cumulative reward of the target policy under no attack. The bold indicates the best attack performance.}
\label{tab:qbert}
\setlength{\tabcolsep}{0.8mm} 
\begin{tabular}{lcccccccc}
\hline
\multicolumn{9}{c}{Qbert}                                                                                                                                                                                                                                                                                                                                                                                                                                                                                                                                                                                                                                                                                                                                                                                                                                                                                                                                                                                                                                                                                                                                                                                                                                                                                                                                                                                                                                                                                                                                                                                                                           \\ \hline
\multicolumn{1}{l|}{}                                                                                                            & \multicolumn{4}{c|}{White-box}                                                                                                                                                                                                                                                                                                                                                                                                                                                                                                                                                                                                                                                                                                                                    & \multicolumn{4}{c}{Black-box}                                                                                                                                                                                                                                                                                                                                                                                                                                                                                                                                                                                                                                                                                                                \\ \hline
\multicolumn{1}{l|}{Method}                                                                                                      & DQN                                                                                                                                                                              & A3C                                                                                                                                                                               & TRPO                                                                                                                                                                              & \multicolumn{1}{c|}{PPO}                                                                                                                                                                               & DQN                                                                                                                                                                              & A3C                                                                                                                                                                               & TRPO                                                                                                                                                                              & PPO                                                                                                                                                                               \\ \hline
\multicolumn{1}{l|}{\begin{tabular}[c]{@{}l@{}}*\\ FGSM\\ Skip\\ S-T\\ EDGE\\ PA-AD\\ AdvRL-GAN\\ TSGE\\ PIA\\ AAT\end{tabular}} & \begin{tabular}[c]{@{}c@{}}17890.93$\pm$2.35\\ 165.43$\pm$2.49\\ 678.35$\pm$3.13\\ 543.44$\pm$2.87\\ 235.68$\pm$3.04\\ 152.34$\pm$2.89\\ 124.82$\pm$1.87\\ \textbf{48.39}$\pm$\textbf{4.46}\\ 98.93$\pm$3.51\\ 50.56$\pm$1.74\end{tabular} & \begin{tabular}[c]{@{}c@{}}21840.35$\pm$3.41\\ 189.34$\pm$2.93\\ 543.27$\pm$3.97\\ 569.38$\pm$4.08\\ 253.96$\pm$2.96\\ 148.92$\pm$3.18\\ 139.84$\pm$2.47\\ 83.98$\pm$3.21\\ 129.46$\pm$4.32\\ \textbf{60.22}$\pm$\textbf{2.19}\end{tabular} & \begin{tabular}[c]{@{}c@{}}22840.56$\pm$1.52\\ 205.33$\pm$3.01\\ 667.32$\pm$2.95\\ 510.35$\pm$3.17\\ 199.84$\pm$3.46\\ 175.62$\pm$3.58\\ 129.83$\pm$4.02\\ \textbf{58.01}$\pm$\textbf{3.96}\\ 109.49$\pm$2.34\\ 58.42$\pm$2.78\end{tabular} & \multicolumn{1}{c|}{\begin{tabular}[c]{@{}c@{}}2458.68$\pm$2.87\\ 210.32$\pm$3.73\\ 680.34$\pm$3.15\\ 568.23$\pm$4.01\\ 1246.21$\pm$2.77\\ 101.23$\pm$3.28\\ 130.58$\pm$3.92\\ 94.35$\pm$2.81\\ 117.43$\pm$3.07\\ \textbf{54.32}$\pm$\textbf{2.81}\end{tabular}} & \begin{tabular}[c]{@{}c@{}}17890.93$\pm$2.35\\ 165.43$\pm$5.46\\ 678.35$\pm$8.32\\ 543.44$\pm$3.45\\ 235.68$\pm$7.01\\ 152.34$\pm$5.32\\ 124.82$\pm$4.98\\ \textbf{48.39}$\pm$\textbf{5.46}\\ 98.93$\pm$4.92\\ 50.56$\pm$3.41\end{tabular} & \begin{tabular}[c]{@{}c@{}}21840.35$\pm$3.41\\ 189.34$\pm$6.42\\ 543.27$\pm$6.51\\ 569.38$\pm$6.73\\ 253.96$\pm$5.82\\ 148.92$\pm$7.03\\ 139.84$\pm$6.54\\ 83.98$\pm$4.98\\ 129.46$\pm$5.94\\ \textbf{60.22}$\pm$\textbf{5.33}\end{tabular} & \begin{tabular}[c]{@{}c@{}}22840.56$\pm$1.52\\ 205.33$\pm$5.93\\ 667.32$\pm$6.48\\ 510.35$\pm$6.82\\ 199.84$\pm$7.03\\ 175.62$\pm$6.35\\ 129.83$\pm$5.98\\ \textbf{58.01}$\pm$\textbf{7.13}\\ 109.49$\pm$6.32\\ 58.42$\pm$5.18\end{tabular} & \begin{tabular}[c]{@{}c@{}}2458.68$\pm$2.87\\ 210.32$\pm$6.25\\ 680.34$\pm$5.48\\ 568.23$\pm$6.95\\ 1246.21$\pm$5.39\\ 101.23$\pm$6.33\\ 130.58$\pm$6.43\\ 94.35$\pm$7.36\\ 117.43$\pm$5.83\\ \textbf{54.32}$\pm$\textbf{4.73}\end{tabular} \\ \hline
\end{tabular}
\end{table*}

\begin{table*}[]
\centering
\caption{Comparisons between the ATT attack and other methods in the Space Invaders game. $*$ denotes the cumulative reward of the target policy under no attack. The bold indicates the best attack performance.}
\label{tab:space_invader}
\setlength{\tabcolsep}{0.8mm}
\begin{tabular}{lcccccccc}
\hline
\multicolumn{9}{c}{Space Invaders}                                                                                                                                                                                                                                                                                                                                                                                                                                                                                                                                                                                                                                                                                                                                                                                                                                                                                                                                                                                                                                                                                                                                                                                                                                                                                                                                                                                                                                                                                                                                                                                                        \\ \hline
\multicolumn{1}{l|}{}                                                                                                            & \multicolumn{4}{c|}{White-box}                                                                                                                                                                                                                                                                                                                                                                                                                                                                                                                                                                                                                                                                                                                               & \multicolumn{4}{c}{Black-box}                                                                                                                                                                                                                                                                                                                                                                                                                                                                                                                                                                                                                                                                                                           \\ \hline
\multicolumn{1}{l|}{Method}                                                                                                      & DQN                                                                                                                                                                               & A3C                                                                                                                                                                             & TRPO                                                                                                                                                                             & \multicolumn{1}{c|}{PPO}                                                                                                                                                                            & DQN                                                                                                                                                                               & A3C                                                                                                                                                                             & TRPO                                                                                                                                                                             & PPO                                                                                                                                                                            \\ \hline
\multicolumn{1}{l|}{\begin{tabular}[c]{@{}l@{}}*\\ FGSM\\ Skip\\ S-T\\ EDGE\\ PA-AD\\ AdvRL-GAN\\ TSGE\\ PIA\\ AAT\end{tabular}} & \begin{tabular}[c]{@{}c@{}}1005.38$\pm$2.41\\ 123.50$\pm$3.42\\ 169.80$\pm$4.57\\ 210.32$\pm$2.78\\ 156.33$\pm$4.02\\ 151.27$\pm$2.54\\ 160.32$\pm$3.87\\ 103.42$\pm$2.74\\ 123.57$\pm$1.98\\ \textbf{59.59}$\pm$\textbf{2.03}\end{tabular} & \begin{tabular}[c]{@{}c@{}}1185.42$\pm$1.98\\ 136.34$\pm$4.07\\ 145.83$\pm$4.26\\ 189.56$\pm$3.52\\ 189.32$\pm$3.49\\ 163.52$\pm$2.97\\ 176.32$\pm$3.44\\ 86.43$\pm$3.14\\ 98.47$\pm$2.18\\ \textbf{51.26}$\pm$\textbf{2.16}\end{tabular} & \begin{tabular}[c]{@{}c@{}}1456.26$\pm$2.74\\ 138.23$\pm$3.78\\ 141.56$\pm$4.01\\ 208.93$\pm$3.42\\ 172.34$\pm$2.97\\ 132.84$\pm$3.15\\ 166.26$\pm$3.28\\ 79.98$\pm$3.61\\ 104.72$\pm$2.73\\ \textbf{60.35}$\pm$\textbf{2.32}\end{tabular} & \multicolumn{1}{c|}{\begin{tabular}[c]{@{}c@{}}1653.48$\pm$2.43\\ 149.23$\pm$3.52\\ 185.39$\pm$3.93\\ 219.27$\pm$4.01\\ 188.93$\pm$3.58\\ 89.73$\pm$3.82\\ 170.32$\pm$3.75\\ \textbf{59.87}$\pm$\textbf{3.93}\\ 93.42$\pm$3.07\\ 61.38$\pm$2.54\end{tabular}} & \begin{tabular}[c]{@{}c@{}}1005.38$\pm$2.41\\ 123.50$\pm$5.43\\ 169.80$\pm$7.02\\ 210.32$\pm$6.54\\ 156.33$\pm$5.93\\ 151.27$\pm$6.28\\ 160.32$\pm$4.35\\ 103.42$\pm$5.28\\ 123.57$\pm$6.53\\ \textbf{68.59}$\pm$\textbf{4.02}\end{tabular} & \begin{tabular}[c]{@{}c@{}}1185.42$\pm$1.98\\ 136.34$\pm$6.03\\ 145.83$\pm$6.43\\ 189.56$\pm$5.86\\ 189.32$\pm$6.87\\ 163.52$\pm$7.13\\ 176.32$\pm$4.56\\ 86.43$\pm$5.23\\ 98.47$\pm$6.33\\ \textbf{51.26}$\pm$\textbf{5.27}\end{tabular} & \begin{tabular}[c]{@{}c@{}}1456.26$\pm$2.74\\ 138.23$\pm$5.91\\ 141.56$\pm$6.23\\ 208.93$\pm$5.97\\ 172.34$\pm$6.37\\ 132.84$\pm$6.84\\ 166.26$\pm$5.36\\ 79.98$\pm$6.38\\ 104.72$\pm$5.86\\ \textbf{60.35}$\pm$\textbf{4.89}\end{tabular} & \begin{tabular}[c]{@{}c@{}}1653.48$\pm$2.43\\ 149.23$\pm$6.32\\ 185.39$\pm$6.74\\ 219.27$\pm$5.86\\ 188.93$\pm$7.03\\ 89.73$\pm$6.74\\ 170.32$\pm$5.84\\ \textbf{59.87}$\pm$\textbf{5.98}\\ 93.42$\pm$6.07\\ 61.38$\pm$5.07\end{tabular} \\ \hline
\end{tabular}
\end{table*}

\begin{figure*}[t]
  \centering
  \subfigure{\includegraphics[width=0.44\textwidth]{hopper_result.pdf}}
  \subfigure{\includegraphics[width=0.44\textwidth]{walker_result.pdf}}
  \subfigure{\includegraphics[width=0.44\textwidth]{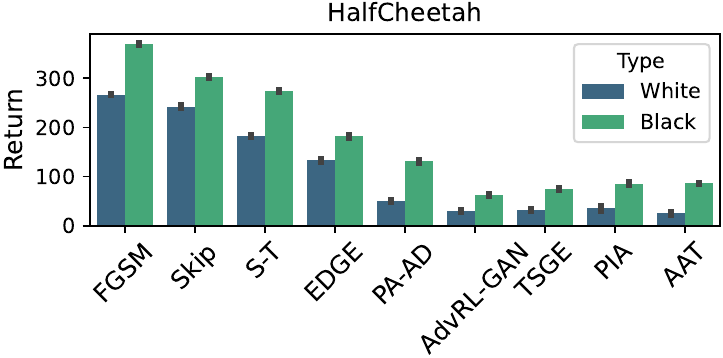}}
  \subfigure{\includegraphics[width=0.44\textwidth]{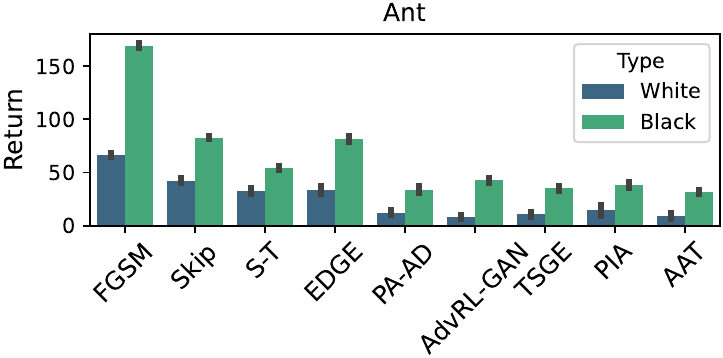}}
  \caption{All experimental results are the average of 10 experiments. In the black-box experimental results, we generate adversarial examples using A3C and D4PG as substitution strategies, respectively.}
  \label{fig:sup_cons_a3c}
\end{figure*}

\subsection{Defense evaluation}
\label{sup:defense_eval}
To further verify the defense capability of existing methods against AAT, we employ two classical strategies: (1) applying spatial smoothing (i.e., median filtering, MF) to mitigate input perturbations, and (2) reconstructing the input states using MimicDiffusion~\cite{DBLP:SongLP024} to remove adversarial noise.
Table~\ref{sup:defance_rec} presents the performance comparison of four fundamental reinforcement learning methods (DQN, A3C, TRPO, and PPO) in the Pong environment, after being processed by two defense methods (MF and MimicDiffusion), under eight different adversarial attacks (FGSM, Skip, S-T, EDGE, PA-AD, AdvRL-GAN, TSGE, PIA, and AAT).
The performance metric is the average score achieved by each method in the Pong environment, where lower scores indicate that the defense method is less capable of defending against adversarial examples.

\begin{table*}[]
\centering
\caption{Defense results of smoothing and input reconstruction against adversarial examples. $*$ denotes the cumulative reward obtained by the policy in the absence of an attack.}
\begin{tabular}{lcccccccc}
\hline
\multicolumn{9}{c}{Pong}                                                                                                                                       \\ \hline
          & \multicolumn{4}{c|}{MF}                                                            & \multicolumn{4}{c}{MimicDiffusion}                            \\
Method    & DQN           & A3C           & TRPO          & \multicolumn{1}{c|}{PPO}           & DQN           & A3C           & TRPO          & PPO           \\
*         & 18.18         & 19.13         & 18.72         & \multicolumn{1}{c|}{20.02}         & 19.48         & 20.22         & 19.84         & 20.46         \\
FGSM      & 17.65         & 18.32         & 18.04         & \multicolumn{1}{c|}{18.73}         & 18.74         & 18.97         & 19.12         & 19.34         \\
Skip      & 17.94         & 18.78         & 18.13         & \multicolumn{1}{c|}{17.94}         & 18.97         & 19.34         & 18.74         & 19.64         \\
S-T       & 17.44         & 17.99         & 18.14         & \multicolumn{1}{c|}{18.24}         & 18.41         & 18.95         & 19.13         & 18.41         \\
EDGE      & 16.7          & 17.41         & 16.78         & \multicolumn{1}{c|}{15.43}         & 16.98         & 17.42         & 17.34         & 17.48         \\
PA-AD     & 9.43          & 10.23         & 11.27         & \multicolumn{1}{c|}{8.97}          & 13.57         & 14.36         & 13.87         & 14.03         \\
AdvRL-GAN & 3.41          & 4.27          & 4.71          & \multicolumn{1}{c|}{\textbf{3.98}} & \textbf{4.54} & 6.73          & 7.51          & \textbf{6.79} \\
TSGE      & 7.84          & 6.73          & 5.97          & \multicolumn{1}{c|}{6.77}          & 10.98         & 13.57         & 12.73         & 13.21         \\
PIA       & 4.56          & 5.41          & \textbf{3.97} & \multicolumn{1}{c|}{4.57}          & 9.83          & 9.14          & 8.73          & 9.42          \\
AAT       & \textbf{2.97} & \textbf{3.54} & 4.15          & \multicolumn{1}{c|}{4.27}          & 5.43          & \textbf{4.54} & \textbf{5.44} & 7.12          \\ \hline
\end{tabular}
\label{sup:defance_rec}
\end{table*}

Both defense mechanisms can effectively mitigate performance degradation when confronted with attack methods such as FGSM, Skip, and S-T.
Under $\text{PA-AD}$ and $\text{EDGE}$ attacks, the effectiveness of the $\text{MF}$ defense begins to significantly decline, with the performance of all strategies dropping below $\text{11.27}$.
In contrast, the policy under the MimicDiffusion defense can still maintain performance in the range of 13 to 14 points.
However, the defense performance of MF and MimicDiffusion is severely challenged when confronted with attacks such as $\text{AdvRL-GAN}$, $\text{PIA}$, and $\text{AAT}$. 
For example, under the $\text{MF}$ defense, the performance scores of the $\text{DQN}$ and $\text{A3C}$ policies in the $\text{AAT}$ attack drop sharply to $\text{2.97}$ and $\text{3.54}$, respectively.

\subsection{Adversarial perturbations concealment}
\label{sup:adver_noise_conceal}
Adversarial examples aim to mislead the policy into making incorrect decisions by introducing imperceptible perturbations that minimally alter the original input to evade detection.
To analyze these perturbations, we visually contrast adversarial examples with their original examples and quantify the perturbation intensity using mean squared error (MSE) and SSIM. 
Fig.~\ref{fig:vis_noise} shows the original states, adversarial examples, adversarial perturbations, and MSE of different attack perturbations in the games Pong and Chopper Command.

\begin{figure}[]
    \centering
  \subfigure[Pong]{\includegraphics[width=0.45\textwidth]{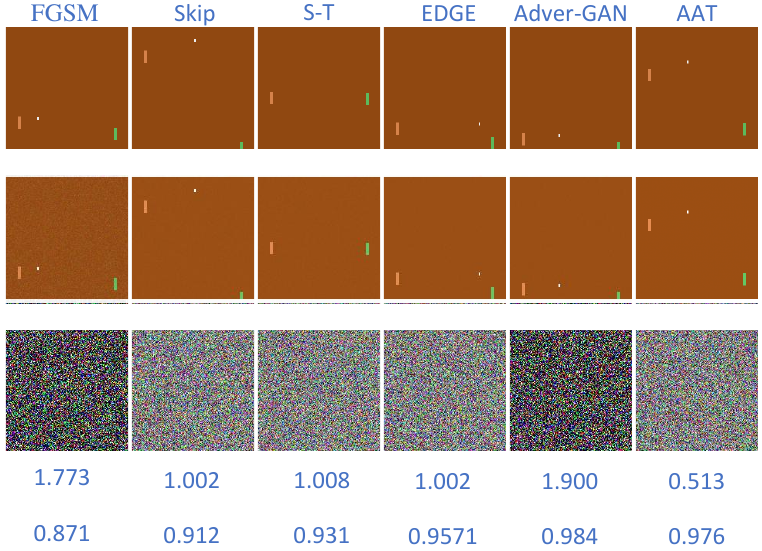}}
  \subfigure[Chopper Command]{\includegraphics[width=0.45\textwidth]{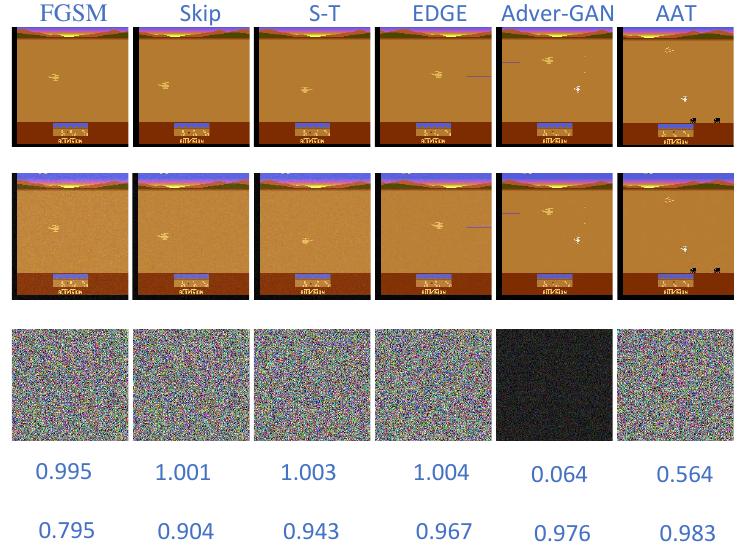}}
  \caption{Adversarial example visualization. In each subfigure, the first row is the original state, the second row is the adversarial example, the third row is the adversarial perturbation that is visualizable, the fourth row is the MSE between the original state and the adversarial example and five row is the SSIM metric. The six columns correspond to FGSM, Skip, S-T, EDGE, advRL-GAN, and AAT respectively.}
  \label{fig:vis_noise}
\end{figure}

We can observe that AAT generates adversarial examples with minimal perturbations, while other attack methods generate adversarial examples that exhibit significant differences from the original states.
Specifically, the FGSM-based attack generates adversarial examples with significant noise traces.
Although reducing the perturbation value can better conceal these noises, it also leads to a decline in attack performance.
By using the C$\&$W method, EDGE is able to effectively mask adversarial perturbations even with a significantly higher MSE value.
AdvRL-GAN employs a generative network to generate adversarial examples and can generate natural perturbations (i.e., higher SSIM values) even with a significant MSE gap.
However, some studies indicate that the stability of adversarial examples generated by generative networks cannot be guaranteed (e.g., generating garbled noise).
In contrast, the AAT method employs a loss optimization approach that considers both the concealment of adversarial examples and their attack performance during the generation process.

\noindent \textbf{Adversarial example detection}.
To further validate the stealthiness of the adversarial examples generated by AAT, we employ two detection methods: the Expected Perturbation Score-based adversarial detection (EPS-AD)~\cite{DBLP:ZhangLYYL0T23} and Layer Regression (LR)~\cite{DBLP:MumcuY25}. 
EPS-AD quantifies the difference between test samples and natural samples using the maximum mean discrepancy of the expected perturbation score.
LR identifies adversarial examples by analyzing how attacks affect different DNN layers to varying degrees.
In our experiments, EPS-AD and LR detect 1,000 adversarial states generated by each attack method. 
We use the detection rate, defined as $r = \frac{n}{m}$, where $n$ is the number of successfully detected adversarial examples and $m$ is the total number tested.
A low detection rate shows that the attack has strong evasion capability, indicating that the generated adversarial examples have better stealthiness.

\begin{table*}[]
\centering
\caption{The results of adversarial detection.}
\begin{tabular}{llccccccccc}
\hline
Detection method & Environment & FGSM & Skip & S-T  & EDGE & PA-AD & AdvRL-GAN     & TSGE & PIA           & AAT           \\ \hline
                 & Pong        & 0.83 & 0.74 & 0.69 & 0.63 & 0.52  & \textbf{0.43} & 0.53      & 0.44          & 0.45          \\
EPS-AD           & Squest      & 0.87 & 0.65 & 0.74 & 0.64 & 0.66  & 0.65          & 0.57      & 0.49          & \textbf{0.43} \\
                 & Qbert       & 0.79 & 0.67 & 0.73 & 0.61 & 0.59  & 0.56          & 0.48      & \textbf{0.44} & 0.47          \\ \hline
                 & Pong        & 0.77 & 0.54 & 0.68 & 0.71 & 0.64  & \textbf{0.34} & 0.39      & 0.41          & 0.35          \\
LR               & Squest      & 0.81 & 0.61 & 0.79 & 0.68 & 0.45  & 0.45          & 0.54      & 0.42          & \textbf{0.39} \\
                 & Qbert       & 0.80 & 0.69 & 0.82 & 0.72 & 0.47  & \textbf{0.35} & 0.48      & 0.39          & \textbf{0.35} \\ \hline
\end{tabular}
\label{tab:sup_detection_results}
\end{table*}

Table~\ref{tab:sup_detection_results} presents the detection performance of two detection methods (EPS-AD and LR) on states generated by various adversarial attack models across three Atari environments (Pong, Seaquest, and Qbert).
Overall, FGSM and Skip attacks are generally easily detected by two types of detection methods, with their detection rates maintained at a high level (in most cases $>$ 0.70).
In comparison, the AAT and generative model-based AdvRL-GAN attacks demonstrate greater stealthiness, exhibiting significantly lower detection rates.
For example, in the Pong environment, the detection rates of EPS-AD for AdvRL-GAN and AAT are 0.43 and 0.45, respectively, while those of LR further decrease to 0.34 and 0.35.
This result indicates that AAT can effectively evade adversarial detection mechanisms.
Furthermore, performance differences also exist between different environments. 
Taking the Seaquest and Qbert environments as examples, the detection rates of the two detection methods under the same attacks are slightly improved compared to the Pong environment. 
This finding indicates that environmental complexity may influence the detectability of adversarial perturbations, suggesting that future research should consider the impact of task attributes and differences in observation spaces on adversarial detection.

\subsection{Substitution policy study}
\label{sup:substitu_policy}
In this subsection, we verify that AAT learns to generate adversarial examples using multiple substitution policies in black-box scenarios.
The training data is collected using three schemes: 1) A2C is employed as the target policy to collect 40,000 trajectories (i.e., $D_1$); 2) A2C and ACER are used as target policies to collect 40,000 trajectories (i.e., $D_2$), with each policy contributing 20,000; 3) A2C, ACER, and ACKTR are used as target policies to collect trajectories (i.e., $D_3$), where A2C and ACER contribute 15,000 each, and ACKTR contributes 10,000.
Fig.~\ref{fig:alter_policy_image} shows the attack results of AAT using a single substitution policy and multiple substitution policies against the target policy in Pong and Qbert environments.
Our experiments indicate that the performance of AAT in black-box scenarios is largely unaffected by the specific choice of substitution policies.
For example, in the Pong and Qbert environments, data collected by either a single substitution  policy or multiple substitution policies can enable AAT to achieve similar attack performance.
However, multiple substitution policies can lead to AAT requiring a longer time to achieve stable attack performance.

\begin{figure*}[h]
  \centering
  \subfigure{\includegraphics[width=0.24\textwidth]{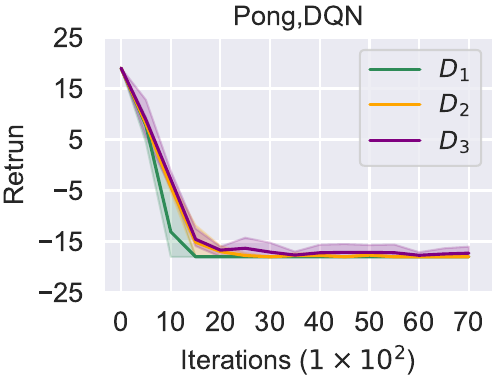}}
   \subfigure{\includegraphics[width=0.24\textwidth]{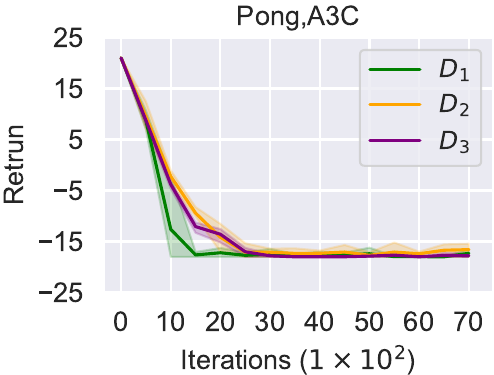}}
   \subfigure{\includegraphics[width=0.24\textwidth]{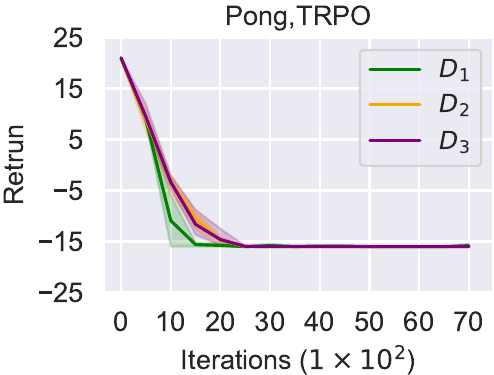}}
  \subfigure{\includegraphics[width=0.24\textwidth]{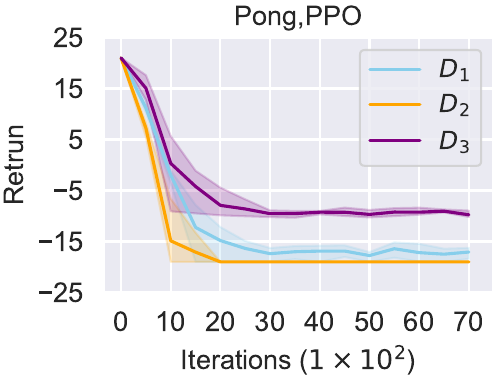}}
  \subfigure{\includegraphics[width=0.24\textwidth]{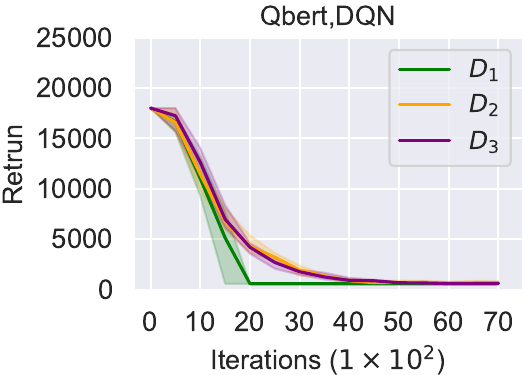}}
   \subfigure{\includegraphics[width=0.24\textwidth]{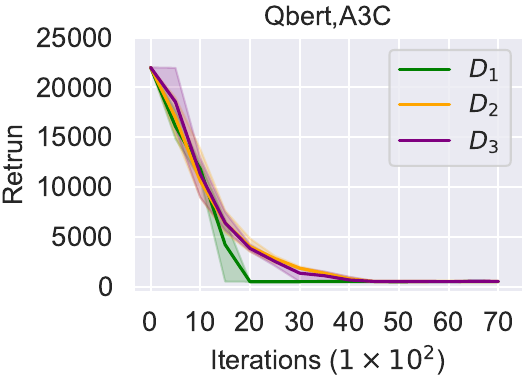}}
   \subfigure{\includegraphics[width=0.24\textwidth]{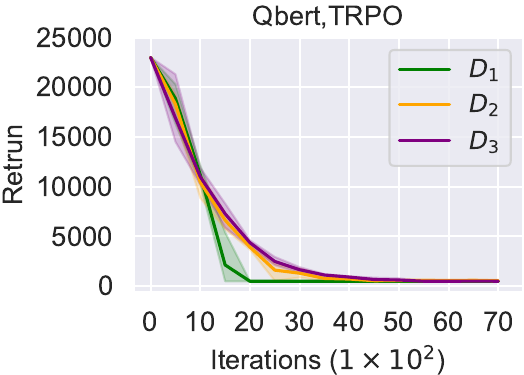}}
  \subfigure{\includegraphics[width=0.24\textwidth]{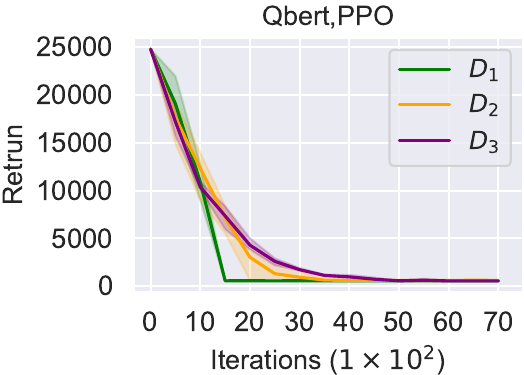}}
  \caption{The impact of single and multiple substitution policies in black-box scenarios.}
  \label{fig:alter_policy_image}
\end{figure*}

\subsection{Training data analysis}
\label{sup:training_data}
\begin{figure}[h]
  \centering
  \subfigure{\includegraphics[width=0.22\textwidth]{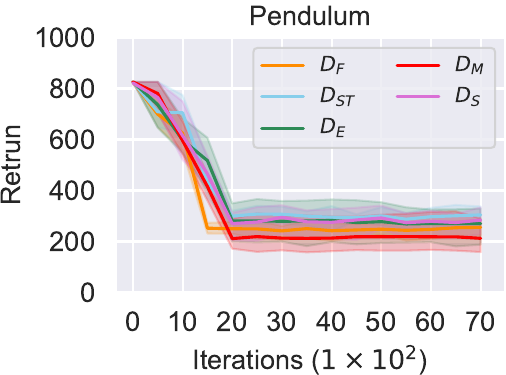}}
  \subfigure{\includegraphics[width=0.22\textwidth]{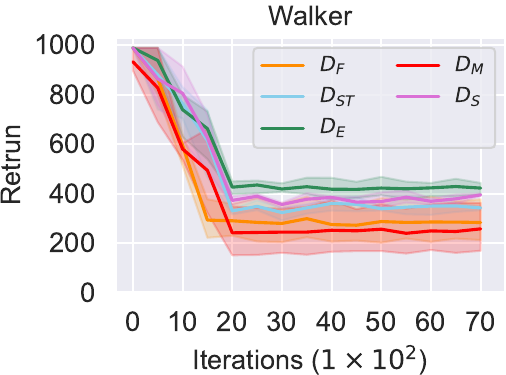}}
  \caption{The impact of data diversity on AAT performance. The y-axis represents the cumulative rewards of agent, while the x-axis is the training epochs of AAT.}
  \label{fig:data_type_picture}
\end{figure}

\noindent\textbf{Training data diversity}.
Data collected from multiple attack strategies can enhance AAT’s performance compared to data from a single strategy, as diverse adversarial patterns improve attack generalization. 
To verify this assumption, we use D4PG as the target policy, and evaluate AAT trained on datasets collected in different ways in the Pendulum and Walker environments.
AAT learns adversarial perturbations using datasets from five sources: FGSM, Skip, EDGE, S-T, and a balanced mixture of all four strategies, where each type of data contains 10,000 trajectories to ensure statistical reliability.
Fig.~\ref{fig:data_type_picture} presents the evaluation results of AAT learning using 10 random seeds on datasets collected from various attack strategies.
Here, $D_F$, $D_S$, $D_{ST}$, and $D_E$ denote datasets from FGSM, Skip, S-T, and EDGE, respectively, while $D_M$ contains an equal mixture (25$\%$ each) to balance diversity and volume. 
Notably, $D_M$ achieves the highest performance, which indicates that adding diverse perturbations during the training process can endow AAT with stronger attack capabilities.

$\quad$
\\
\noindent\textbf{Training data collection method}.
In order to explore the impact of different data collection methods on AAT, we study ablation experiments using the Pong and Chopper Command environments.
Specifically, AAT learns to generate adversarial examples using data collected from four attack strategies: FGSM, Skip, EDGE and AdvRL-GAN, with each attack strategy collecting 20,000 historical trajectories.
Table~\ref{tab:data_type} shows the average performance over 10 runs of AAT, where $D_{FGSM}$ denotes data collected using FGSM, $D_{Skip}$ efers to data obtained through Skip, $D_{EDGE}$ represents data gathered via EDGE, and $D_{AGAN}$ corresponds to data collected by AdvRL-GAN.
Experimental results indicate that the data collected by most methods has no obvious impact on the attack performance of AAT.

$\quad$
\\
\noindent\textbf{Historical data learning advantages}.
In this section, we use experiments to illustrate the advantages of AAT in scenarios where real-time interaction is limited.
Suppose there is a target model (e.g., PPO) that restricts each user to at most 1,000 epochs per day due to cost constraints, where each epoch corresponds to a trajectory.
To overcome the constraints of real-time interactive learning, the RL-based attack method (i.e., AdvRL-GAN) leverages a continual learning approach to learn the generation of adversarial perturbations.
Specifically, AdvRL-GAN interacts with the environment for 1,000 epochs to optimize the generation of adversarial perturbations and then saves the model parameters.
Subsequently, AdvRL-GAN reinitializes in the environment, reloads the saved parameters, and interacts for an additional 1,000 epochs to learn the generation of adversarial perturbations.
This learning process is repeated until the number of interactions reaches 10,000 trajectories.
On the contrary, AAT and BCQ (i.e., an offline RL method) randomly select 10,000 trajectories from medium data to learn adversarial perturbations.

\begin{figure}[]
  \centering
  \subfigure{\includegraphics[width=0.22\textwidth]{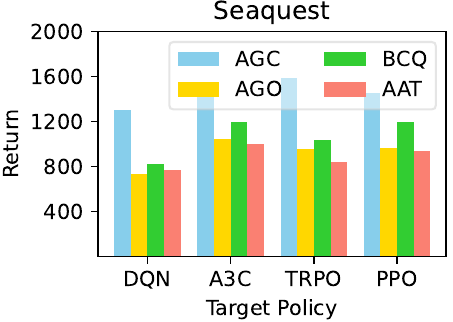}}
  \subfigure{\includegraphics[width=0.22\textwidth]{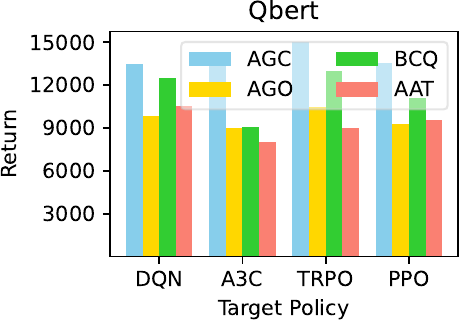}}
  \caption{The performance of the target policy after being attacked by different attack methods.
 The y-axis represents the cumulative rewards of agent, while the x-axis is the training epochs of AAT.}
  \label{fig:clearning}
\end{figure}

Fig.~\ref{fig:clearning} shows the comparison results between AAT, BCQ and AdvRL-GAN on Seaquest and Qbert environments. 
Here, AGC denotes AdvRL-GAN learning adversarial perturbations using a continual learning approach, while AGO represents AdvRL-GAN learning adversarial perturbations by interacting with the environment for 10,000 trajectories at once.
Our experiments indicate that AAT achieves superior attack performance compared to AGC under the same number of interactions and even surpasses AGO in overall attack effectiveness.
Although the BCQ algorithm outperforms AGC in learning from historical data, its attack performance remains poor because it merely imitates the adversarial perturbations in the data.
Additionally, an interesting phenomenon is that the attack performance of AGC is consistently lower than that of AGO.
A potential explanation for this phenomenon is the contemporary preference for one-shot learning paradigms in neural network optimization.

$\quad$
\\
\noindent \textbf{Random perturbations}.
In this section, we explore the impact of random perturbations on AAT performance in mixed data.
Specifically, we train AAT in Chopper Command and Space Invaders environments using two equal number of historical datasets (i.e., $D_1$ and $D_2$), where $D_1$ does not contain randomly perturbed data and $D_2$ contains 20$\%$ randomly perturbations data. 
Note that random perturbation refers to adding random perturbations to the state of the target policy to disrupt its action choices.
Fig.~\ref{fig:data_random_image} shows the results of AAT attacking DQN, A3C, TRPT and PPO during training.
Despite exhibiting slower convergence and higher variance when trained on data with random perturbations, AAT demonstrates superior performance.

\begin{figure*}[h]
  \centering
  \subfigure{\includegraphics[width=0.24\textwidth]{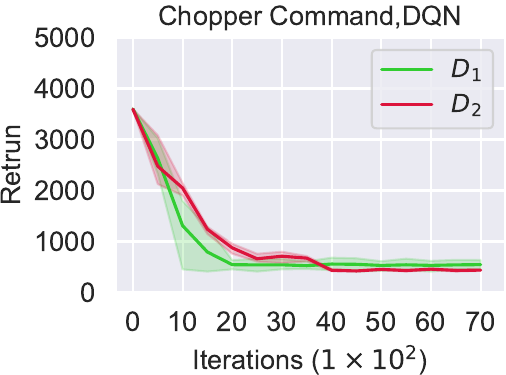}}
   \subfigure{\includegraphics[width=0.24\textwidth]{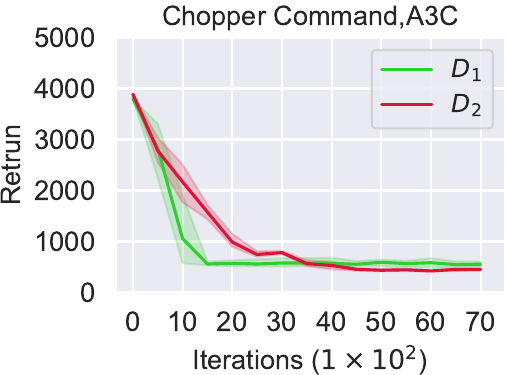}}
   \subfigure{\includegraphics[width=0.24\textwidth]{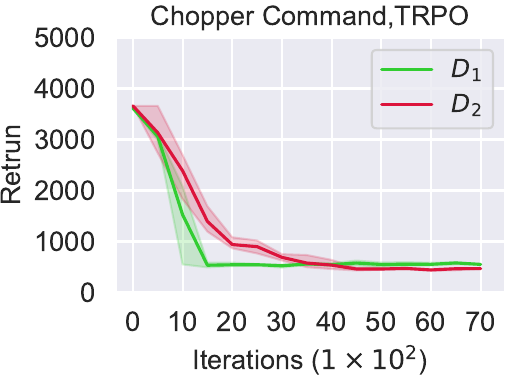}}
  \subfigure{\includegraphics[width=0.24\textwidth]{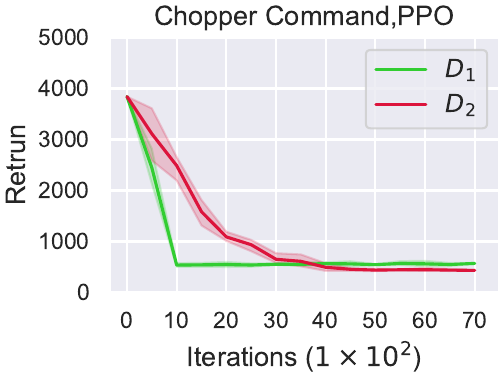}}
  \subfigure{\includegraphics[width=0.24\textwidth]{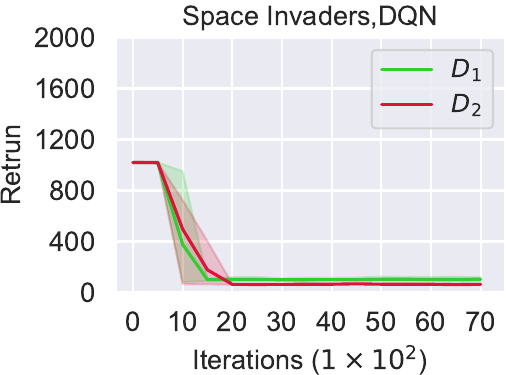}}
   \subfigure{\includegraphics[width=0.24\textwidth]{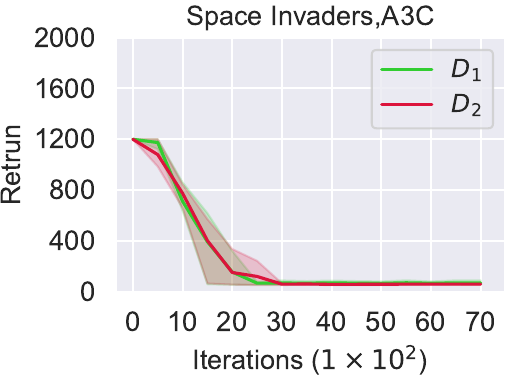}}
   \subfigure{\includegraphics[width=0.24\textwidth]{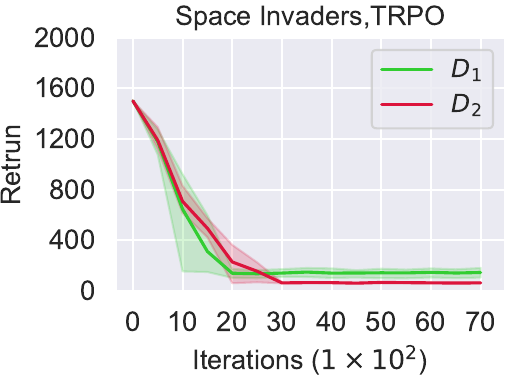}}
  \subfigure{\includegraphics[width=0.24\textwidth]{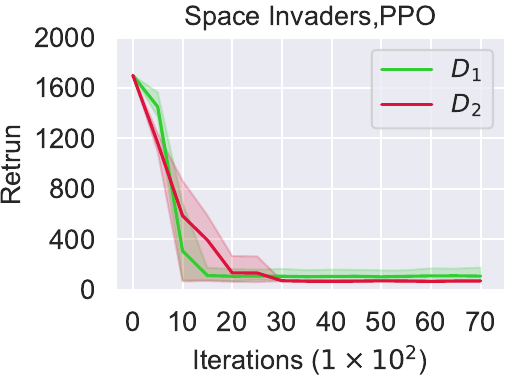}}
  \caption{The impact of random perturbation on AAT attack performance.}
  \label{fig:data_random_image}
\end{figure*}

\begin{table}[]
\centering
\caption{The impact of different data collection methods on AAT performance.}
\scalebox{0.83}{
\begin{tabular}{l|l|cccc}
\hline
Environments                                               & Data type                                                        & DQN                                                                           & A3C                                                                           & TRPO                                                                          & PPO                                                                           \\ \hline
Pong                                                       & \begin{tabular}[c]{@{}l@{}}$D_{FGSM}$\\ $D_{SKip}$\\ $D_{EDGE}$\\ $D_{AGAN}$\end{tabular} & \begin{tabular}[c]{@{}c@{}}-15.18\\ -9.36\\ -12.33\\ -9.67\end{tabular}       & \begin{tabular}[c]{@{}c@{}}-13.27\\ -11.59\\ -10.64\\ -12.91\end{tabular}     & \begin{tabular}[c]{@{}c@{}}-12.63\\ -10.84\\ -13.01\\ -10.79\end{tabular}     & \begin{tabular}[c]{@{}c@{}}-14.14\\ -12.56\\ -10.01\\ -9.95\end{tabular}      \\ \hline
\begin{tabular}[c]{@{}l@{}}Chopper \\ Command\end{tabular} & \begin{tabular}[c]{@{}l@{}}$D_{FGSM}$\\ $D_{SKip}$\\ $D_{EDGE}$\\ $D_{AGAN}$\end{tabular} & \begin{tabular}[c]{@{}c@{}}1457.34\\ 1678.43\\ 1832.56\\ 1467.43\end{tabular} & \begin{tabular}[c]{@{}c@{}}1356.98\\ 1834.14\\ 1754.36\\ 1548.21\end{tabular} & \begin{tabular}[c]{@{}c@{}}1543.55\\ 1672.81\\ 1854.66\\ 1456.32\end{tabular} & \begin{tabular}[c]{@{}c@{}}1437.63\\ 1783.29\\ 1659.26\\ 1499.57\end{tabular} \\ \hline
\end{tabular}
}
\label{tab:data_type}
\end{table}

\subsection{Efficiency exploration}
\label{sup:effic_explor}
We explore the time efficiency of generating adversarial examples for FGSM, Skip, S-T, EDGE, AdvRL-GAN, and AAT.
Table~\ref{tab:time_sample} provides the time required for these methods to generate an adversarial example in the Breakout, Pong, Seaquest, Humanoid, Hopper, and Walker environments. 
This time is calculated as the total generation time divided by the number of adversarial examples.
We can see that FGSM calculates adversarial perturbations based on gradients multiple steps, hence it is comparatively time-consuming.
Skip and S-T methods are slower because of critical frame calculations and threshold comparisons. 
EDGE efficiently provides key trajectories through linear Gaussian methods, but the C$\&$W method is time-consuming due to multiple optimization processes. 
AdvRL-GAN generates examples via a RL method but takes a long time due to the need to rebuild examples using a GAN.
Compared to the above methods, AAT directly outputs adversarial perturbations based on current input information without a complex computational process, thus it can quickly generate adversarial examples.

\begin{table}[]
\caption{The average time to generate an adversarial examples. Note that the time unit is milliseconds.}
\label{tab:time_sample}
\centering
\scalebox{0.80}{ 
\begin{tabular}{lccccll}
\hline
Method & Pong & Breakout & Seaquest & Humanoid & Hopper & Walker \\ \hline
\begin{tabular}[c]{@{}l@{}}FGSM\\ Skip\\ S-T\\ EDGE\\ AdvRL-GAN\\ AAT\end{tabular} & 
\begin{tabular}[c]{@{}c@{}}5.3\\ 8.5\\ 13.4\\ 12.8\\ 14.2\\ \textbf{1.56}\end{tabular} & 
\begin{tabular}[c]{@{}c@{}}8.4\\ 13.7\\ 9.8\\ 16.6\\ 18.8\\ \textbf{2.2}\end{tabular} & 
\begin{tabular}[c]{@{}c@{}}4.3\\ 7.2\\ 5.5\\ 9.4\\ 9.83\\ \textbf{2.4}\end{tabular} & 
\begin{tabular}[c]{@{}c@{}}6.3\\ 8.4\\ 8.3\\ 4.1\\ 10.58\\ \textbf{1.5}\end{tabular} & 
\begin{tabular}[c]{@{}l@{}}9.6\\ 13.2\\ 14.1\\ 13.5\\ 15.4\\ \textbf{2.1}\end{tabular} & 
\begin{tabular}[c]{@{}l@{}}7.8\\ 12.9\\ 15.6\\ 16.4\\ 16.8\\ \textbf{2.6}\end{tabular} \\ \hline
\end{tabular}
}
\end{table}

\begin{figure*}[h]
  \centering
  \subfigure{\includegraphics[width=0.24\textwidth]{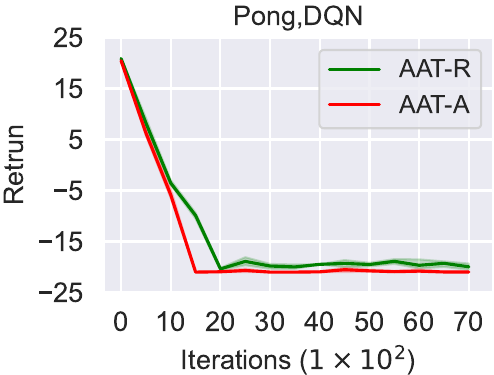}}
  \subfigure{\includegraphics[width=0.24\textwidth]{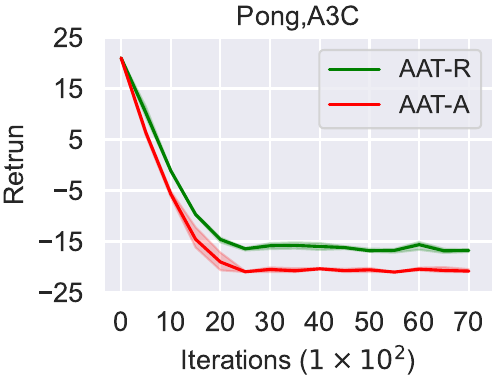}}
    \subfigure{\includegraphics[width=0.24\textwidth]{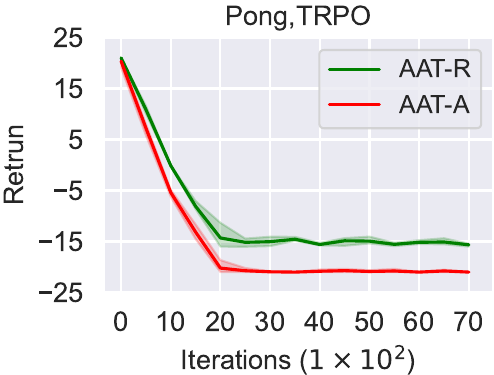}}
    \subfigure{\includegraphics[width=0.24\textwidth]{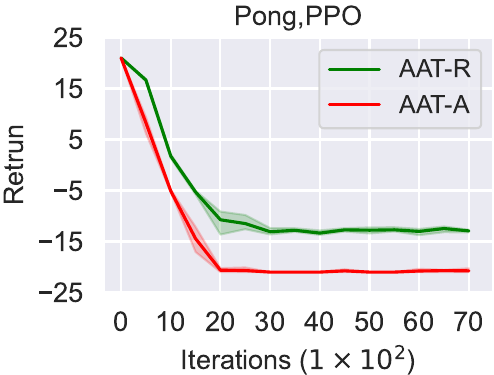}}
  \subfigure{\includegraphics[width=0.24\textwidth]{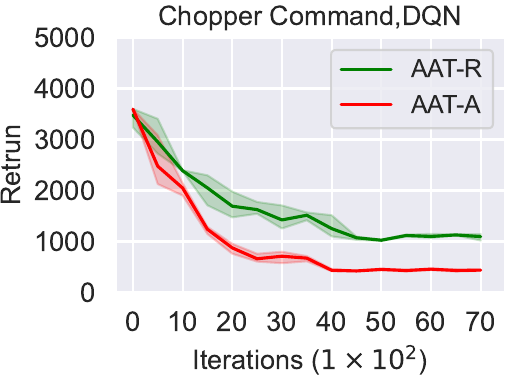}}
   \subfigure{\includegraphics[width=0.24\textwidth]{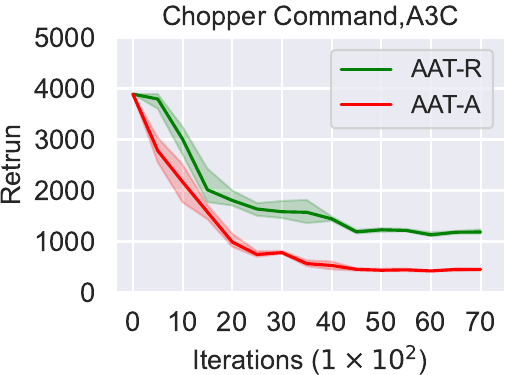}}
   \subfigure{\includegraphics[width=0.24\textwidth]{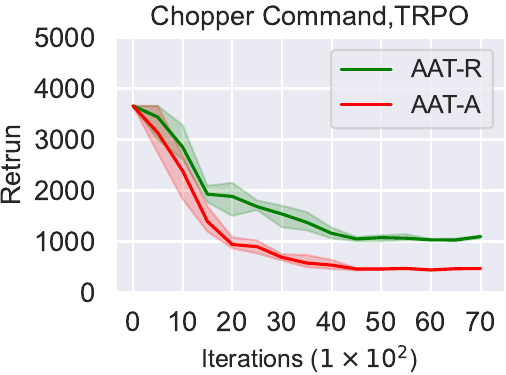}}
  \subfigure{\includegraphics[width=0.24\textwidth]{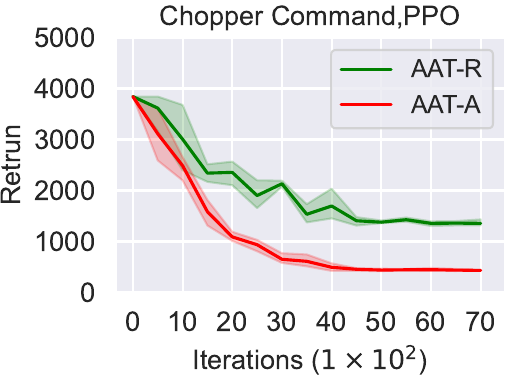}}
   \subfigure{\includegraphics[width=0.24\textwidth]{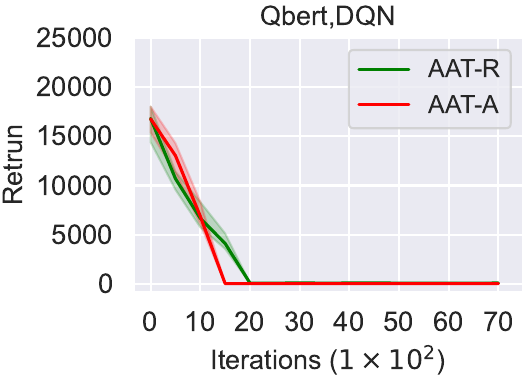}}
  \subfigure{\includegraphics[width=0.24\textwidth]{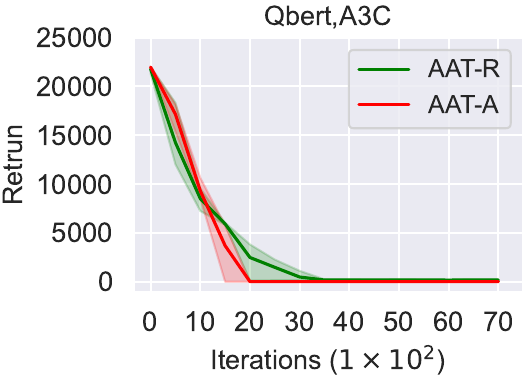}}
  \subfigure{\includegraphics[width=0.24\textwidth]{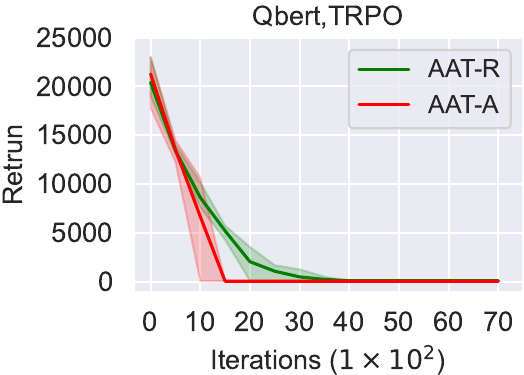}}
  \subfigure{\includegraphics[width=0.24\textwidth]{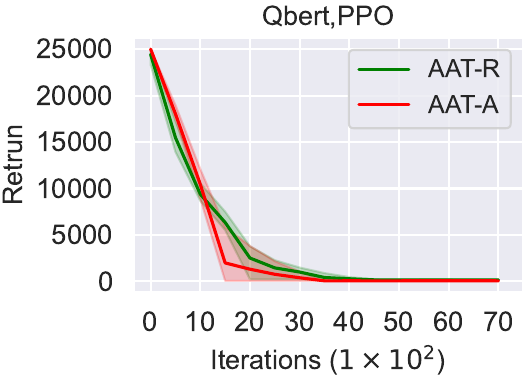}}
  \caption{AAT employs varied learning approaches for attacking different target policies across different environments. The y-axis represents the cumulative rewards of agent, while the x-axis is the training epochs of AAT.}
  \label{fig:advantage_image_sup}
\end{figure*}

\begin{table}[t]
\centering
\caption{The performance of the MSCSA}
\begin{tabular}{llcccc}
\hline
            &                                                    & \multicolumn{2}{c}{White-box}                                                                                 & \multicolumn{2}{c}{Black-box}                                                                                     \\ \hline
Environment & Method                                             & D4PG                                                  & A3C                                                   & D4PG                                                    & A3C                                                     \\ \hline
HalfCheetah & \begin{tabular}[c]{@{}l@{}}SA\\ MSCSA\end{tabular} & \begin{tabular}[c]{@{}c@{}}51.25\\ \textbf{25.56}\end{tabular} & \begin{tabular}[c]{@{}c@{}}42.34\\ \textbf{25.17}\end{tabular} & \begin{tabular}[c]{@{}c@{}}52.23\\ \textbf{46.59}\end{tabular}   & \begin{tabular}[c]{@{}c@{}}56.37\\ \textbf{40.27}\end{tabular}   \\ \hline
Hoper       & \begin{tabular}[c]{@{}l@{}}SA\\ MSCSA\end{tabular} & \begin{tabular}[c]{@{}c@{}}190.31\\ \textbf{118.61}\end{tabular} & \begin{tabular}[c]{@{}c@{}}74.42\\ \textbf{68.72}\end{tabular} & \begin{tabular}[c]{@{}c@{}}154.43\\ \textbf{138.31}\end{tabular} & \begin{tabular}[c]{@{}c@{}}\textbf{101.45}\\ 108.52\end{tabular} \\ \hline
Humanoid    & \begin{tabular}[c]{@{}l@{}}SA\\ MSCSA\end{tabular} & \begin{tabular}[c]{@{}c@{}}26.43\\ \textbf{5.65}\end{tabular}  & \begin{tabular}[c]{@{}c@{}}13.47\\ \textbf{4.65}\end{tabular}  & \begin{tabular}[c]{@{}c@{}}37.82\\ \textbf{28.31}\end{tabular}   & \begin{tabular}[c]{@{}c@{}}41.37\\ \textbf{30.64}\end{tabular}   \\ \hline
Pendulum    & \begin{tabular}[c]{@{}l@{}}SA\\ MSCSA\end{tabular} & \begin{tabular}[c]{@{}c@{}}40.62\\ \textbf{36.51}\end{tabular} & \begin{tabular}[c]{@{}c@{}}45.34\\ \textbf{30.26}\end{tabular} & \begin{tabular}[c]{@{}c@{}}80.61\\ \textbf{78.34}\end{tabular}   & \begin{tabular}[c]{@{}c@{}}82.42\\ \textbf{68.35}\end{tabular}   \\ \hline
Walker      & \begin{tabular}[c]{@{}l@{}}SA\\ MSCSA\end{tabular} & \begin{tabular}[c]{@{}c@{}}85.34\\ \textbf{60.61}\end{tabular} & \begin{tabular}[c]{@{}c@{}}35.36\\ \textbf{20.42}\end{tabular} & \begin{tabular}[c]{@{}c@{}}175.64\\ \textbf{128.43}\end{tabular} & \begin{tabular}[c]{@{}c@{}}\textbf{105.52}\\ 108.36\end{tabular} \\ \hline
\end{tabular}
\label{tab:mscas_per}
\end{table}

\subsection{Additional ablation results}
\label{sup:ab_result}
\noindent \textbf{MSCSA performance}.
In this section, we compare the performance of Adversarial Attack Training (AAT) using MSCSA and traditional self-attention (SA) in the DeepMind Control Suite, aiming to analyze their effectiveness in generating adversarial perturbations. 
To evaluate their adversarial perturbation generation capabilities, we conduct experiments on D4PG and A3C agents across 10 random seeds, with detailed results summarized in Table~\ref{tab:mscas_per}.
The experimental results demonstrate that MSCSA outperforms SA in both white-box and black-box scenarios.
This discrepancy can be attributed to fundamental differences in attention mechanisms: while SA primarily captures global dependencies across temporal sequences, Multi-scale architecture of MSCSA enables simultaneous processing of both short-term and long-term context. 
Consequently, SA-based perturbations tend to overemphasize historical patterns at the expense of current state characteristics, leading to suboptimal attack efficacy at critical decision points.

\begin{table*}[t]
\centering
\caption{Ablation on different window growth strategies and initial window sizes $Len$. The experimental results are the average of 10 runs using D4PG as the target algorithm in the Hopper environment.}
\label{tab:ablation_window}
\begin{tabular}{lccccp{5.3cm}}
\toprule
\textbf{Strategy} & $\boldsymbol{Len}$ & \textbf{Window Seq.} & \textbf{Reward} & \textbf{Cost} (ms) & \textbf{Notes} \\
\midrule
Exponential       & 5      & [5, 10, 20]     & 118.3  & 125       & Best trade-off of local and global modeling \\
Exponential       & 8      & [8, 16, 32]    & 108.7  & 356         & High cost with little gain \\
Linear            & 5      & [5, 10, 15]     & 235.6  & 119       & Less effective in modeling long-term context \\
Linear            & 8      & [8, 12, 16]    & 209.3  & 298         & Overhead increases without performance gain \\
Fixed             & 5      & [5, 5, 5]      & 583.1  & 89          & Poor modeling capacity, lacks multi-scale \\
Exponential       & 2      & [2, 4, 8]      & 629.8  & 90   & Efficient but may underutilize long-term info \\
\bottomrule
\end{tabular}
\end{table*}

\begin{figure*}[t]
  \centering
  \subfigure{\includegraphics[width=0.24\textwidth]{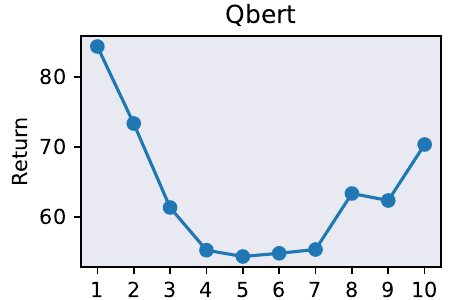}}
  \subfigure{\includegraphics[width=0.24\textwidth]{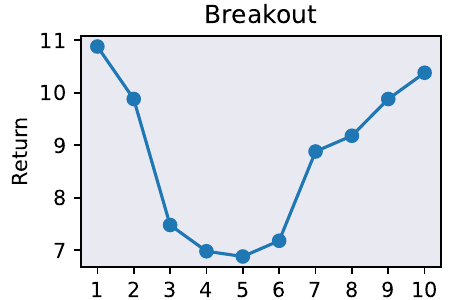}}
  \subfigure{\includegraphics[width=0.24\textwidth]{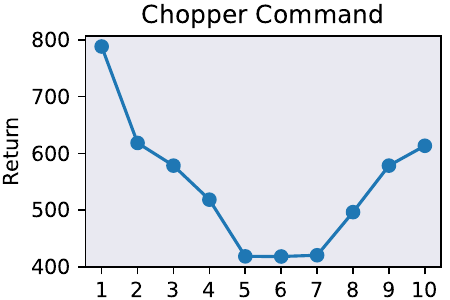}}
  \subfigure{\includegraphics[width=0.24\textwidth]{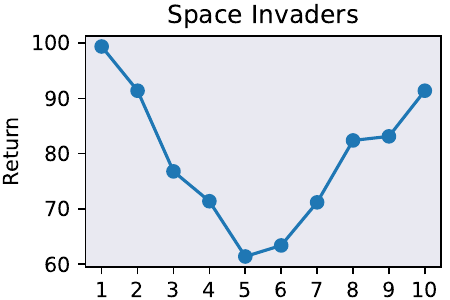}}
  \subfigure{\includegraphics[width=0.24\textwidth]{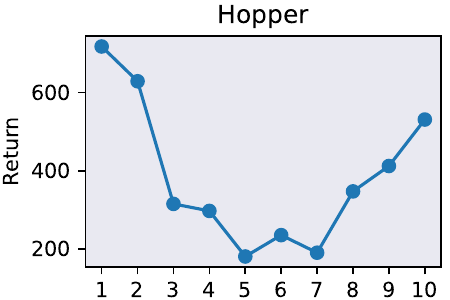}}
  \subfigure{\includegraphics[width=0.24\textwidth]{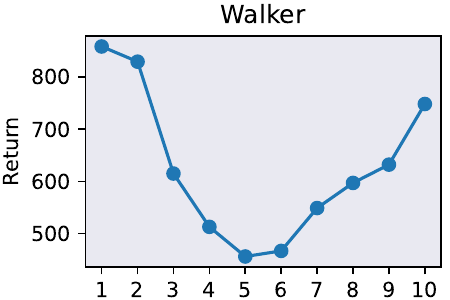}}
   \subfigure{\includegraphics[width=0.24\textwidth]{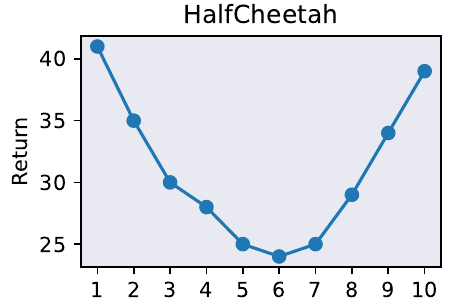}}
    \subfigure{\includegraphics[width=0.24\textwidth]{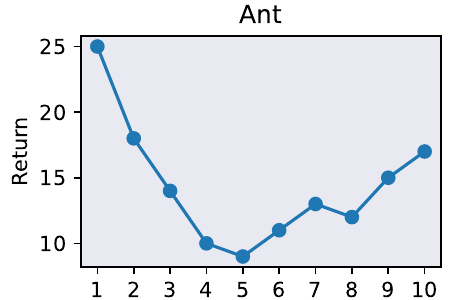}}
  \caption{The impact of different window lengths on AAT performance. We adopt PPO as the target policy in the Atari and employ D4PG as the target policy in MuJoCo.}
  \label{fig:len_aboty}
\end{figure*}

\begin{figure*}[t]
  \centering
  \subfigure{\includegraphics[width=0.24\textwidth]{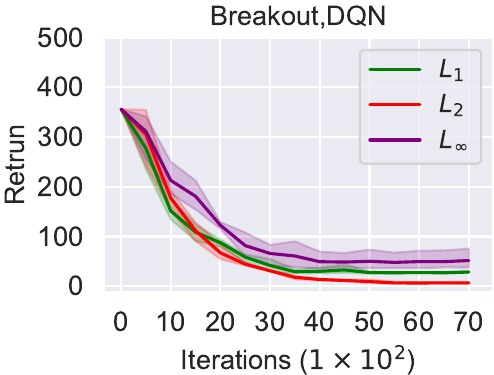}}
  \subfigure{\includegraphics[width=0.24\textwidth]{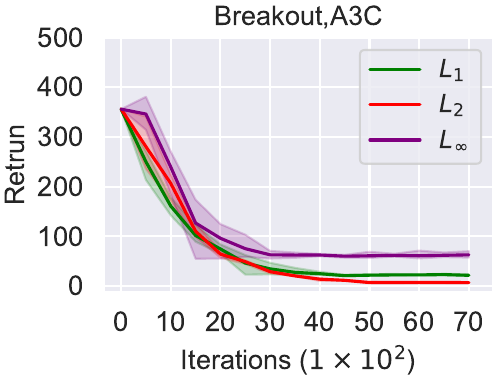}}
  \subfigure{\includegraphics[width=0.24\textwidth]{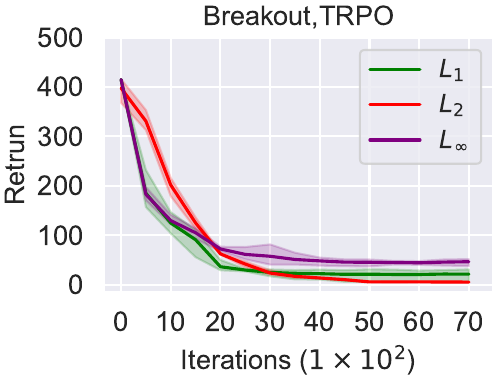}}
  \subfigure{\includegraphics[width=0.24\textwidth]{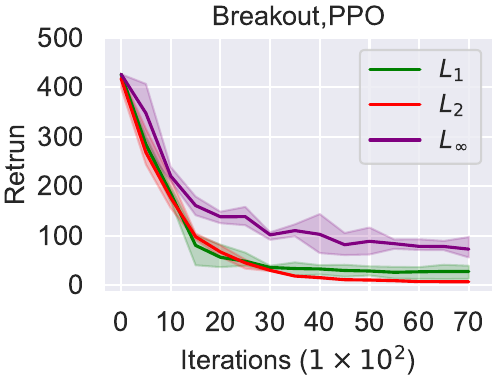}}
  \subfigure{\includegraphics[width=0.24\textwidth]{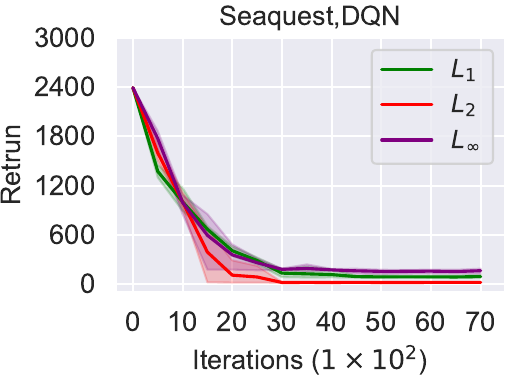}}
  \subfigure{\includegraphics[width=0.24\textwidth]{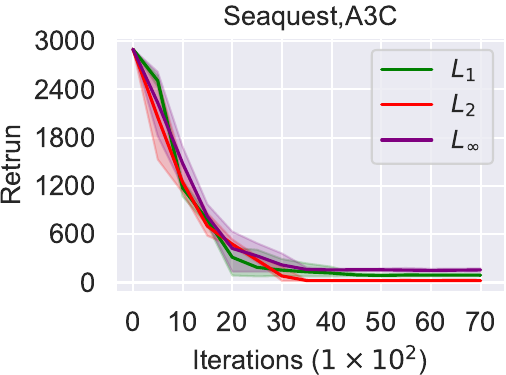}}
   \subfigure{\includegraphics[width=0.24\textwidth]{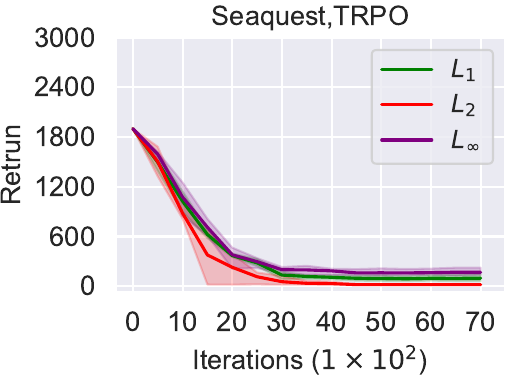}}
    \subfigure{\includegraphics[width=0.24\textwidth]{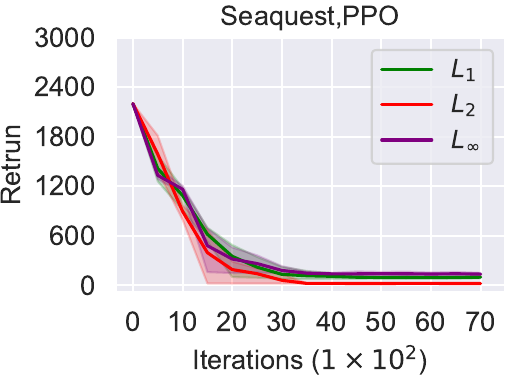}}
  \subfigure{\includegraphics[width=0.24\textwidth]{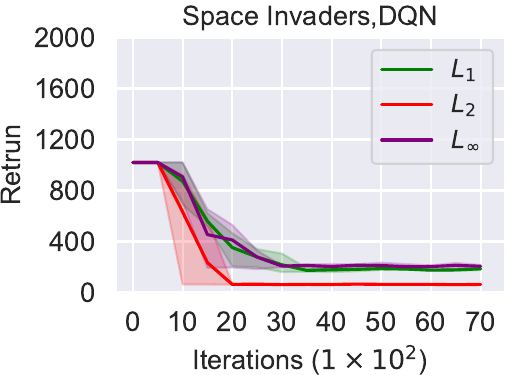}}
   \subfigure{\includegraphics[width=0.24\textwidth]{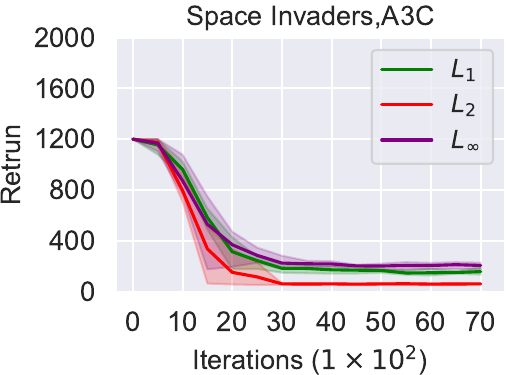}}
   \subfigure{\includegraphics[width=0.24\textwidth]{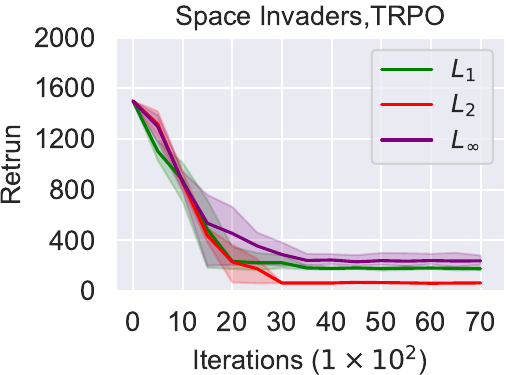}}
  \subfigure{\includegraphics[width=0.24\textwidth]{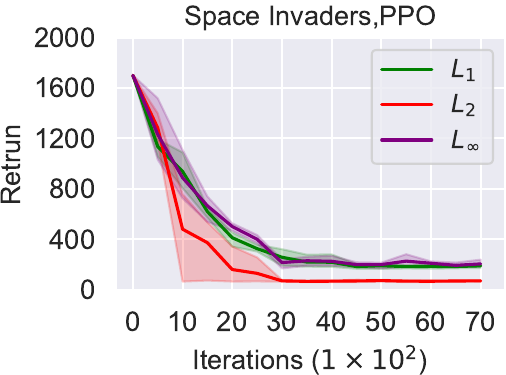}}
  \caption{AAT uses $L_1$, $L_2$, and $L_\infty$ losses to affect different target policies in different environments during the learning process.}
  \label{fig:normal_image}
\end{figure*}

\begin{table*}[h]
\centering
\caption{Advantageous ablation results.}
\begin{tabular}{ll|llll|llll}
\hline
                &                                                         & \multicolumn{4}{c|}{White-box}                                                                                                                                                                                                                                                                                         & \multicolumn{4}{c}{Black-box}                                                                                                                                                                                                                                                                                         \\ \hline
Environment     & Method                                                  & \multicolumn{1}{c}{DQN}                                                     & \multicolumn{1}{c}{A3C}                                                     & \multicolumn{1}{c}{TRPO}                                                    & \multicolumn{1}{c|}{PPO}                                                     & \multicolumn{1}{c}{DQN}                                                     & \multicolumn{1}{c}{A3C}                                                     & \multicolumn{1}{c}{TRPO}                                                    & \multicolumn{1}{c}{PPO}                                                     \\ \hline
Pong            & \begin{tabular}[c]{@{}l@{}}AAT-GA\\ AAT-WA\end{tabular} & \multicolumn{1}{c}{\begin{tabular}[c]{@{}c@{}}-19.83\\ -\textbf{20.96}\end{tabular}} & \multicolumn{1}{c}{\begin{tabular}[c]{@{}c@{}}-18.18\\ -\textbf{20.61}\end{tabular}} & \multicolumn{1}{c}{\begin{tabular}[c]{@{}c@{}}-\textbf{19.98}\\ -19.97\end{tabular}} & \multicolumn{1}{c|}{\begin{tabular}[c]{@{}c@{}}-19.25\\ -\textbf{19.87}\end{tabular}} & \multicolumn{1}{c}{\begin{tabular}[c]{@{}c@{}}-17.26\\ -\textbf{18.34}\end{tabular}} & \multicolumn{1}{c}{\begin{tabular}[c]{@{}c@{}}-17.31\\ -\textbf{18.45}\end{tabular}} & \multicolumn{1}{c}{\begin{tabular}[c]{@{}c@{}}-15.92\\ -\textbf{16.06}\end{tabular}} & \multicolumn{1}{c}{\begin{tabular}[c]{@{}c@{}}-17.61\\ -\textbf{18.98}\end{tabular}} \\ \hline
Chopper Command & \begin{tabular}[c]{@{}l@{}}AAT-GA\\ AAT-WA\end{tabular} & \begin{tabular}[c]{@{}l@{}}\textbf{400.31}\\ 408.68\end{tabular}                     & \begin{tabular}[c]{@{}l@{}}567.33\\ \textbf{410.21}\end{tabular}                     & \begin{tabular}[c]{@{}l@{}}435.65\\ \textbf{420.98}\end{tabular}                     & \begin{tabular}[c]{@{}l@{}}420.68\\ \textbf{418.36}\end{tabular}                      & \begin{tabular}[c]{@{}l@{}}693.27\\ \textbf{568.56}\end{tabular}                     & \begin{tabular}[c]{@{}l@{}}784.56\\ \textbf{628.54}\end{tabular}                     & \begin{tabular}[c]{@{}l@{}}610.32\\ \textbf{609.88}\end{tabular}                     & \begin{tabular}[c]{@{}l@{}}642.55\\ \textbf{635.36}\end{tabular}                     \\ \hline
Seaquest        & \begin{tabular}[c]{@{}l@{}}AAT-GA\\ AAT-WA\end{tabular} & \begin{tabular}[c]{@{}l@{}}19.32\\ \textbf{18.34}\end{tabular}                       & \begin{tabular}[c]{@{}l@{}}\textbf{20.11}\\ 20.22\end{tabular}                       & \begin{tabular}[c]{@{}l@{}}21.52\\ \textbf{19.84}\end{tabular}                       & \begin{tabular}[c]{@{}l@{}}23.01\\ \textbf{21.42}\end{tabular}                        & \begin{tabular}[c]{@{}l@{}}90.23\\ \textbf{80.45}\end{tabular}                       & \begin{tabular}[c]{@{}l@{}}84.51\\ \textbf{70.89}\end{tabular}                       & \begin{tabular}[c]{@{}l@{}}79.46\\ \textbf{60.42}\end{tabular}                       & \begin{tabular}[c]{@{}l@{}}79.87\\ \textbf{73.58}\end{tabular}                       \\ \hline
Qbert           & \begin{tabular}[c]{@{}l@{}}AAT-GA\\ AAT-WA\end{tabular} & \begin{tabular}[c]{@{}l@{}}69.35\\ \textbf{50.56}\end{tabular}                       & \begin{tabular}[c]{@{}l@{}}\textbf{60.03}\\ 60.22\end{tabular}                       & \begin{tabular}[c]{@{}l@{}}65.46\\ \textbf{58.42}\end{tabular}                       & \begin{tabular}[c]{@{}l@{}}71.21\\ \textbf{54.32}\end{tabular}                        & \begin{tabular}[c]{@{}l@{}}635.67\\ \textbf{598.55}\end{tabular}                     & \begin{tabular}[c]{@{}l@{}}535.34\\ \textbf{478.35}\end{tabular}                     & \begin{tabular}[c]{@{}l@{}}635.78\\ \textbf{483.21}\end{tabular}                     & \begin{tabular}[c]{@{}l@{}}638.42\\ \textbf{556.34}\end{tabular}                     \\ \hline
Space Invaders  & \begin{tabular}[c]{@{}l@{}}AAT-GA\\ AAT-WA\end{tabular} & \begin{tabular}[c]{@{}l@{}}70.23\\ \textbf{59.59}\end{tabular}                       & \begin{tabular}[c]{@{}l@{}}63.46\\ \textbf{51.26}\end{tabular}                       & \begin{tabular}[c]{@{}l@{}}68.35\\ \textbf{60.35}\end{tabular}                       & \begin{tabular}[c]{@{}l@{}}73.21\\ \textbf{61.38}\end{tabular}                        & \begin{tabular}[c]{@{}l@{}}110.42\\ \textbf{103.66}\end{tabular}                     & \begin{tabular}[c]{@{}l@{}}134.58\\ \textbf{126.31}\end{tabular}                     & \begin{tabular}[c]{@{}l@{}}124.68\\ \textbf{109.84}\end{tabular}                     & \begin{tabular}[c]{@{}l@{}}136.27\\ \textbf{118.28}\end{tabular}                     \\ \hline
Breakout        & \begin{tabular}[c]{@{}l@{}}AAT-GA\\ AAT-WA\end{tabular} & \begin{tabular}[c]{@{}l@{}}8.97\\ \textbf{6.56}\end{tabular}                         & \begin{tabular}[c]{@{}l@{}}8.79\\ \textbf{8.21}\end{tabular}                         & \begin{tabular}[c]{@{}l@{}}6.53\\ \textbf{5.98}\end{tabular}                         & \begin{tabular}[c]{@{}l@{}}9.32\\ \textbf{6.88}\end{tabular}                          & \begin{tabular}[c]{@{}l@{}}35.32\\ \textbf{28.32}\end{tabular}                       & \begin{tabular}[c]{@{}l@{}}25.67\\ \textbf{20.36}\end{tabular}                       & \begin{tabular}[c]{@{}l@{}}36.77\\ \textbf{22.22}\end{tabular}                       & \begin{tabular}[c]{@{}l@{}}33.67\\ \textbf{20.21}\end{tabular}                       \\ \hline
\end{tabular}
\label{tab:advantage_ady}
\end{table*}

$\quad$
\\
\textbf{Ablation on window initialization and growth strategy.}
To assess the impact of temporal window design in MSCSA, we conduct ablations over both the growth strategy and the initial window size $Len$. As shown in Table~\ref{tab:ablation_window}, exponential growth with $r=2$ and $Len=5$ achieves the best average reward while maintaining moderate computational cost. Using larger initial windows (e.g., $Len=8$) increases inference cost without notable performance gains, while smaller values (e.g., $Len=2$) are efficient but may underrepresent longer dependencies. Linear growth performs competitively but is slightly inferior in both performance and scale diversity. Fixed-length windows degrade significantly, highlighting the importance of multi-scale temporal abstraction. These results validate the choice of geometric scaling from a moderate base window as a robust design for balancing short- and long-term temporal modeling.

To evaluate the sensitivity of MSCSA to the initial window size $Len$, we linearly scan $Len$ from 1 to 10 (with a step size of 1).
Keeping the growth policy and the scale factor K constant, we perform repeated runs with 10 different random seeds for each $Len$ and record the impact on the performance of AAT.
As shown in Fig.~\ref{fig:len_aboty}, in most environments, setting the window size to 5 achieves better attack performance in both Atari and MuJoCo environments.
However, the window size can be set to a larger length to achieve higher attack performance in some environments, such as in the Chopper Command and HalfCheetah environments.

$\quad$
\\
\noindent\textbf{Advantage improvement.}To validate the efficacy of using the advantage function, we carry out ablation studies in six different settings.
During the training of AAT, we evaluate the attack outcomes every 1,000 iterations.
Fig.~\ref{fig:advantage_image_sup} illustrates the attack performance of the AAT method in the Pong,  Chopper Command and environments.
Here, AAT-R refers to the method incorporating future reward sum learning, while AAT-A represents the approach integrating advantage function learning.
Overall, AAT-R and AAT-A demonstrate differences in attack effectiveness and convergence speed across various target policies within identical environments.
Before $10 \times 10^2$ iterations, the effects of AAT-R and AAT-A attacks closely resemble each other in the Breakout environment. This similarity can be primarily attributed to the model being in its early learning stages.
As the number of iterations increases, the attack effectiveness of AAT-A begins to surpass that of AAT-R, eventually converging around $40 \times 10^2$.
In the Pong and Chopper Command environments, AAT-A outperforms AAT-R in terms of overall attack effectiveness, which demonstrates the feasibility of using advantage functions to achieve higher attack effectiveness.

$\quad$
\\
\textbf{Constraint investigation.}
we study the impact of employing different constraint normal in Eq.4 on the outcomes of the attacks.
Specifically, $L_1$, $L_2$, and $L_\infty$ are utilized to compute the difference between adversarial examples and original states, and they serve as the loss function for the AAT method.
Fig.\ref{fig:normal_image} depicts the attack outcomes under various L-norms.
In the current context, $L_2$ is preferable over $L_1$ and $L_\infty$ as the constraint norm for AAT.

$\quad$
\\
\noindent\textbf{Weighted advantage performance.} In order to verify the impact of the weighted advantage and the advantage calculated by the Q and V functions (i.e., ordinary advantage) on the performance of AAT, we test the attack performance of AAT with different advantage learning methods in Atari.
Table 3 provides the comparison results in white-box and black-box scenarios, where AAT- GA indicates that AAT uses ordinary advantage and AAT-WA indicates that AAT uses weighted advantage.
The experimental results show that the weighted advantage is generally superior to the ordinary advantage, especially achieving a comprehensive surpass in black-box scenarios.
This phenomenon occurs because weighted advantages can confine the advantages of unseen states within a certain range, which ensures that the performance will not significantly degrade due to overestimation of the state advantages.

\subsection{AAT generalization}
\label{sup:genera_multi}
To verify the generalization capability of AAT, we conduct experiments in cooperative multi-agent environments based on the experimental setup described in FGSMLW~\cite{DBLP:LinDZLP20}.       
Our evaluation employs QMIX as the target policy and uses the 2s3z map in the StarCraft Multi-Agent Challenge (SMAC).
In this scenario, each cooperative multi-agent team comprises five units: two Stalkers and three Zealots.
Within this framework, the attacker targets a single Stalker as the victim agent, perturbing its observations with the objective of degrading the team's overall cooperative performance.

\begin{figure*}[t]
  \centering
  \subfigure{\includegraphics[width=0.45\textwidth]{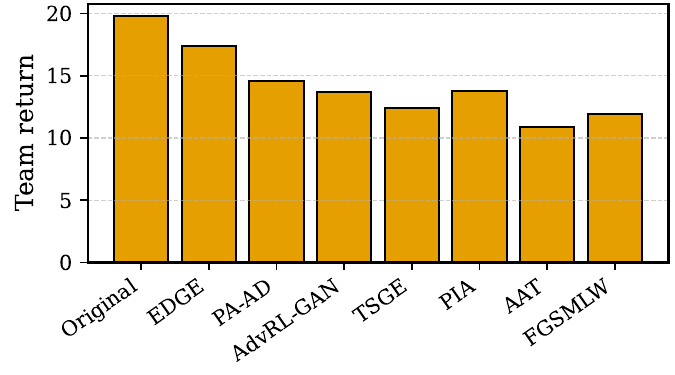}}
  \subfigure{\includegraphics[width=0.45\textwidth]{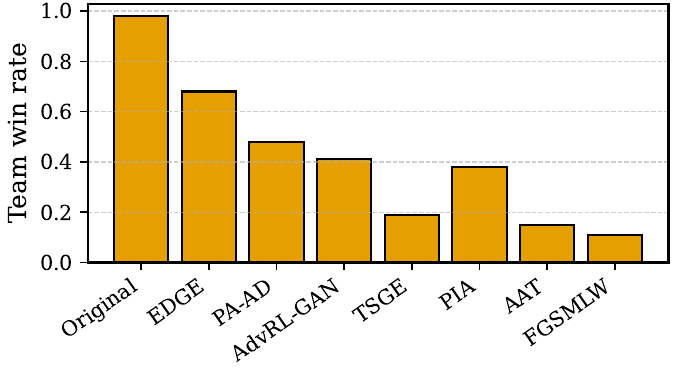}}
  \caption{The attack outcome of AAT in a multi-agent setting. The experimental results are the average of 10 random seeds.}
  \label{fig:multi_agent}
\end{figure*}

Fig.~\ref{fig:multi_agent} illustrates the attack results of different methods, where Original indicates that QMIX is not under attack. 
The results show that AAT significantly reduces both the team reward and win rate more effectively than a range of baseline methods, including EDGE, PA-AD, AdvRL-GAN, TSGE, and PIA.
Furthermore, in terms of reward reduction, AAT even surpasses the FGSMLW method, which specializes in attacking c-MARL algorithms.

\end{document}